\documentclass{amsart}

\usepackage{mathtools}
\usepackage{mathabx}
\usepackage{booktabs}
\usepackage[T1]{fontenc}
\usepackage{amsmath,amsfonts,mathrsfs,amssymb,amsthm}
\usepackage{enumerate}
\usepackage{bm}
\usepackage[usenames,dvipsnames]{xcolor}
\definecolor{myblue}{rgb}{0.21, 0.34, 0.74}
\definecolor{mygrey}{rgb}{0.55, 0.57, 0.67}
\definecolor{myred}{rgb}{0.79, 0.0, 0.09}
\definecolor{mygreen}{rgb}{0.05, 0.5, 0.06}
\usepackage{tikz}

\usepackage[colorlinks=true, citecolor=myblue, linkcolor=black]{hyperref}
\hypersetup{
    colorlinks=true,
    linkcolor=myblue,
    filecolor=black,
    urlcolor=myblue,
    citecolor=myblue
}

\DeclareMathAlphabet{\mathscrbf}{OMS}{mdugm}{b}{n}
\usepackage[scr=boondoxo]{mathalpha}
\usepackage{cleveref}
\usepackage{subcaption}

\usepackage{comment}

\usepackage{wrapfig}
\usepackage{upgreek}

\renewcommand{\S}{\textsection}

\newcommand{\R}{\mathbb{R}}
\newcommand{\proj}{\bm{\mathsf{P}}^\perp}
\newcommand{\diff}{\, \mathrm{d}}
\newcommand{\dist}{\mathrm{dist}}

\newcommand{\dive}{\mathrm{div}}
\renewcommand{\S}{\mathbb{S}^{d-1}}
\newcommand{\Tan}{\mathrm{T}}
\newcommand{\cP}{\mathcal{P}}

\DeclareMathOperator*{\argmin}{arg\,min}
\DeclareMathOperator*{\arginf}{arg\,inf}

\numberwithin{equation}{section}

\newcommand{\gradW}{\nabla\mkern-10mu\nabla}

\usepackage{comment}
\usepackage[font=small,labelfont=bf]{caption}

\theoremstyle{plain}
\newtheorem{theorem}{Theorem}[section]

\newtheorem{definition}[theorem]{Definition}
\newtheorem{proposition}[theorem]{Proposition}
\newtheorem{lemma}[theorem]{Lemma}

\newtheorem{remark}[theorem]{Remark}

\newtheorem{example}[theorem]{Example}
\newtheorem{open}{Problem}

\begin{document}

\title{A mathematical perspective on Transformers}

\author[Geshkovski]{Borjan Geshkovski}
\address{Department of Mathematics, Massachusetts Institute of Technology, 77 Massachusetts Ave, 02139 Cambridge MA, USA}
\email{borjan@mit.edu}

\author[Letrouit]{Cyril Letrouit}
\address{CNRS \& Université Paris-Saclay\\  Laboratoire de mathématiques d'Orsay, 307 rue Michel Magat, Bâtiment 307, 91400 Orsay, France}
\email{cyril.letrouit@universite-paris-saclay.fr}

\author[Polyanskiy]{Yury Polyanskiy}
\address{Department of EECS, Massachusetts Institute of Technology, 77 Massachusetts Ave, 02139 Cambridge MA, USA}
\email{yp@mit.edu}

\author[Rigollet]{Philippe Rigollet}
\address{Department of Mathematics, Massachusetts Institute of Technology, 77 Massachusetts Ave, 02139 Cambridge MA, USA}
\email{rigollet@math.mit.edu}

\subjclass{Primary: 34D05, 34D06, 35Q83; Secondary: 52C17}

\date{}

\keywords{Transformers, self-attention, interacting particle systems, clustering, gradient flows}

\maketitle

\begin{abstract}
Transformers play a central role in the inner workings of large language models. We develop a mathematical framework for analyzing Transformers based on their interpretation as interacting particle systems, with a particular emphasis on long-time clustering behavior.
    Our study explores the underlying theory and offers new perspectives for mathematicians as well as computer scientists.
\end{abstract}

\setcounter{tocdepth}{1}
\tableofcontents

\section{Outline}
The introduction of \emph{Transformers} in 2017 by Vaswani et al.~\cite{vaswani2017attention} marked a significant milestone in the development of neural network architectures. Central to this contribution is \emph{self-attention}, a novel mechanism which distinguishes Transformers from traditional architectures, and which plays a substantial role in their superior practical performance. In fact, this innovation has been a key catalyst for the progress of artificial intelligence in areas such as computer vision and natural language processing, notably with the emergence of large language models. As a result, understanding the mechanisms by which Transformers, and especially self-attention, process data is a crucial yet largely uncharted research area.

A common characteristic of deep neural networks (DNNs) is their compositional nature: data is processed sequentially, layer by layer, resulting in a discrete-time dynamical system (we refer the reader to the textbook \cite{GooBenCou16} for a general introduction). 
This perspective has been successfully employed to model \emph{residual neural networks}---see Section~\ref{sec: resnets} for more details---as continuous-time dynamical systems called neural ordinary differential equations (neural ODEs)~\cite{chen2018neural,weinan2017proposal,haber2017stable}. 
In this context, an input $x(0)\in\R^d$, say an image, is evolving according to a given time-varying velocity field as $\dot x(t)=v_t(x(t))$ over some time interval $(0,T)$. 
As such, a DNN can be seen as a flow map $x(0) \mapsto x(T)$ from $\R^d$ to $\R^d$. Even within the restricted class of velocity fields $\{v_t\}_{t\geq0}$ imposed by classical DNN architectures, such flow maps enjoy strong approximation properties as exemplified by a long line of work on these questions \cite{lin2018resnet, zhang2020approximation, li2022deep, tabuada2020universal, ruiz2023neural, cheng2023interpolation}.

Following~\cite{sander2022sinkformers} and ~\cite{vuckovic2020mathematical}, we observe that Transformers are in fact flow maps on $\cP(\R^d)$, the space of probability measures over $\R^d$. 
To realize this flow map from measures to measures, Transformers evolve a \emph{mean-field  
 interacting particle system}. 
More specifically, every particle (called a \emph{token} in this context) follows the flow of a vector field which depends on the empirical measure of all particles.
In turn, the \emph{continuity equation} governs the evolution of the empirical measure of particles, whose long-time behavior is of crucial interest.  
In this regard, our main observation is that particles tend to cluster under these dynamics. This phenomenon is of particular relevance in learning tasks such as \emph{next-token prediction}, wherein one seeks to map a given input sequence (i.e., a sentence) of $n$ tokens (i.e., words) onto a given next token. In this case, the output measure encodes the probability distribution of the next token, and its clustering indicates a small number of possible outcomes. Our results on a simplified but insightful toy model indicate that the limiting distribution is
actually a point mass, leaving no room for diversity or randomness, which is at odds with practical observations. 
This apparent paradox is resolved by the existence of a long-time metastable state. As can be seen from Figures \ref{fig: phase.diag.Id} and \ref{fig: random.QKV}, the Transformer flow appears to possess two different time-scales: in a first phase, tokens quickly form a few clusters, while in a second (much slower) phase, through the process of pairwise merging of clusters, all tokens finally collapse to a single point. This appears to corroborate behavior observed empirically in trained Transformer models, which goes under the names \emph{token uniformity}, \emph{over-smoothing} \cite{chen2022principle, ru2023token, guo2023contranorm, wu2024demystifying, wu2024role, dovonon2024setting, scholkemper2024residual}, or \emph{rank-collapse} \cite{dong2021attention, feng2022rank, noci2022signal, joudaki2023impact, zhao2023are, zhai2023stabilizing, noci2024shaped, bao2024self, cowsik2024geometric}; see also Figure~\ref{fig: albert}.

The goal of this manuscript is twofold. On the one hand, we aim to provide a general and accessible framework to study Transformers from a mathematical perspective. In particular, the structure of these interacting particle systems enables concrete connections to established mathematical topics, such as nonlinear transport equations, Wasserstein gradient flows, collective behavior models, and optimal configurations of points on spheres, among others. On the other hand, we describe several promising research directions with a particular focus on the long-time clustering phenomenon. The main results we present are new, and we also provide what we believe are interesting open problems throughout the paper.

The rest of the paper is arranged in three parts.

\subsubsection*{Part~\ref{part: modeling}: Modeling}

We define an idealized model of the Transformer architecture that captures two of the main characteristics of transformers: \emph{self-attention} and \emph{layer-normalization}. Following a perspective put forward in classical architectures such as ResNets~\cite{chen2018neural,weinan2017proposal,haber2017stable}, we view the successive layers of a neural network as time discretizations of a dynamical system of interacting particles. Layer-normalization effectively constrains particles to evolve on the unit sphere $\S$, whereas self-attention is the particular nonlinear coupling of the particles done through the empirical measure
(Section~\ref{sec: ips}). In turn, the empirical measure evolves according to the continuity partial differential equation (Section~\ref{s:WGF}). Even after significant simplifications, this toy model retains macroscopic characteristics of trained Transformers, namely clustering. We also introduce a simpler surrogate model for self-attention which has the convenient property of being a Wasserstein gradient flow~\cite{ambrosio2005gradient} for an energy functional that is well-studied in the context of optimal configurations of points on the sphere and sheds a complementary light of the source of clustering.

\subsubsection*{Part~\ref{part: clustering}: Clustering}
In this part we recall existing and establish new mathematical results that indicate clustering of tokens in the long-time limit. (See Figure \ref{fig: clustering.graph} for a summary.) Our first results (Theorem \ref{thm: beta.tiny} in Section \ref{sec: temperature} and Theorem \ref{thm: beta.interval} in Section \ref{sec: large.beta}) are in extreme regimes of a temperature parameter $\beta^{-1}$ that appears in the equation. We then move to the high-dimensional case in Section \ref{sec: high.d}, where we begin by recalling Theorem \ref{thm: boumal}---a result of \cite{markdahl2017almost}, recently revisited in \cite{criscitiello2024synchronization}---which entails long-time clustering at any temperature when $d\geq 3$. We provide an exponential rate of convergence in Theorem~\ref{thm: d.infty} when $d\geq n$---
here $n$ denotes the number of particles---. 
We complement this result with an even more precise characterization of the rate of contraction of particles into a cluster. Namely, we describe the histogram of all inter-particle distances, and the time at which all particles are already nearly clustered (Theorem~\ref{thm: phase.transition.curve}). 

\subsubsection*{Part~\ref{sec: beyond}: Further questions}
We propose potential avenues for future research, largely in the form of open questions substantiated by numerical observations. We first focus on the case $d=2$ (Section~\ref{sec: circle}) and elicit a link to Kuramoto oscillators. 
We briefly show in Section~\ref{sec:repulsive} how a simple and natural modification of our model leads to non-trivial questions related to optimal configurations on the sphere. The remaining sections explore interacting particle systems that allow for parameter tuning of the Transformers architectures, a key feature of practical implementations. 

\part{Modeling} \label{part: modeling}

We begin this part by presenting the mathematical model for a Transformer in Section~\ref{sec: ips}. Throughout the paper we focus on a simplified version that includes the self-attention mechanism as well as layer normalization (Section ~\ref{s:interacting}), but excludes additional feed-forward layers commonly used in practice (Section~\ref{sec: toward.complete}). This nonetheless  leads to a highly nonlinear mean-field interacting particle system. In turn, this system implements, via the continuity equation, a flow map from initial to terminal distributions of particles that we present in Section~\ref{s:WGF}.

\section{Interacting particle system}  \label{sec: ips}

Before writing down the Transformer model, we first provide a brief preliminary discussion to clarify our methodological choice of treating the discrete layer indices in the model as a continuous time variable in Section~\ref{sec: resnets}, echoing previous work on ResNets.
The specifics of the toy Transformer model are presented in Section~\ref{s:interacting}, and a complete model is presented in Section~\ref{sec: toward.complete}.

\subsection{Residual neural networks} \label{sec: resnets}

One of the standard paradigms in machine learning is that of supervised learning, where one aims to approximate an unknown function $f:\R^d\to\R^m$, from data, $\mathscr{D} = \{x^{(i)}, f(x^{(i)})\}_{i\in[N]}$ say. 
This is typically done by choosing one among an arsenal of possible parametric models, whose parameters are then fit to the data by means of minimizing some user-specified cost. With the advent of graphical processing units (GPUs) in the realm of computer vision \cite{krizhevsky2012imagenet}, large neural networks have become computationally accessible, resulting in their popularity as one such parametric model.

Within the class of neural networks, \emph{residual neural networks} (ResNets for short) have become a staple DNN architecture
since their introduction in \cite{he2016identity}.
In their most basic form, ResNets approximate a function $f$ at $x \in\R^d$ through a sequence of affine transformations, a component-wise nonlinearity, and skip connections. Put in formulae, 
\begin{equation} \label{eq: resnet}
    \begin{dcases}
        x(k+1) = x(k) + w(k)\sigma(a(k)x(k)+b(k)) &\text{ for } k\in\{0, \ldots, L-1\}\\
        x(0) =x \,.
    \end{dcases}
\end{equation}
Here $\sigma$ is a Lipschitz function applied component-wise to the input vector, while $\theta(\cdot) = (w(\cdot), a(\cdot), b(\cdot))\in\R^{d\times d}\times\R^{d\times d}\times\R^{d}$ are trainable parameters. We say that~\eqref{eq: resnet} has $L\geq1$ hidden layers (or $L+1$ layers, or is of depth $L$).
The output $x_i(L)$ serves as a \emph{representation} of the input $x^{(i)}$ that is then fed into a last layer that corresponds to a classical learning task such as linear or logistic regression in order to predict the label $f(x^{(i)})$.
One can also devise generalizations of~\eqref{eq: resnet}, for instance in which matrix-vector multiplications are replaced by discrete convolutions in order to reflect other common architectures such as convolutional neural networks~\cite[Chapter~9]{GooBenCou16}.
The key element that all these models share is that they all have \emph{skip-connections}, namely, the previous step $x_i(k)$ appears explicitly in the iteration for the next one.

One upside of~\eqref{eq: resnet}, which is the one of interest to our narrative, is that the layer index $k$ can naturally be interpreted as a time variable, motivating the continuous-time analogue
\begin{equation} \label{eq: neural.ode}
    \begin{dcases}
        \dot{x}(t)=w(t)\sigma(a(t)x(t)+b(t)) &\text{ for } t\in(0, T)\\
        x(0) = x.
    \end{dcases}
\end{equation}
These are dubbed \emph{neural ordinary differential equations} (neural ODEs). Since their introduction in \cite{chen2018neural,weinan2017proposal,haber2017stable}, neural ODEs have emerged as a flexible mathematical framework to implement and study ResNets. We use this abstraction here for convenience, as dynamical systems are more naturally defined in continuous time. Although this approach is helpful for exposition, we emphasize that all the results presented can also be derived in discrete time using the same techniques.

\subsection{The interacting particle system} \label{s:interacting} 
Unlike ResNets, which operate on a single input vector $x(0)\in\mathbb{R}^d$ at a time, Transformers operate on a sequence of vectors of length $n$, namely, $(x_i(0))_{i\in[n]}\in(\R^d)^n$. This perspective is rooted in natural language processing, where each vector represents a word, and the entire sequence a sentence or a paragraph. In particular, it allows to process words together with their context.  A sequence element $x_i(0)\in\mathbb{R}^d$ is called a \emph{token}, and the entire sequence $(x_i(0))_{i\in[n]}$ a \emph{prompt}. We use the words ``token'' and ``particle'' interchangeably.

All practical implementations of Transformers make use of \emph{layer normalization} \cite{ba2016layer}, most commonly in the form of \emph{root mean square (RMS) normalization} \cite{zhang2019root}. 
(See lines 105--116 in \cite{mistralai} for instance.)
RMS normalization takes the sequence of tokens output after each layer, divides each token by its Euclidean norm\footnote{The original form instead consisted in an entry-wise standardization of every token, namely subtracting the mean of all tokens, then dividing by the standard deviation.} (plus a small parameter to avoid a possible division by zero), and multiplies the result by a trained diagonal matrix. This process aims to ensure that tokens do not diverge, thus avoiding rounding errors and overflow. The result is an evolution on a time-varying axis-aligned ellipsoid. To simplify the presentation and obtain insight and precise results, we assume that the trained diagonal matrix is equal to the idenitity, so that we work on the unit sphere $\S$ throughout. This simplification is justified empirically in the trained {\sf ALBERT XLarge v2} model described in Figure~\ref{fig: albert}, wherein this diagonal matrix is constant over all layers, with entries of mean value equal to $0.44$ and standard deviation equal to $0.008$. 
Furthermore, current text embeddings provided by OpenAI, namely {\tt text-embedding-3-small} and {\tt text-embedding-3-large}, return norm-one embedding vectors. While we can only speculate as to the actual implementation of these models, this is an indication that layer normalization could be as simple as the one used in our toy model.

A Transformer is then a flow map on $(\S)^n$: the input sequence  $(x_i(0))_{i\in[n]}\in(\S)^n$ is an initial condition which is evolved through the dynamics
\begin{equation} \label{eq: transformerSd.QKV}
    \dot{x}_i(t) =\proj_{x_i(t)} \left(\frac{1}{Z_{\beta,i}(t)}\sum_{j=1}^n e^{\beta\langle Q(t)x_i(t), K(t)x_j(t)\rangle}V(t)x_j(t)\right)
\end{equation}
for all $i\in[n]$ and $t\geq0$. (We refer the reader to \eqref{eq: albert} and Section \ref{sec: toward.complete} for the full model.)
Here and henceforth $$\proj_{x}y=y-\langle x,y\rangle x$$ 
denotes the projection of $y\in\R^d$ onto $\Tan_x\S$. The \emph{partition function} $Z_{\beta, i}(t)>0$ reads
\begin{equation} \label{eq: SA.QKV}
Z_{\beta, i}(t) = \sum_{k=1}^n e^{\beta\langle Q(t)x_i(t), K(t)x_k(t)\rangle}. 
\end{equation}
where $(Q(\cdot),K(\cdot),V(\cdot))$ (standing for Query, Key, and Value) are parameter matrices learned from data, and $\beta>0$ a fixed number intrinsic to the model\footnote{In practical implementations the inner products are multiplied by $d^{-\frac12}$, which along with the typical magnitude of $Q, K$ leads to the appearance of $\beta$.}, which, can be seen as an inverse temperature using terminology from statistical physics. Note that $Q(\cdot), K(\cdot)$ need not be square. 

The interacting particle system~\eqref{eq: transformerSd.QKV}--\eqref{eq: SA.QKV}, a simplified version of which was first written down in \cite{lu2019understanding, dutta2021redesigning, sander2022sinkformers}, importantly contains the true novelty that Transformers carry with regard to other models: the \emph{self-attention mechanism}
\begin{equation} \label{eq:P}
 A_{ij}(t):=\frac{e^{\beta\langle Q(t)x_i(t), K(t)x_j(t)\rangle}}{Z_{\beta, i}(t)}\,, \hspace{1cm} (i,j)\in[n]^2,
\end{equation}
which is the nonlinear coupling mechanism in the interacting particle system.
The $n \times n$ stochastic matrix $A(t)$ (rows are probability vectors) is called the \emph{self-attention matrix}. The wording \emph{attention} stems from the fact that $A_{ij}(t)$ captures the attention given by particle $i$ to particle $j$ relatively to all particles $\ell\in[n]$. In particular, a particle pays attention to its neighbors where neighborhoods are dictated by the matrices $Q(t)$ and $K(t)$ in~\eqref{eq:P}. 

It has been observed numerically that the probability vectors $(A_{ij}(\cdot))_{j\in[n]}$ (for $i\in[n]$) in a trained self-attention matrix exhibit behavior related to the syntactic and semantic structure of sentences in natural language processing tasks (see \cite[Figures 3-5]{vaswani2017attention}). 
To illustrate our conclusions as pedagogically as possible, 
throughout the paper we focus on a simplified scenario wherein the parameter matrices $(Q, K, V)$ are constant, and even all equal to the identity unless stated otherwise, resulting in the dynamics 
\begin{equation} \label{SA}
    \dot{x}_i(t) =\proj_{x_i(t)} \left(\frac{1}{Z_{\beta,i}(t)}\sum_{j=1}^n e^{\beta\langle x_i(t), x_j(t)\rangle}x_j(t)\right) \tag{SA}
\end{equation}
for $i\in[n]$ and $t\geq0$ and, as before
\begin{equation} \label{eq: SA}
Z_{\beta, i}(t) = \sum_{k=1}^n e^{\beta\langle x_i(t), x_k(t)\rangle}.
\end{equation}
The appearance of clusters in Transformers is actually corroborated by numerical experiments with trained models (see Figure~\ref{fig: albert}). 
While we focus on a much simplified model, numerical evidence shows that the clustering phenomenon looks qualitatively the same in the cases $Q=K=V=\lambda I_d$, $\lambda>0$, and generic random $(Q,K,V)$ (see Figures~\ref{fig: phase.diag.Id} and~\ref{fig: random.QKV} for instance). We refer the interested reader directly to Sections~\ref{sec: temperature}, \ref{sec: large.beta}, \ref{sec: high.d}; here, we continue the presentation on the modeling of different mechanisms appearing in the Transformer architecture.

\begin{figure}[h!]
    \centering
    \includegraphics[scale=0.25]{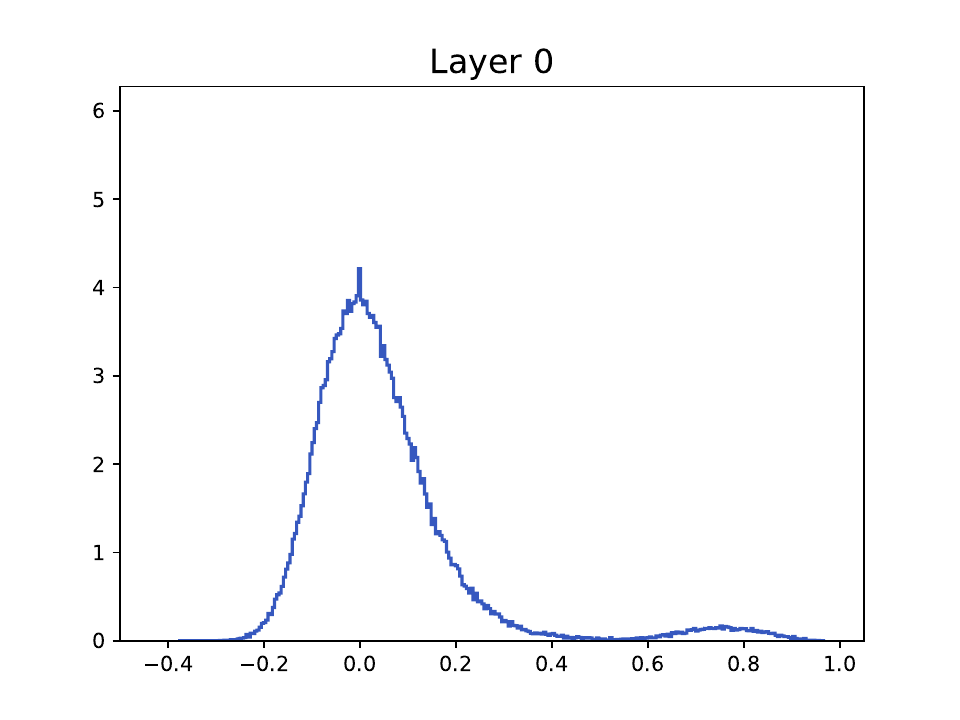}
    \includegraphics[scale=0.25]{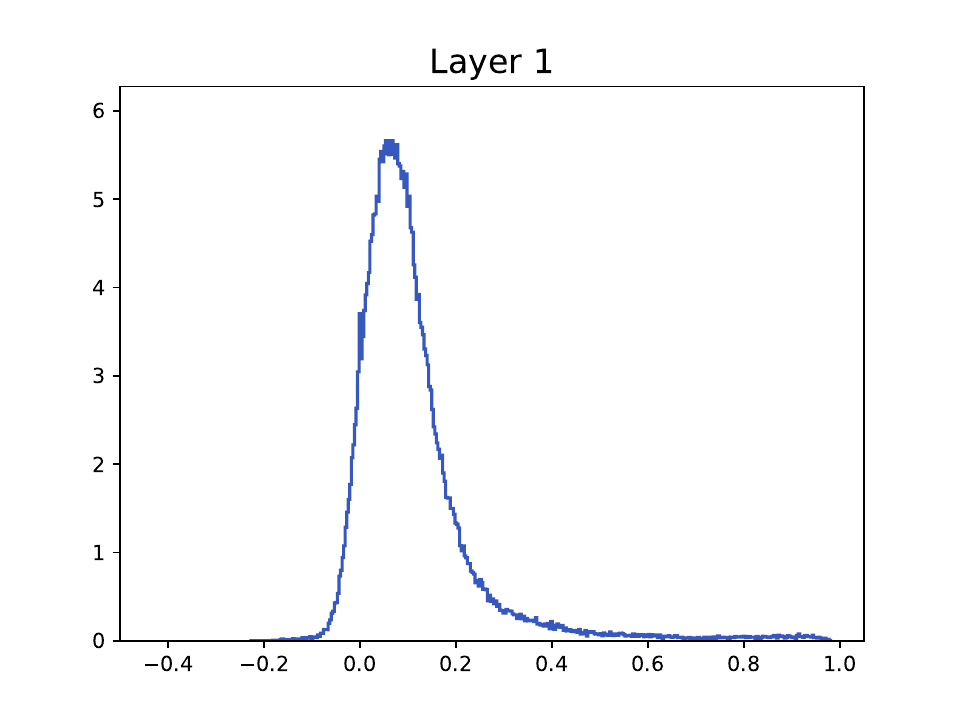}
    \includegraphics[scale=0.25]{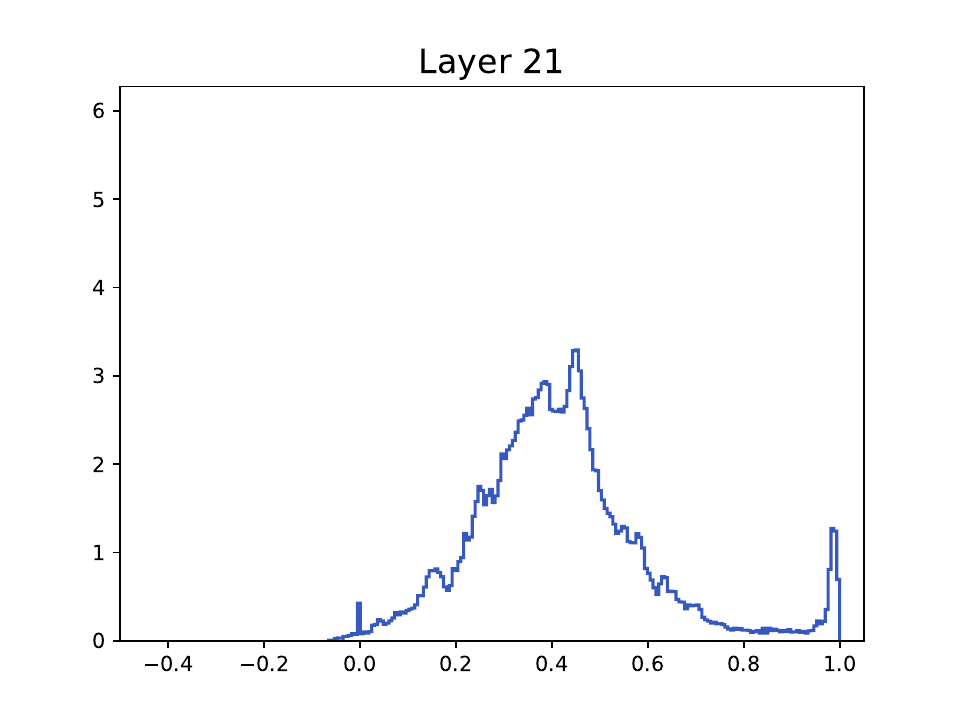}
    \includegraphics[scale=0.25]{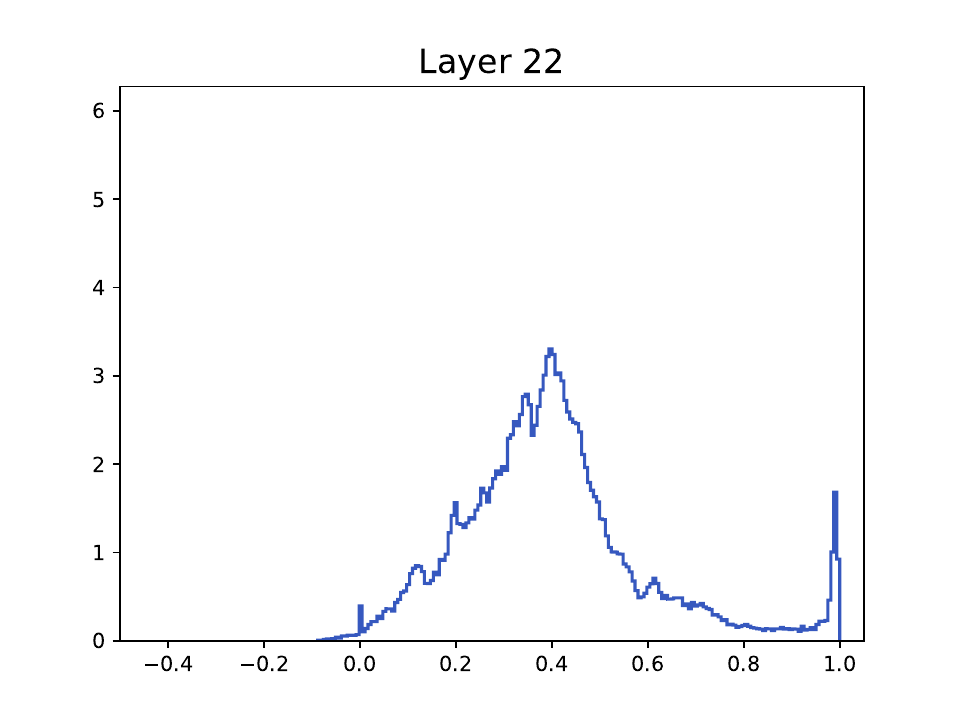}
    \includegraphics[scale=0.25]{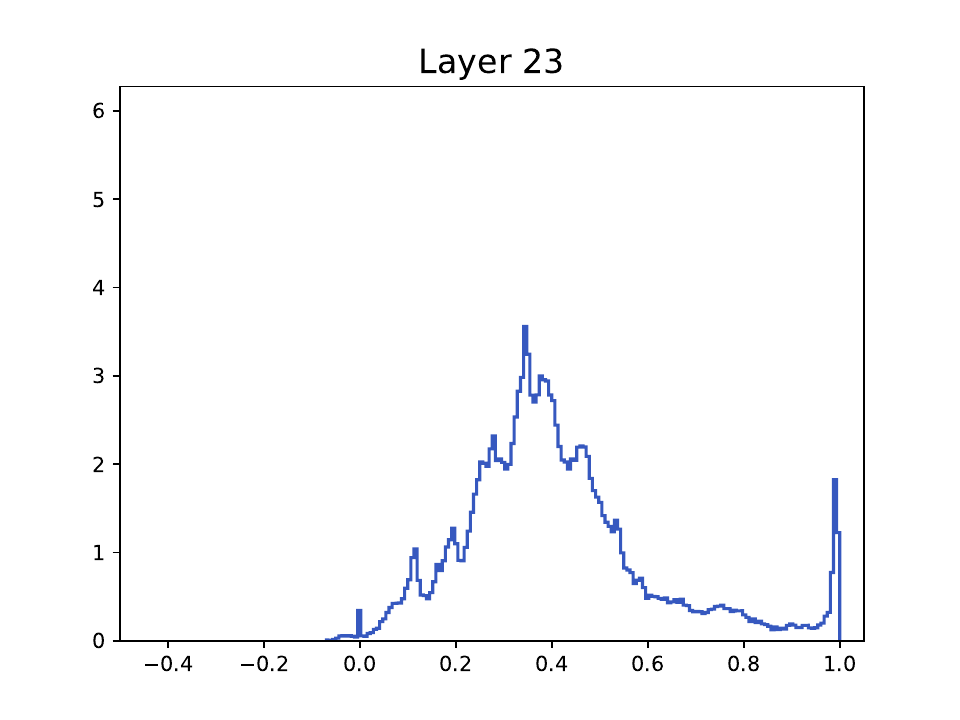}
    \includegraphics[scale=0.25]{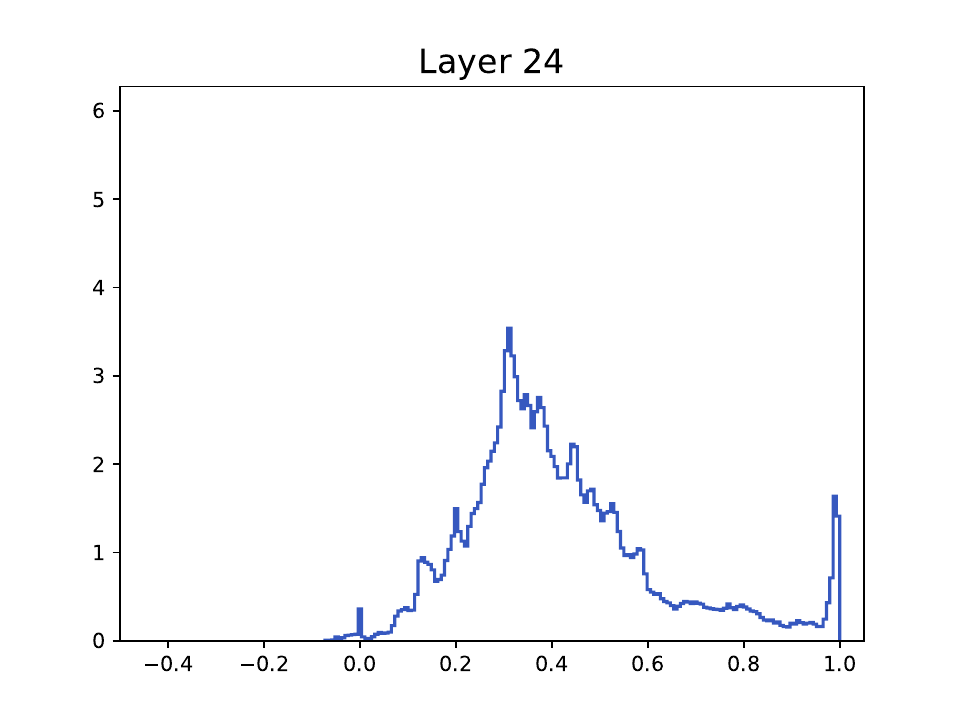}
    \caption{Histogram of $\{\langle x_i(t),x_j(t)\rangle\}_{(i,j)\in[n]^2, i\neq j}$ at different layers $t$ in the context of the trained {\sf ALBERT XLarge v2} model (\cite{lanalbert} and \href{https://huggingface.co/albert-xlarge-v2}{\tt https://huggingface.co/albert-xlarge-v2})\protect\footnotemark, which has constant parameter matrices. Here we randomly selected a single prompt, which in this context is a paragraph from a random Wikipedia entry, and then generate the histogram of the pairwise inner products. We see the progressive emergence of clusters all the way to the $24$th (and last) hidden layer (top), as evidenced by the growing mass at $1$. If the number of layers is increased, up to 48 say, the clustering is further enhanced (bottom).
    }
    \includegraphics[scale=0.25]{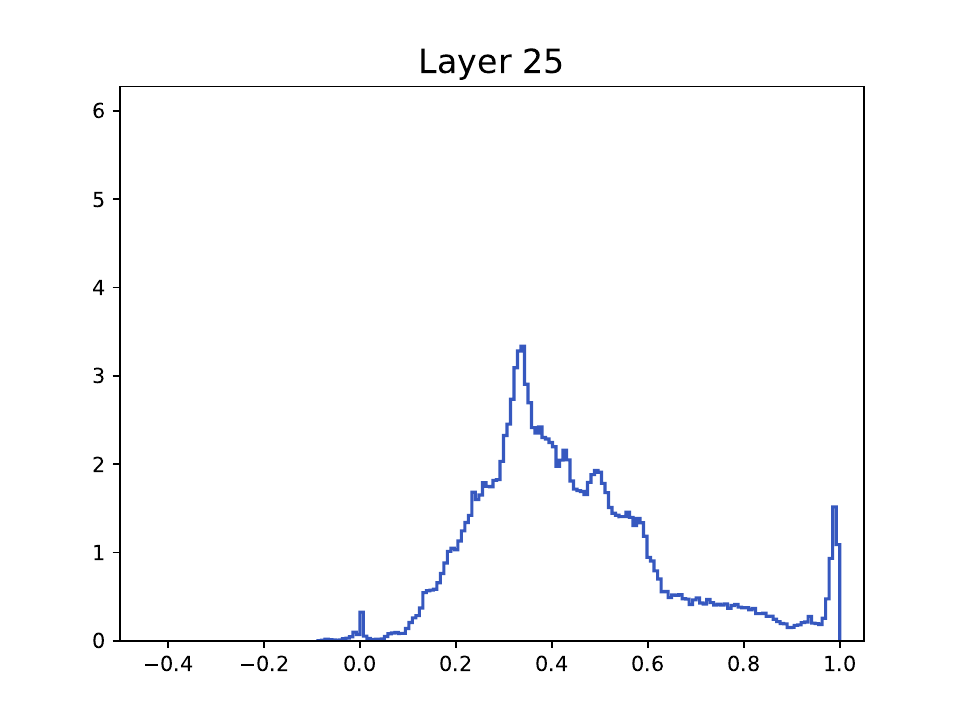}
    \includegraphics[scale=0.25]{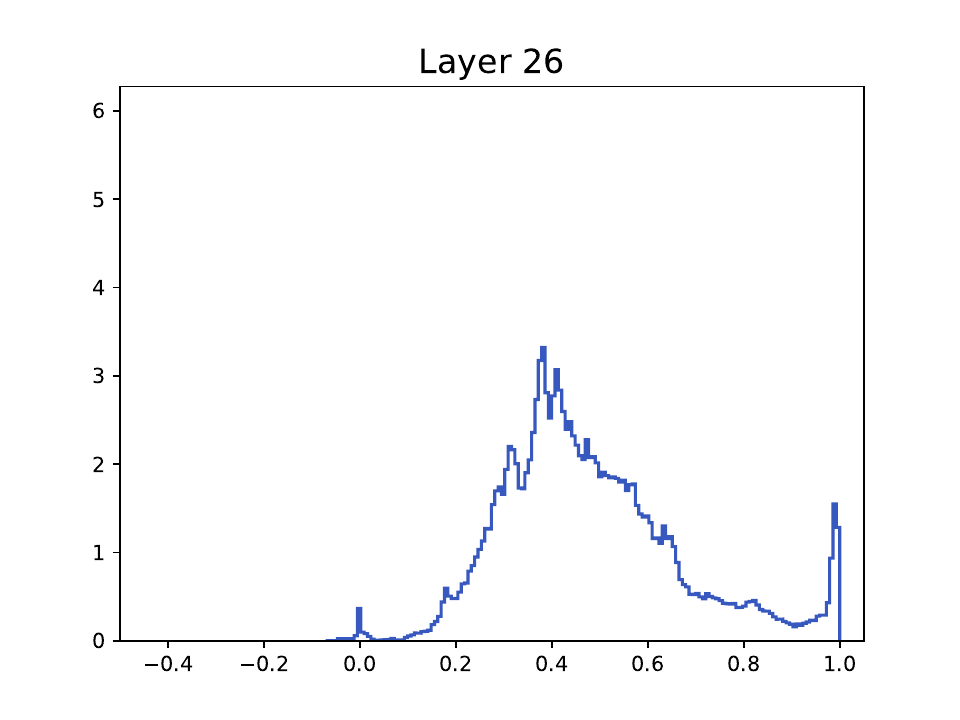}
    \includegraphics[scale=0.25]{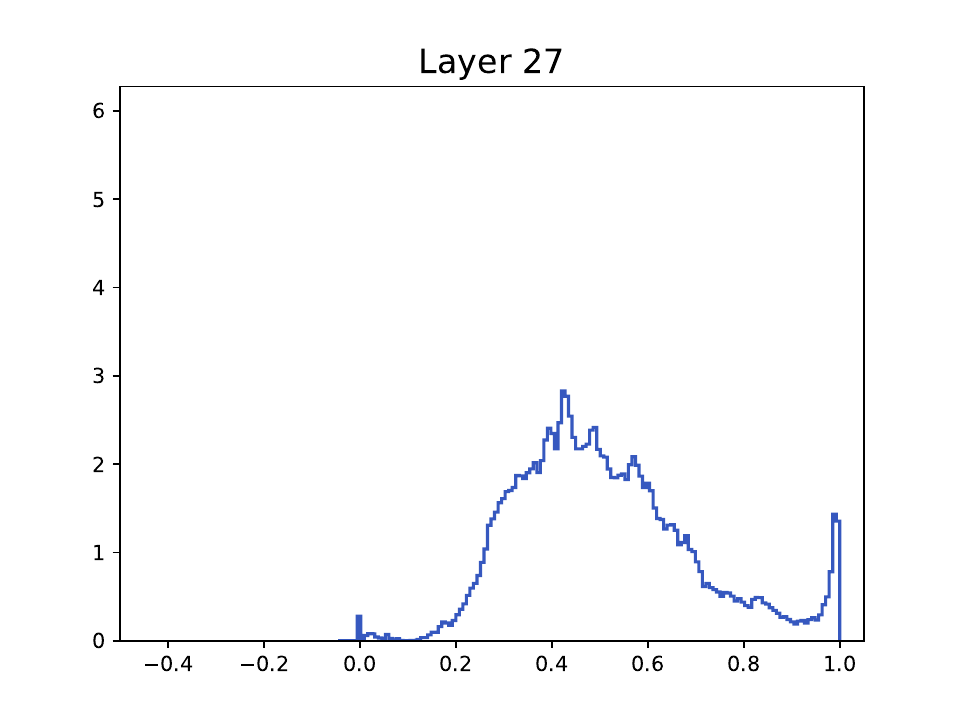}
    \includegraphics[scale=0.25]{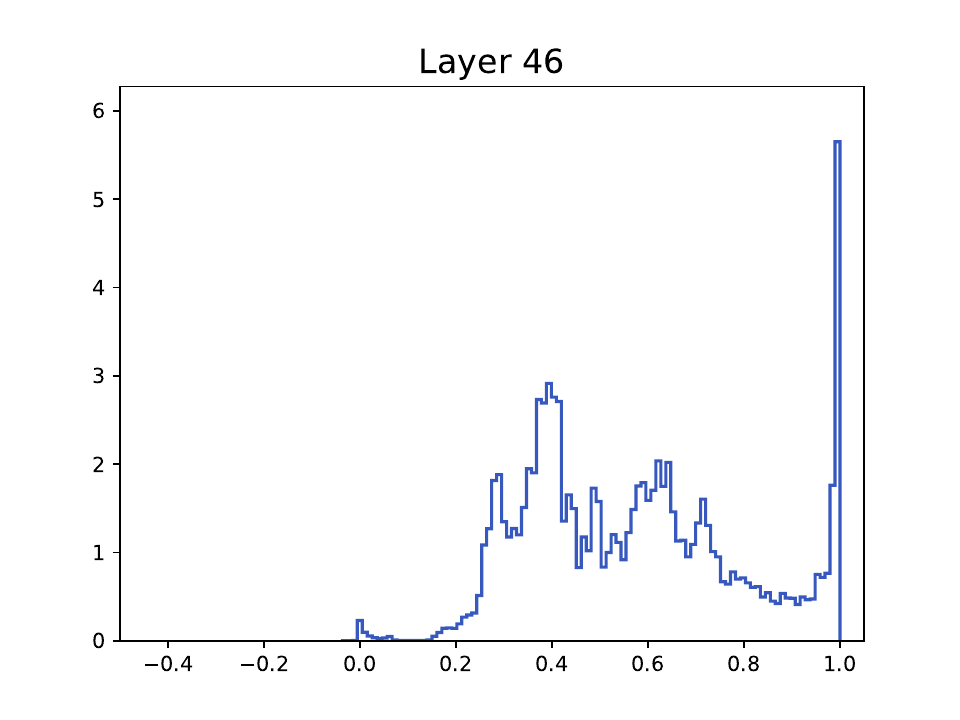}
    \includegraphics[scale=0.25]{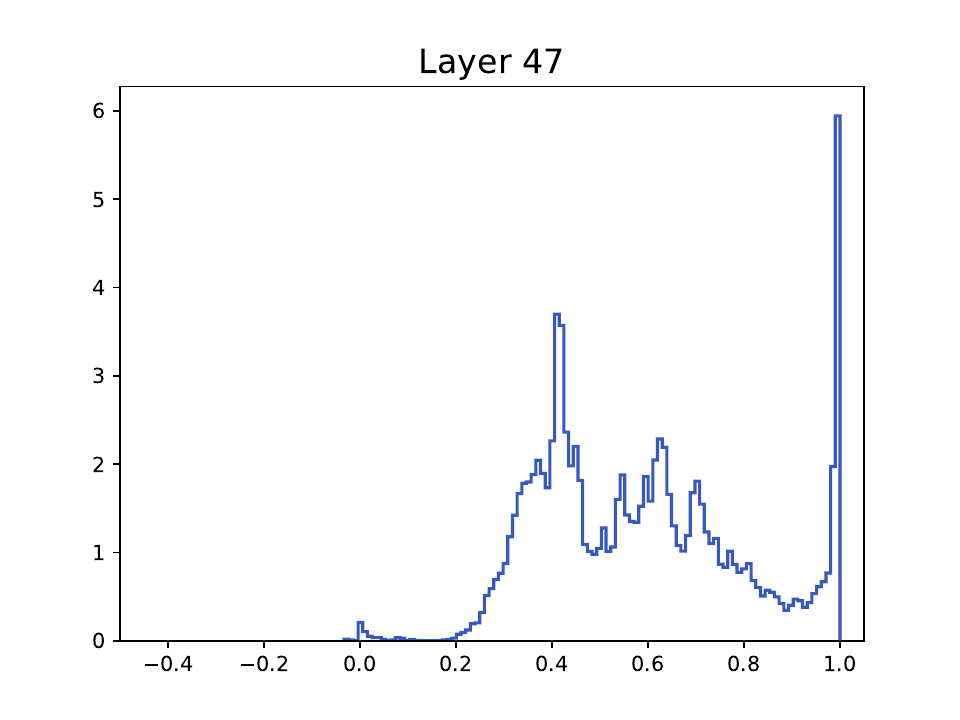}
    \includegraphics[scale=0.25]{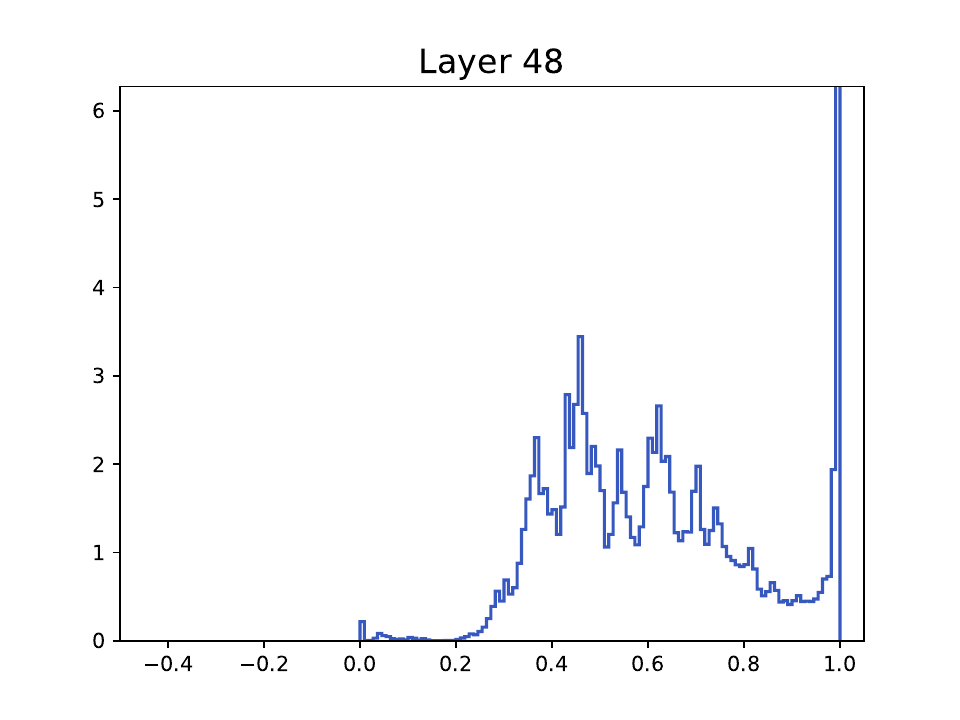}
    \label{fig: albert}
\end{figure}

\footnotetext{\textsf{ALBERT XLarge v2} contains all the mechanisms described in this text, namely, is a system of the form~\eqref{eq: albert} {(or rather the discretization thereof)} with $12$ or $24$ layers. The sequence length $n$ is of the order of $512$ or $1024$, and the tokens evolve in $\R^{4096}$. The dynamics are therefore high-dimensional, lending weight to assumptions made later on (Section~\ref{sec: high.d}).
}

\begin{remark}[Collective behavior]
The dynamics~\eqref{SA} have a strong resemblance to the vast literature on nonlinear systems arising in the modeling of collective behavior. In addition to the connection to the classical Kuramoto model describing synchronization of oscillators \cite{kuramoto1975self, acebron2005kuramoto} (made evident
in Section~\ref{s:kuramoto}), Transformers are perhaps most similar to the Krause model \cite{krause2000discrete}
$$
\dot{x}_i(t) = \sum_{j=1}^n a_{ij}(x_j(t)-x_i(t)), \hspace{1cm} a_{ij} = \frac{\phi(\|x_i-x_j\|^2)}{\sum_{k=1}^n \phi(\|x_i-x_k\|^2)}.
$$
which is non-symmetric in general ($a_{ij}\neq a_{ji}$), much like~\eqref{eq: transformerSd.QKV}. When $\phi$ is compactly supported, it has been shown in \cite{jabin2014clustering} that the particles $x_i(t)$ assemble in several clusters as $t\rightarrow +\infty$. Other related models include those of Vicsek \cite{vicsek1995novel}, Hegselmann-Krause \cite{rainer2002opinion} and Cucker-Smale \cite{cucker2007emergent}.  All these models exhibit a clustering behavior under various assumptions (see \cite{motsch2014heterophilious, tadmor2023swarming} and the references therein).  Yet, none of the opinion dynamics models discussed above contain parameters appearing within the nonlinearity as in~\eqref{SA}, {whilst set on the sphere}. 
\end{remark}

\begin{remark}[Permutation equivariance]
A function $f:(\S)^n\rightarrow (\S)^n$ is permutation equivariant if $f(\pi X) = \pi(f_1(X),\ldots,f_n(X))$ for any $X\in(\R^d)^n$ and for any permutation $\pi\in\mathbf{S}_n$ of $n$ elements. Otherwise put, if we permute the input $X$, then the output $f(X)$ is permuted in the same way.
Given $t>0$, the Transformer~\eqref{SA}, mapping $(x_i(0))_{i\in[n]} \mapsto (x_i(t))_{i\in[n]}$, is  permutation-equivariant on $(\S)^n$. 
\end{remark}

\subsection{Toward the complete Transformer} \label{sec: toward.complete}
There are a couple of additional mechanisms used in practical implementations that we do not explicitly address or use in this study.  The mathematical analysis of these mechanisms remains open.

\subsubsection{Multi-headed attention} 
Practical implementations spread out the computation of the self-attention mechanism at every $t$ through a sequence of \emph{heads}, leading to the so-called \emph{multi-headed self attention}. This consists in considering the following modification to~\eqref{SA}:
\begin{equation} \label{eq: multihead}
\dot{x}_i(t) = \proj_{x_i(t)}\left(\sum_{h=1}^H\sum_{j=1}^n \frac{e^{\beta\langle Q_h(t) x_i(t), K_h(t)x_j(t)\rangle}}{Z_{\beta, i, h}(t)} V_h(t) x_j(t)\right)
\end{equation}
where $Z_{\beta, i, h}(t)$ is defined as in~\eqref{eq: SA.QKV} for the matrices $Q_h(t)$ and $K_h(t)$. The integer $H\geq1$ is called the number of heads\footnote{In practical implementations, $H$ is a divisor of $d$, and the query and key matrices $Q_h(t)$ and $K_h(t)$  are $\frac{d}{H}\times d$ rectangular. This allows for further parallelization of computations and increased expressiveness. For mathematical purposes, we focus on working with arbitrary integers $H$, and square weight matrices $Q_h$ and $K_h$.}.

The introduction of multiple heads also allows for drawing some interesting parallels with the literature on feed-forward neural networks, such as ResNets~\eqref{eq: resnet}. Considerable effort has been expended to understand $2$-layer neural networks with width tending to $+\infty$; more precisely, consider~\eqref{eq: resnet} with $L=1$, $w\in\mathbb{R}^{d\times \ell}$, $a\in\mathbb{R}^{\ell\times d}$, and $\ell\to+\infty$.
The infinite-width limit for Transformers is in fact very natural, as it is realized by stacking an arbitrary large number of heads: $H\to +\infty$.
Hence, the same questions as for $1$-hidden layer neural networks may be asked: for instance {the question of universal approximation}, in the vein of  \cite{cybenko1989approximation, barron1993universal}.

\subsubsection{Feed-forward layers}\label{sec:feedforward} The complete Transformer dynamics combines all of the above mechanisms with a feed-forward layer; in the discrete-time context, this is actually done by using a Lie-Trotter splitting scheme for  
\begin{equation} \label{eq: albert}
    \dot{x}_i(t) = \proj_{x_i(t)}\left(\sum_{h=1}^H \sum_{j=1}^n \frac{e^{\beta\langle Q_h(t)x_i(t), K_h(t)x_j(t)\rangle}}{Z_{\beta, i, h}(t)}V_h(t)x_j(t)+w(t)\sigma(a(t)x_i(t)+b(t))\right),
\end{equation}
where $w(t), a(t), b(t)$ and $\sigma$ are all as in~\eqref{eq: neural.ode}. The interested reader is referred to \cite{lu2019understanding, phuong2022formal}
for all the details\footnote{and lines 123--130 in \cite{gpt2} for some relevant source code.}. 
The feed-forward layers (convolutional layers can alternatively be considered) are of critical importance in applications and drive the existing results on approximation properties of Transformers~\cite{yun2019transformers}. 
Nevertheless, the analysis of this model is beyond the scope of our current methods.

\section{Measure to measure flow map} \label{s:WGF}

An important aspect of Transformers is that they are not hard-wired to take into account the order of the input sequence, contrary to other architectures used for natural language processing such as recurrent neural networks. In these applications, each token $x_i(0)\in\R^d$ contains not only a word embedding $w_i\in\R^d$, but also an additional \emph{positional encoding} (we postpone a discussion to Remark~\ref{rem: pos.enc}) which allows tokens to also carry their position in the input sequence. Therefore, an input sequence is perfectly encoded as a \emph{set} of tokens $\{x_1(0), \ldots, x_n(0)\}$, or equivalently as the empirical measure of its constituent tokens $\frac{1}{n} \sum_{i=1}^n \delta_{x_i(0)}$. Recall that the output of a Transformer is also a probability measure, namely $\frac{1}{n} \sum_{i=1}^n \delta_{x_i(t)}$, albeit one that captures the likelihood of the next token. As a result, one can view Transformers as flow maps between probability measures\footnote{See \cite{de2019stochastic, vuckovic2020mathematical, zweig2021functional} for further related work on neural networks acting on probability measures.} on $\S$.
To describe this flow map, we appeal to the continuity equation, which governs precisely the evolution of the empirical measure of particles subject to dynamics. This perspective is already present in~\cite{sander2022sinkformers}, the only modification here being that we add the projection on the sphere arising from layer normalization. 

After introducing the continuity equation in Section \ref{sec: cont.eq}, we show that a particular interaction energy functional, which is maximized at any point mass, increases along solutions thereof in Section \ref{sec: interaction.energy}. Motivated by this monotonicity property, in Section \ref{sec: usa} we propose an illustrative modified model which has the nice property of being a Wasserstein gradient flow for this energy. 
Finally, in Section \ref{sec: new.gradient.flow}, we demonstrate that the original equation presented in Section \ref{sec: cont.eq} is itself a gradient flow for the same energy, upon changing the metric underlying the definition of the gradient.

\subsection{The continuity equation} \label{sec: cont.eq}

The vector field driving the evolution of a single particle in~\eqref{SA} clearly depends on all $n$ particles. In fact, one can equivalently rewrite the dynamics as 
\begin{equation} \label{eq: mean.field.ips}
    \dot{x}_i(t) = \mathcal{X}[\mu(t)](x_i(t))
\end{equation}
for all $i\in[n]$ and $t\geq0$, where 
$$
\mu(t,\cdot)=\frac{1}{n}\sum_{i=1}^n \delta_{x_i(t)}(\cdot)
$$ 
is the empirical measure, while the vector field $\mathcal{X}[\mu]:\S \to \Tan\S$ reads
\begin{equation} \label{eq: vfSd}
    \mathcal{X}[\mu](x) = \proj_{x}\left(\frac{1}{Z_{\beta,\mu}(x)}\int e^{\beta\langle x, y\rangle}y\diff\mu(y)\right)
\end{equation}
with
\begin{equation} \label{eq: partition.function}
 Z_{\beta,\mu}(x)=\int e^{\beta\langle x,y\rangle}\diff\mu(y).   
\end{equation} 
In other words,~\eqref{SA} is a \emph{mean-field interacting particle system}. The evolution of $\mu(t)$ is governed by the continuity equation\footnote{Unless stated otherwise, $\nabla$ and $\dive$ henceforth stand for the spherical gradient and divergence respectively, and all integrals are taken over $\S$.} 
\begin{equation} \label{eq: conteqSd}
    \begin{cases}
        \partial_t \mu + \dive(\mathcal{X}[\mu]\mu) = 0 &\text{ on } \R_{\geq0}\times\S\\
        \mu_{|t=0} = \mu(0) &\text{ on } \S
    \end{cases}
\end{equation}
satisfied in the sense of distributions.

\begin{remark}[Well-posedness]
    Global existence of weak, measure-valued solutions to~\eqref{eq: conteqSd} for arbitrary initial conditions $\mu(0)\in\mathcal{P}(\S)$ follows by arguing exactly as in \cite[Lemma A.3]{geshkovski2023emergence}. Here and henceforth, $\mathcal{P}(\S)$ stands for the set of Borel probability measures on $\S$.
\end{remark}

\begin{remark}[Positional encoding] \label{rem: pos.enc}
For the sake of completeness, in this brief segue we discuss a few ways to perform positional encoding. The original one, proposed in \cite{vaswani2017attention}, proceeds as follows. Consider a sequence $(w_i)_{i\in[n]}\in(\R^d)^n$ of word embeddings. Then the positional encoding $p_i\in\R^d$ of the $i$-th word embedding is defined as $(p_i)_{2k} = \sin(\frac{i}{M^{2k/d}})$ and $(p_i)_{2k+1} = \cos(\frac{i}{M^{2k/d}})$ for $k\in[d/2-1]$, and $M>0$ is a user-defined scalar equal to $10^4$ in \cite{vaswani2017attention}. The $i$-th token is then defined as the addition: $x_i(0) = w_i+p_i$. Subsequent works simply use either a random\footnote{This rationale supports the assumption that initial tokens are drawn at random, which we make use of later on.} positional encoding (i.e., $p_i$ is just some random vector) or a trained transformation.
The addition can also be replaced with a concatenation $x_i(0) = [w_i; p_i]$. (See \cite{lin2022survey, xiao2023introduction} for  details.)
\end{remark}

\begin{remark}[Mean field limit]
Although the analysis in this paper is focused on the flow of the empirical measure, one can also consider~\eqref{eq: conteqSd} for arbitrary initial probability measures $\mu(0)\in\mathcal{P}(\S)$.
Both views can be linked through a mean-field limit-type result, which can be shown by making use of the Lipschitz nature of the vector field $\mathcal{X}[\mu]$. 
The argument is classical and dates back at least to the work of Dobrushin~\cite{dobrushin1979vlasov}. 
Consider an initial empirical measure $\mu_n(0)=\frac1n\sum_{i=1}^n \delta_{x_i(0)}$, and suppose that the points $x_i(0)$ are such that $\lim_{n\to+\infty} W_1(\mu_n(0),\mu(0))=0$ for some probability measure $\mu(0)\in\mathcal{P}(\S)$. (Here $W_1$ denotes 1-Wasserstein distance---see \cite{villani2009optimal} for definitions---.)
Consider the solutions $\mu_n(t)$ and $\mu(t)$ to~\eqref{eq: conteqSd} with initial data $\mu_n(0)$ and $\mu(0)$ respectively. Dobrushin's argument is then centered around the estimate
\begin{equation*} 
     W_1(\mu_n(t),\mu(t))\leq e^{O(1)|t|}W_1(\mu_n(0),\mu(0))
\end{equation*}
for any $t\in\R$, which in the case of~\eqref{eq: conteqSd} can be shown without much difficulty (see \cite[Chapter 4, Section~1]{villani2001limite} or \cite[Section~1.4.2]{golse2016dynamics}). This elementary mean-field limit result has a couple of caveats. First, the time-dependence is exponential. Second, if one assumes that the points $x_i(0)$ are sampled i.i.d. according to $\mu_0$, then $W_1(\mu_n(0),\mu(0))$  converges to zero at rate $n^{-\frac1{d-1}}$ \cite{Dud69,boissard2014mean}, which deteriorates quickly when $d$ grows.
Dimension-free convergence has been established in some cases, for instance by replacing the Wasserstein distance with a more careful choice of metric as in \cite{han2023class, lacker2023hierarchies} or more generally in \cite{SriFukGre12}. 
Similarly, the exponential time-dependence might also be improved, as recent works in the context of flows governed by Riesz/Coulomb singular kernels, with diffusion, can attest \cite{rosenzweig2023global, guillin2021uniform} (see \cite{lacker2023sharp} for a result in the smooth kernel case). 
We do not address this question in further detail here. For more references on this well-established topic, the reader is referred to \cite{villani2001limite, golse2016dynamics, serfaty2020mean} and the references therein.
\end{remark}

\subsection{The interaction energy} \label{sec: interaction.energy}
One can naturally ask whether the evolution in~\eqref{eq: conteqSd} admits some quantities which are monotonic when evaluated along the flow. As it turns out, the \emph{interaction} \emph{energy}
\begin{equation} \label{eq: interaction.energy}
\mathsf{E}_\beta[\mu] = \frac{1}{2\beta}\iint e^{\beta\langle x, x'\rangle}\diff\mu(x)\diff\mu(x')
\end{equation}
is one such quantity. Indeed,
\begin{align} \label{eq: dissipation.softmax}
    \frac{\diff}{\diff t}\mathsf{E}_\beta[\mu(t)] &= \iint\beta^{-1}e^{\beta\langle x,x'\rangle}\diff\partial_t\mu(t,x)\diff\mu(t,x')\nonumber\\
    &=\int \mathcal{X}[\mu(t)](x)\cdot\int \nabla \left(\beta^{-1}e^{\beta\langle x,x'\rangle}\right)\diff\mu(t,x')\diff\mu(t,x)\nonumber\\
    &=\int \Big\|\mathcal{X}[\mu(t)](x)\Big\|^2 Z_{\beta,\mu(t)}(x)\diff\mu(t,x)
\end{align}
for any $t\geq0$ by using integration by parts. Recalling the definition of $Z_{\beta,\mu}(x)$ in~\eqref{eq: partition.function}, we see that $e^{-\beta}\leq Z_{\beta,\mu}(x)\leq e^{\beta}$ for all $x\in\S$. The identity~\eqref{eq: dissipation.softmax} therefore indicates that $\mathsf{E}_\beta$ increases along trajectories of~\eqref{eq: conteqSd}. (Similarly, should $V=-I_d$, the energy $\mathsf{E}_\beta$ would decrease along trajectories.) This begs the question of characterizing the global minima and maxima of $\mathsf{E}_\beta$, which is the goal of the following result.

\begin{proposition} \label{prop: existence.uniqueness.energy} 
    Let $\beta>0$ and $d\geq 2$. The unique global minimizer of $\mathsf{E}_\beta$ over $\mathcal{P}(\S)$ is the uniform measure\footnote{That is, the Lebesgue measure on $\S$, normalized to be a probability measure.} $\sigma_d$. Any global maximizer of $\mathsf{E}_\beta$ over $\mathcal{P}(\S)$ is a Dirac mass $\delta_{x^*}$ centered at some point $x^*\in\S$. 
\end{proposition}

This result lends credence to our nomenclature of the case $V=I_d$ as \emph{attractive}, and $V=-I_d$ as \emph{repulsive}. 
The reader should be wary however that in this result we are minimizing or maximizing $\mathsf{E}_\beta$ among \emph{all} probability measures on $\S$. Should one focus solely on discrete measures, many global minima appear---these are discussed in Section~\ref{sec:repulsive}---. This is one point where the particle dynamics and the mean-field flow deviate. 
We now provide a brief proof of Proposition~\ref{prop: existence.uniqueness.energy} (see \cite{shuotan2017} for a different approach).

\begin{proof}[Proof of \Cref{prop: existence.uniqueness.energy}]
    The fact that any global maximizer is a Dirac mass is easy to see. We proceed with proving the rest of the statement.
    Let $f(t)=e^{\beta t}$.
    The interaction energy then reads
    \begin{equation*}
        \mathsf{E}_\beta[\mu] = \frac12\iint f(\langle x,x'\rangle)\diff\mu(x)\diff \mu(x').
    \end{equation*}
    The proof relies on an ultraspherical (or Gegenbauer) polynomial expansion of $f(t)$:
    \begin{equation*}
        f(t)=\sum_{k=0}^{+\infty}\widehat{f}(k;\lambda)\frac{k+\lambda}{\lambda}C_k^\lambda(t)
    \end{equation*}
    for $t\in[-1,1]$, where $\lambda=\frac{d-2}{2}$, $C_k^\lambda$ are Gegenbauer polynomials, and 
    \begin{equation*}
        \widehat{f}(k;\lambda)=\frac{\Gamma(\lambda+1)}{\Gamma(\lambda+\frac12)\Gamma(\frac12)}\frac{1}{C_k^\lambda(1)}\int_{-1}^1 f(t)C^\lambda_k(t)(1-t^2)^{\lambda-\frac12}\diff t
    \end{equation*}
    where $C_k^\lambda(1)>0$ (see \cite[Section 1.2]{dai2013approximation}). 
    According to \cite[Proposition 2.2]{bilyk2019geodesic}, a necessary and sufficient condition for \Cref{prop: existence.uniqueness.energy} to hold is to ensure that $\widehat{f}(k;\lambda)>0$ for all $k\geq1$.
    To show this, we use the Rodrigues formula \cite[4.1.72]{szeg1939orthogonal}
    \begin{equation*}
        C_k^\lambda(t)= \frac{(-1)^k2^k}{k!}\frac{\Gamma(k+\lambda)\Gamma(k+2\lambda)}{\Gamma(\lambda)\Gamma(2k+2\lambda)}(1-t^2)^{-(\lambda-\frac12)}\left(\frac{\diff}{\diff t}\right)^k(1-t^2)^{k+\lambda-\frac12},
    \end{equation*}
    and the fact that $C_k^\lambda(-t)=(-1)^kC_k^\lambda(t)$ for $t\in[-1,1]$, which in combination with integration by parts yield
    \begin{equation*}
        \int_{-1}^1 t^\ell C_k^\lambda(t)(1-t^2)^{\lambda-\frac12}\diff t\begin{cases}
            >0 &\text{ if } \ell\geq k \text{ and } \ell-k \text{ is even}\\
            =0 &\text{ otherwise }.
        \end{cases}
    \end{equation*}
    We conclude by using the power series expansion of $f$.
\end{proof}

\subsection{A Wasserstein gradient flow proxy} \label{sec: usa}

In view of~\eqref{eq: dissipation.softmax}, one could hope to see the continuity equation~\eqref{eq: conteqSd} as the \emph{Wasserstein gradient flow} of $\mathsf{E}_\beta$, or possibly some other functional (see the seminal papers \cite{otto2001geometry, jordan1998variational}, and \cite{ambrosio2005gradient, villani2009optimal} for a complete treatment). The long-time asymptotics of the PDE can then be analyzed by studying convexity properties of the underlying functional, by analogy with gradient flows in the Euclidean case.

For~\eqref{eq: conteqSd} to be the Wasserstein gradient flow of $\mathsf{E}_\beta$, the vector field $\mathcal{X}[\mu]$ defined in ~\eqref{eq: vfSd} ought to be
the gradient of the first variation $\delta\mathsf{E}_\beta$ of $\mathsf{E}_\beta$. However, notice that $\mathcal{X}[\mu]$ is a logarithmic derivative: 
\begin{equation} \label{eq: logder}
    \mathcal{X}[\mu](x) = \nabla\log\int\beta^{-1}e^{\beta\langle x, y\rangle}\diff\mu(y).
\end{equation}
(This observation goes beyond $Q=K=I_d$ and $V=\pm I_d$ insofar as $Q^\top K=K^\top Q = \pm V$; see~\cite[Assumption~1]{sander2022sinkformers}.  Because of the lack of symmetry, it has been shown in \cite{sander2022sinkformers} that~\eqref{eq: logder} is not 
the gradient of the first variation of a functional.

To overcome this limitation on $\R^d$, thus without layer normalization,~\cite{sander2022sinkformers} propose two ways to "symmetrize" \eqref{eq: conteqSd} that both lead to a Wasserstein gradient flow;~see~\cite[Proposition~2]{sander2022sinkformers}. We focus here on the simplest one which consists in removing the logarithm in~\eqref{eq: logder}, or equivalently to removing the denominator in~\eqref{eq: vfSd}. 
This is one point where working on the unit sphere is useful: otherwise, the equation on $\R^d$ without layer normalization (as considered in \cite{sander2022sinkformers}) is ill-posed for general choices of matrices $V$, due to the fact that the magnitude of the vector field $\mathcal{X}[\mu]$ grows exponentially with the size of the support of $\mu$.
On the contrary, on $\S$ the resulting equation is perfectly well-posed. 

\begin{remark}[Doubly stochastic kernel]
    Considering the Transformer dynamics on $\R^d$, thus without layer normalization, the authors in \cite{sander2022sinkformers} propose an alternative symmetric model: they replace  the self-attention (stochastic) matrix by a doubly stochastic one, generated from the Sinkhorn iteration. This leads to a Wasserstein gradient flow, whereby the resulting attention mechanism is implicitly expressed as a limit of Sinkhorn iterations. Understanding the emergence of clusters for this model is an interesting but possibly challenging question.
\end{remark}

In view of the above discussion, we are inclined to propose the surrogate model
\begin{equation} \label{USA}
\dot{x}_i(t) =  \proj_{x_i(t)} \left(\frac1n \sum_{j=1}^n e^{\beta\langle x_i(t), x_j(t)\rangle}x_j(t)\right), \tag{USA}
\end{equation}
which is obtained by replacing the partition function $Z_{\beta, i}(t)$ by $n$. 
As a matter of fact,~\eqref{USA} presents a remarkably similar qualitative behavior---all of the results we show in this paper are essentially the same for both dynamics---.

The continuity equation corresponding to~\eqref{USA}, namely
\begin{equation} \label{eq: pde.nosoftmaxZ}
\begin{dcases}
    \partial_t \mu(t,x) + \dive\left(\proj_{x}\left(\int e^{\beta\langle x, x'\rangle}x'\diff\mu(t,x')\right)\mu(t, x)\right) = 0 &\\
    \mu_{|t=0}=\mu_0
    \end{dcases}
\end{equation}
for $(t,x)\in\R_{\geq0}\times\S$, can now be seen as a Wasserstein gradient flow for the interaction energy $\mathsf{E}_\beta$ defined in~\eqref{eq: interaction.energy}.

\begin{lemma} 
Consider the interaction energy $\mathsf{E}_\beta:\mathcal{P}(\S)\to\R_{\geq0}$ 
defined in~\eqref{eq: interaction.energy}.
Then the vector field 
\begin{equation*} 
\mathcal{X}[\mu](x) = \proj_{x}\left(\int e^{\beta\langle x, x'\rangle}x'\diff\mu(x')\right)
\end{equation*}
satisfies
\begin{equation}\label{e:XmuE}
\mathcal{X}[\mu](x) = \nabla\delta\mathsf{E}_\beta[\mu](x)
\end{equation}
for any $\mu\in\mathcal{P}(\S)$ and $x\in\S$, where $\delta\mathsf{E}_\beta[\mu]$ denotes the first variation of $\mathsf{E}_\beta$. 
\end{lemma}

We omit the proof which follows from standard Otto calculus~\cite{otto2001geometry},~\cite[Chapter~15]{villani2009optimal},~\cite[Chapter~5]{chewi2024statistical}. 
We can actually write~\eqref{e:XmuE} more succinctly by recalling the definition of the convolution of two functions on $\S$ \cite[Chapter 2]{dai2013approximation}: for any $g\in L^1(\S)$ and $f:[-1,1]\to\R$ such that $t\mapsto(1-t^2)^{\frac{d-3}{2}}f(t)$ is integrable,
\begin{equation*}
(f\ast g)(x) = \int f(\langle x, y\rangle)g(y)\diff\sigma_d(y).
\end{equation*}
This definition has a natural extension to the convolution of a function $f$ (with the above integrability) and a measure $\mu\in\mathcal{P}(\S)$. We can hence rewrite
\begin{equation*}
\mathsf{E}_\beta[\mu]=\frac12\int(\mathsf{G}_\beta\ast\mu)(x)\diff\mu(x)
\end{equation*}
where $[-1,1]\ni\mathsf{G}_\beta(t)=\beta^{-1} e^{\beta t}$, and so
\begin{equation*}
\mathcal{X}[\mu](x) = \nabla(\mathsf{G}_\beta\ast\mu)(x).
\end{equation*}
Thus,~\eqref{eq: pde.nosoftmaxZ} takes the equivalent form 
\begin{equation} \label{eq: aggregation.eq}
\begin{dcases}
    \partial_t \mu(t, x)+\dive\Big(\nabla\big(\mathsf{G}_\beta\ast\mu(t,\cdot)\big)(x)\mu(t,x)\Big)=0 &\text{ for } (t,x)\in\R_{\geq0}\times\S \\
    \mu_{|t=0}=\mu_0 &\text{ for } x\in\S.
\end{dcases}
\end{equation}
The considerations above lead us to the following Lyapunov identity.

\begin{lemma} \label{lem: dissipation}
    
    The solution $\mu\in C^0(\R_{\geq0};\mathcal{P}(\S))$ 
    to~\eqref{eq: pde.nosoftmaxZ} satisfies 
    \begin{equation*} 
        \frac{\diff}{\diff t}\mathsf{E}_\beta[\mu(t)] =    \int\left\|\nabla\Big(\mathsf{G}_\beta\ast \mu(t,\cdot)\Big)(x)\right\|^2\diff\mu(t,x)
    \end{equation*}
    for $t\geq0$. 
\end{lemma}

Interestingly,~\eqref{eq: aggregation.eq} is an \emph{aggregation} equation, versions of which have been studied in great depth in the literature. For instance, clustering in the spirit of an asymptotic collapse to a single Dirac measure located at the center of mass of the initial density $\mu(0,\cdot)$ has been shown for aggregation equations with singular kernels in \cite{blanchet2008infinite, bertozzi2011lp, carrillo2011global}, motivated by the Patlak-Keller-Segel model of chemotaxis. Here, one caveat (and subsequently, novelty) is that~\eqref{eq: aggregation.eq} is set on $\S$ which makes the analysis developed in these references difficult to adapt or replicate. 

\begin{remark}[Particle version] 
\label{rem: particle.gf}
Let us briefly sketch the particle version of the Wasserstein gradient flow~\eqref{eq: pde.nosoftmaxZ}. When $\mu(t)=\frac1n \sum_{i=1}^n \delta_{x_i(t)}$, the interaction energy~\eqref{eq: interaction.energy} takes the form 
$$
\mathsf{E}_\beta(X)=\frac{1}{2\beta n^2} \sum_{i=1}^n\sum_{j=1}^n e^{\beta\langle x_i,x_j\rangle}
$$
where $X=(x_1,\ldots,x_n)\in (\mathbb{S}^{d-1})^n$. Denoting by $\nabla_{X}$ the gradient associated to the standard 
Riemannian metric on $(\mathbb{S}^{d-1})^n$, we get the dynamics
\begin{equation}\label{e:dynonX}
\dot{X}(t)=n \nabla_X \mathsf{E}_\beta(X(t)).
\end{equation}
Indeed, the gradient on $(\mathbb{S}^{d-1})^n$ is simply $\nabla=(\partial_1,\ldots,\partial_n)$ where $\partial_i$ is the gradient in $\mathbb{S}^{d-1}$ acting on the $i$-th copy in $(\mathbb{S}^{d-1})^n$. Therefore
\begin{equation*}
{\partial_i \mathsf{E}_\beta}(X(t))=\frac{1}{\beta n^2} \sum_{j=1}^n {\proj_{x_i(t)}}\left(e^{\beta\langle x_i(t),x_j(t)\rangle}\beta x_j(t)\right)=\frac1n \dot{x}_i(t)    
\end{equation*}
which yields~\eqref{e:dynonX}. 

Note that~\eqref{SA} also corresponds to a gradient flow of the same interaction energy  albeit
with respect to a Riemannian metric on the sphere different from the standard one
(for $n=2$ the two are conformally equivalent). We provide more detail in the following section.
\end{remark}

\subsection{\eqref{SA} is a gradient flow for a modified metric} \label{sec: new.gradient.flow}

We will now briefly demonstrate that for a particular choice of parameters $(Q, K, V)$, the true dynamics \eqref{SA} can be seen as a gradient flow for $\mathsf{E}_\beta$ upon a modification of the metric on the tangent space of $(\S)^n$. This will facilitate qualitative analysis later on by using standard tools from dynamical systems. This insight can then be extrapolated to the corresponding continuity equation \eqref{eq: conteqSd} as well, as seen in Section \ref{sec: case.of.measures}.

\subsubsection{The case of particles}
We suppose that
\begin{equation*}
    Q^\top K \text{ is symmetric}, \hspace{1cm} V=Q^\top K.
\end{equation*}
We define a new metric on $(\S)^n$ as follows. Let $X=(x_1,\ldots,x_n)\in(\S)^n$. Consider the inner product on $\Tan_X(\S)^n$ given by
\begin{equation} \label{e:scalarproduct}
\left\langle(a_1,\ldots,a_n), (b_1,\ldots,b_n)\right\rangle_X = \sum_{i=1}^n Z_{\beta,i}(X) \langle a_i,b_i \rangle\,,
\end{equation}
where $a_i,b_i\in\Tan_{x_i}\S$, and 
$$ Z_{\beta,i}(X) = \sum_{j=1}^n e^{\beta \langle Vx_i, x_j\rangle}.$$
Set 
$$
\mathsf{E}_\beta(X)=\frac{1}{2\beta}\sum_{i=1}^n\sum_{j=1}^n e^{\beta\langle Vx_i,x_j\rangle}.
$$
We now show that the dynamics~\eqref{eq: transformerSd.QKV} can be equivalently written as
$$ \dot{X}(t) = \nabla \mathsf{E}_\beta(X(t)),$$
where the gradient $\nabla$ is computed with respect to the metric \eqref{e:scalarproduct} on $(\S)^n$. To this end, we ought to show that for all vector fields $Y$ on $(\S)^n$ and for all $X\in(\S)^n$, 
\begin{equation}\label{eq:r3}
	\frac{\diff}{\diff t}\Big|_{t=0} \mathsf{E}_\beta\left(\Phi^t_Y(X)\right) = \langle Y(X), B(X)\rangle_X
\end{equation}
holds, where $\Phi^t_Y$ is the flow associated to the vector field $Y$, whereas $B=(B_1,\ldots,B_n)$ with 
$$ B_i = \proj_{x_i} \left(\frac{1}{Z_{\beta,i}(X)}\sum_{j=1}^n e^{\beta\langle Vx_i,x_j\rangle }Vx_j\right) \in \Tan_{x_i}
\S.$$
By linearity, it is sufficient to show \eqref{eq:r3} for vector fields $Y$ of the form
$$ Y(X) = (Ax_1,0,\ldots,0)\in \Tan_X(\S)^n\,$$
where $A$ is an arbitrary non-zero skew-symmetric matrix. 
Clearly
\begin{equation}\label{eq:r4}
\Phi_Y^t(X) = (e^{tA} x_1,x_2,\ldots,x_n).
\end{equation}
One first computes
\begin{equation*}
\frac{\diff}{\diff t}\Big|_{t=0} \mathsf{E}_\beta\left(\Phi^t_Y(X)\right)  = \sum_{j=1}^n e^{\beta\langle Vx_1,x_j\rangle} \langle Ax_1, Vx_j\rangle.
\end{equation*}
Now observe that $\langle Ax_1, y\rangle = \langle Ax_1, z\rangle$ for all skew-symmetric matrices $A$ if and only if $x_1(y^\top -z^\top)$ is a symmetric matrix. Since $\proj_{x_1} = I_d-x_1x_1^\top$, we see that
\begin{equation*}
    \sum_{j=1}^n e^{\beta\langle Vx_1,x_j\rangle} \langle Ax_1,Vx_j\rangle=\langle Y(X), B(X)\rangle_X,
\end{equation*}
as desired.

\subsubsection{The case of measures} \label{sec: case.of.measures}
    The above insight, which consists in reweighing the metric with respect to which the gradient is being taken, can be formally adapted to Wasserstein setting for \eqref{eq: conteqSd}---which we recall, is not a gradient flow for the standard Wasserstein gradient of $\mathsf{E}_\beta$---. But note that \eqref{eq: conteqSd} writes as 
    \begin{equation} \label{eq: first.rewriting}
        \partial_t \mu(t) + \dive\left(\frac{\nabla \delta\mathsf{E}_\beta[\mu(t)]}{\delta\mathsf{E}_\beta[\mu(t)]}\mu(t)\right)=0,
    \end{equation}
    with $\delta\mathsf{E}_\beta[\mu](x) = \int e^{\beta\langle x, y\rangle}\diff \mu(y)$.
    To avoid technicalities we henceforth only focus on the absolutely continuous case. 
    For a fixed $\mu\in\mathcal{P}_{\mathrm{ac}}(\S)$, as in the well-known formal Riemannian reinterpretation of the Wasserstein space using Otto calculus \cite{otto2001geometry}, \cite[Chapter 5]{chewi2024statistical}, we consider
    \begin{equation*}
        \mathrm{T}_{\mu}\mathcal{P}_{\mathrm{ac}}(\S) = \overline{\left\{ \nabla \psi\colon \psi\in C^\infty(\S)\right\}}^{L^2(\mu)},
    \end{equation*}
    which, rather than endowing with the standard formal metric tensor given by the $\dot{H}^1$--inner product 
    \begin{equation*}
        \langle \nabla \psi_1, \nabla \psi_2\rangle_{\mu} := \int \langle \nabla \psi_1(x), \nabla\psi_2(x) \rangle\diff \mu(x),
    \end{equation*}
    we instead endow with
    \begin{equation*}
        \langle \nabla \psi_1, \nabla \psi_2\rangle_{\mu, \mathsf{E}_\beta} := \int \langle \nabla \psi_1(x), \nabla\psi_2(x) \rangle\,\delta\mathsf{E}_\beta[\mu](x)\diff \mu(x).
    \end{equation*}
    Continuing the Riemannian geometry analogy, through this metric tensor we can define a distance between $\mu_0,\mu_1\in\mathcal{P}_{\mathrm{ac}}(\S)$ by solving the variational problem 
    \begin{align*}
    \inf_{(\mu(t),v(t))_{t\in[0,1]}}\Bigg\{\int_0^1 \|v(t)\|^2_{\mu(t),\mathsf{E}_\beta}\diff t&\colon (\mu,v) \text{ satisfy } \eqref{eq: CE}, \mu(0)=\mu_0, \mu(1)=\mu_1
        \Bigg\},
    \end{align*}
    where
    \begin{equation} \label{eq: CE}
    \partial_t\mu(t, x) + \dive(v(t, x)\mu(t, x))=0 \hspace{1cm} \text{ on } [0,1]\times\S.    
    \end{equation}
    This variational problem is a generalization of the celebrated Benamou-Brenier formula \cite{benamou2000computational}, the value of which we dub $W_{2,\mathsf{E}_\beta}^2(\mu_0,\mu_1)$: it is a weighed Wasserstein distance. 
    For a curve of measures $(\mu(t))_{t\geq0}$ with tangent vectors $(v(t))_{t\geq0}$ (meaning they solve \eqref{eq: CE}), the Wasserstein gradient $\gradW_{\mathsf{E}_\beta}\mathsf{E}_\beta$ induced by this geometric setup is then the element of $\mathrm{T}_{\mu(t)}\mathcal{P}_{\mathrm{ac}}(\S)$ such that 
    $$\partial_t \mathsf{E}_\beta[\mu(t)] = \langle \gradW_{\mathsf{E}_\beta}\mathsf{E}_\beta[\mu(t)], v(t)\rangle_{\mu(t), \mathsf{E}_\beta}.$$
    We can now demonstrate that the vector field driving \eqref{eq: first.rewriting} is the Wasserstein gradient of $\mathsf{E}_\beta$ corresponding to this geometric setup. Indeed as in \cite[Definition 5.9]{chewi2024statistical} we first have 
    \begin{equation*}
    \partial_t \mathsf{E}_\beta[\mu(t)] = \int\delta\mathsf{E}_\beta[\mu(t)]\diff \partial_t\mu(t).    
    \end{equation*}
    We then find 
    \begin{equation*}
        \int\delta\mathsf{E}_\beta[\mu(t)]\diff \partial_t\mu(t) = \int \langle \nabla \delta\mathsf{E}_\beta[\mu(t)], v(t)\rangle \diff \mu(t) = \left\langle \frac{\nabla \delta\mathsf{E}_\beta[\mu(t)]}{\delta\mathsf{E}_\beta[\mu(t)]}, v(t)\right\rangle_{\mu(t),\mathsf{E}_\beta}, 
    \end{equation*}
    as desired. The literature studying weighed Wasserstein distances such as the one above is rather scarce, but a relevant reference is \cite{Li21}.

\part{Clustering} \label{part: clustering}

\begin{figure}[h!]
    \centering
    \begin{tikzpicture}[scale=1.2]
        \fill[fill=red, opacity=0.65] (2, 0.85) rectangle (2.5, 3.75);
        \fill[fill=green, opacity=0.4] (2, 0) rectangle (2.5, 0.85);
        \fill[fill=green, opacity=0.4] (2, 3.75) rectangle (2.5, 4.5);
        
        \fill[fill=green, opacity=0.4] (2.5, 0) -- (10, 0) -- (10, 4.5) -- (2.5, 4.5) -- cycle;
        
        \draw (2, 0.1) -- (2, -0.1) node[below] {$2$};
        \draw (2.5, 0.1) -- (2.5, -0.1) node[below] {$3$};
        \draw (5, 0.1) -- (5, -0.1) node[below] {$n$};
        \draw (9, 0.1) -- (9, -0.1) node[below] {$d^*$};
        \node at (6, -0.5) {$d$};
        
        \draw (1.95, 0.85) -- (2.1, 0.85) node[left] {$1\hspace{0.15cm}$};
        \draw (1.95, 3.75) -- (2.1, 3.75) node[left] {$\frac{n^2}{\pi^2}\hspace{0.15cm}$};
        \node at (1.45, 2) {$\beta$};

        \node[rotate=90] at (1, 0.15) {fast};
        \node[rotate=90] at (1.05, 0.6) {$\leftarrow$};
        
        \node[rotate=90] at (1.05, 1.2) {$\rightarrow$};
        \node[rotate=90] at (1, 1.7) {slow};



        \node[rotate=90] at (2.215, 2.25) {\footnotesize no result when $d=2$};




        
        \draw[dashed, thick] (2, 0) -- (2, 4.5);
        \draw[dashed, thick] (2.5, 0) -- (2.5, 4.5);
        \draw[dashed, thick] (5, 0) -- (5, 4.5);
        \draw[dashed, thick] (9, 0) -- (9, 4.5);
        
        \draw[thick, ->] (2, 0) -- (10,0) node[right] {};
        \draw[thick, ->] (2, -0.1) -- (2, 4.5) node[above] {};
    \end{tikzpicture}
    \caption{Green zones indicate regimes where convergence to a single cluster as $t\to+\infty$ can be proven. Here $n\geq 2$ is fixed. When $d$ is larger than specific thresholds, the long-time asymptotics can be chiseled out in finer detail. Convergence is slow when $\beta\gg1$ (relative to the size of $d, n$), as even the exponential decay constant when $d\geq n$ is of the form $\lambda = O(e^{-\beta})$ and thus degenerates. 
    One rather expects dynamic metastability. 
    Section \ref{sec: temperature} addresses the case where $\beta$ is small and Section \ref{sec: large.beta} where it's large, whereas Section \ref{sec: high.d} covers the high-dimensional case at arbitrary $\beta$.}
    \label{fig: clustering.graph}
\end{figure}
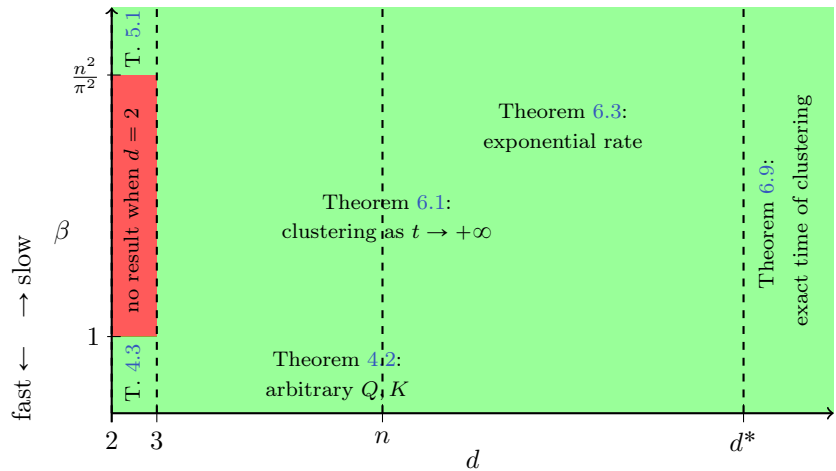

As alluded to in the introductory discussion, clustering is of particular relevance in tasks such as sentiment analysis, masked language modeling, summarization, and so on.  
Therein, the output measure encodes the probability distribution of the missing tokens for instance, and its clustering indicates a small number of possible outcomes.
In Sections~\ref{sec: temperature}, \ref{sec: large.beta}, and \ref{sec: high.d}, we show several results (mostly summarized in Figure \ref{fig: clustering.graph}) which indicate that the limiting distribution is a point mass. 
While it may appear that this leaves no room for diversity or randomness, which is at odds with practical observations, 
these results hold for the specific choice of parameter matrices, and apply in possibly very long-time horizons. Numerical experiments indicate a more subtle picture for different parameters---for instance, there is an appearance of a long metastable phase during which the particles coalesce in a small number of clusters, which seems consistent with behavior in trained models (Figure~\ref{fig: albert} in Section~\ref{sec: metastability})---. We are not able to theoretically explain this behavior as of now.

Ultimately, the appearance of clusters is somewhat natural 
 since the Transformer dynamics are a weighted average of all particles, with the weights being hard-wired to perform a fast selection of particles most similar to the $i$-th particle being queried. This causes the emergence of leaders which attract all particles in their vicinity. In the natural language processing interpretation, where particles represent tokens, this further elucidates the wording \emph{attention} as the mechanism of inter-token attraction, and the amplitude of the inner product between tokens can be seen as a measure of their \emph{semantic similarity}.

\section{A single cluster for small $\beta$} \label{sec: temperature}

As seen in Figure \ref{fig: clustering.graph}, the dimension $d$ and inverse temperature $\beta$ appear to play a key role in the clustering results we obtain. In this section and in Section \ref{sec: large.beta}, we begin by focusing on extreme choices of $\beta$ whilst $d, n$ are fixed. We first focus on the case $\beta=0$ (Section \ref{sec: beta.nul}), before moving to the case $\beta\ll 1$ by a perturbation argument (Section \ref{sec: beta.presque.nul}). We cover the case where $\beta$ is sufficiently large, but finite, in Section \ref{sec: large.beta}. The case $\beta=+\infty$ is of little interest since all particles are fixed by the evolution.

\subsection{The case $\beta=0$} \label{sec: beta.nul}
For $\beta=0$, both~\eqref{SA} and~\eqref{USA} read as
  \begin{equation}\label{e:Snonres0}
     \dot{x}_i(t) = \proj_{x_i(t)}\left(\frac1n\sum_{j=1}^n x_j(t)\right), \hspace{1cm} t\geqslant0.
 \end{equation}
The following result shows that generically over the initial points, a single cluster emerges. It complements a known convergence result (\cite[Theorem 2]{frouvelle2019long}) for \eqref{e:Snonres0}. 
In \cite[Theorem 2]{frouvelle2019long}, the authors show convergence to an antipodal configuration, in the sense that $n-1$ particles converge to some $x^*\in\S$, with the last particle converging to $-x^*$. Moreover, once convergence is shown to hold, it holds with an exponential rate.
Mimicking the proof strategy of \cite[Theorem 2.2]{benedetto2015complete} and \cite[Theorem 3.2]{ha2018relaxation}, 
we sharpen this result by showing that the appearance of an antipodal particle is non-generic over the choice of initial conditions.

\begin{theorem} \label{p:beta0} 
Let $d, n\geq 2$. For Lebesgue almost any initial sequence $(x_i(0))_{i\in[n]}\in (\S)^n$, there exists some point $x^*\in\mathbb{S}^{d-1}$ such that the unique solution $(x_i(\cdot))_{i\in[n]}\in C^0(\mathbb{R}_{\geq0};(\mathbb{S}^{d-1})^n)$ to the corresponding Cauchy problem for~\eqref{e:Snonres0} satisfies 
\begin{equation*}
    \lim_{t\to+\infty} x_i(t)=x^*
\end{equation*}
for any $i\in[n]$.
\end{theorem}

This is also referred to as convergence toward \emph{consensus} in collective behavior models.
We refer the interested reader to Appendix~\ref{app: proof.2} for the proof, which relies on a gradient flow reinterpretation of \eqref{e:Snonres0}, much like the proofs of several subsequent theorems. We provide some comments in Section \ref{sec: large.beta}.

\subsection{The case $\beta\ll 1$} \label{sec: beta.presque.nul}
Theorem~\ref{p:beta0} has some implications for small but positive $\beta$, something which is already seen in \Cref{fig: phase.diag.Id} and \Cref{fig: metastability}. 
This is essentially due to the fact that, formally,
\begin{equation*} 
\dot{x}_i(t)=\proj_{x_i(t)}\left(\frac1n\sum_{j=1}^nx_j(t)\right)+O(\beta)
\end{equation*}
for $\beta\ll1$.
So, during a time $\ll\beta^{-1}$, the particles do not feel the influence of the remainder $O(\beta)$ and behave as in the regime $\beta=0$. This motivates

\begin{theorem}\label{th:beta_small}
    Fix $d, n\geq 2$. For $\beta\geq0$, let $\mathscr{S}_\beta\subset(\S)^n$ be the subset consisting of all initial sequences for which the associated solution to the Cauchy problem for~\eqref{SA} (or~\eqref{USA}) converges to one cluster as $t\rightarrow+\infty$. 
    Then
    \begin{equation*}
     \lim_{\beta\to 0}\mathbb{P}(\mathscr{S}_\beta)=1.   
    \end{equation*}
    
    More generally, if $Q$ and $K$ are arbitrary $d\times d$ matrices, then the same result also holds for the Cauchy problem for \eqref{eq: transformerSd.QKV} with $V=I_d$ (or the natural analogue of \eqref{USA} with these parameters).
\end{theorem}

\begin{proof} We focus on the dynamics~\eqref{SA}, but the proof is in fact identical in the case of~\eqref{USA}. 

For $\alpha\in[0,1)$, we say that a set formed from $n$ points $z_1,\ldots,z_n\in(\S)^n$ is \emph{$\alpha$--clustered} if for any $i,j\in[n]$, $\langle z_i,z_j\rangle >\alpha$ holds. Observe that if $\{z_1,\ldots,z_n\}$ is $\alpha$--clustered for some $\alpha\geq 0$, then the solution to the Cauchy problem for~\eqref{SA} (for arbitrary $\beta\geq0$) with this sequence as initial condition converges to a single cluster, since $w=z_1$ satisfies the assumption in \Cref{lem: hemisphere.clustering}.

Now, for any integer $m\geq1$, we denote by $\mathscr{S}_0^m\subset \mathscr{S}_0$ the set of initial sequences $x_1(0),\ldots,x_n(0)$ in $(\S)^n$ for which the solution $(x_i^0(\cdot))_{i\in[n]}$ to the associated Cauchy problem for~\eqref{e:Snonres0} is $\frac34$--clustered at time $t=m$, namely 
\begin{equation} \label{eq: harry.potter}
  \langle x_i^0(m), x_j^0(m)\rangle>\frac34  
\end{equation}
holds for all $i, j\in[n]$. 
We see that $\mathscr{S}_0^m$ is an open set for any integer $m\geq1$. Moreover, $\mathscr{S}_0^m\subset \mathscr{S}_0^{m+1}$ according to the proof of \Cref{lem: hemisphere.clustering}, and $\bigcup_{m=1}^{+\infty} \mathscr{S}_0^m=\mathscr{S}_0$. This implies that 
\begin{equation}\label{e:Ps0n}
    \lim_{m\to+\infty}\mathbb{P}(\mathscr{S}_0^m)=1.
\end{equation}
We now show that the solution to~\eqref{SA} is near that of~\eqref{e:Snonres0}, starting from the same initial condition, when $\beta$ is small.  
Using the Duhamel formula, we find
\begin{align*}
    x_i^\beta(t)-x_i^0(t) &=\int_0^t \sum_{j=1}^n\left(\frac{e^{\beta\langle Qx_i^\beta(s),Kx_j^\beta(s)\rangle}}{\sum_{k=1}^n e^{\beta\langle Qx_i^\beta(s),Kx_k^\beta(s)\rangle}}\right)\proj_{x_i^\beta(s)}(x_j^\beta(s))\diff s\\
    &\quad-\int_0^t\frac{1}{n}\sum_{j=1}^n\proj_{x_i^0(s)}(x_j^0(s)) \diff s \nonumber\\
    &=\int_0^t \sum_{j=1}^n\left(\frac{1}{n}+O\left(\frac{\beta}{n}\right)\right)\proj_{x_i^\beta(s)}(x_j^\beta(s))\diff s\nonumber\\
    &\quad-\int_0^t\frac{1}{n}\sum_{j=1}^n\proj_{x_i^0(s)}(x_j^0(s)) \diff s, \nonumber
\end{align*}
where we used that all particles lie on $\S$ for all times. Employing Gr\"onwall, we deduce
\begin{align} \label{e:approxsphere}
    \left\|x_i^\beta(t)-x_i^0(t)\right\|\leq O(\beta)e^{3t}
\end{align}
for all $t\geq0$, $\beta\geq0$ and $i\in[n]$. Due to~\eqref{e:approxsphere}, there exists some $\beta_m>0$ such that for any $\beta\in[0,\beta_m]$, 
\begin{equation} \label{eq: youareawizardharry}
    \left\|x_i^\beta(m)-x_i^0(m)\right\|\leq\frac18.
\end{equation}
For this to hold, we clearly need $\beta_m\to0$ as $m\to+\infty$.
Combining~\eqref{eq: harry.potter} and~\eqref{eq: youareawizardharry}, we gather that for any initial condition in $\mathscr{S}_0^m$, the solution $(x_i^\beta(\cdot))_{i\in[n]}$ to the corresponding Cauchy problem for 
\eqref{SA} is $\frac12$--clustered at time $t=m$, namely satisfies
\begin{equation*}
    \langle x_i^\beta(m), x_j^\beta(m)\rangle>\frac12
\end{equation*}
for all $i, j\in[n]$ and $\beta\in[0,\beta_m]$.
Thus $\mathscr{S}_0^m\subset \mathscr{S}_\beta$ for any $\beta\in[0,\beta_m]$ by virtue of \Cref{lem: hemisphere.clustering}, which together with~\eqref{e:Ps0n} concludes the proof.
\end{proof}

In the specific case where $Q^\top K=I_d$, we can in fact significantly sharpen Theorem \ref{th:beta_small} by relying on the gradient flow structure evoked in Section \ref{sec: new.gradient.flow}. Namely, we can show the following.

\begin{theorem} \label{thm: beta.tiny}
    Fix $d, n\geq 2$. There exists a numerical constant $C>0$ such that whenever $$\beta\leq Cn^{-1}$$ 
    the following holds.
    For Lebesgue almost any $(x_i(0))_{i\in[n]}\in(\S)^n$, there exists $x^*\in\S$ such that the solution $(x_i(\cdot))_{i\in[n]}\in C^0(\R_{\geq0};(\S)^n)$ to the corresponding Cauchy problem for \eqref{SA} (resp. for \eqref{USA}) satisfies 
    \begin{equation*}
        \lim_{t\to+\infty} x_i(t) = x^*
    \end{equation*}
    for all $i\in[n]$. 
    
    Moreover, when $d=2$ one can take $\beta\leq 1$ and the same conclusion holds.
\end{theorem}

The proof can be found in Appendix \ref{app: beta.tiny}. As for Theorem \ref{p:beta0}, we briefly discuss the general strategy in Section \ref{sec: large.beta} just below.
The improvement to $\beta\leq1$ when $d=2$ is due to the recent paper \cite{criscitiello2024synchronization}. 

\section{A single cluster for large $\beta$} \label{sec: large.beta}

The chain of thought leading to the proof of Theorem \ref{thm: beta.tiny} also leads to

\begin{theorem} \label{thm: beta.interval}
    Fix $d, n\geq 2$. There exists a constant $C=C(d)>0$ depending only on $d$ such that whenever 
    $$\beta\geq Cn^2,$$
    the conclusion of Theorem \ref{thm: beta.tiny} holds for both \eqref{SA} and \eqref{USA}.
\end{theorem}

We refer the interested reader to Appendix \ref{apx:beta_high} for the proof, which, much like those of Theorem \ref{p:beta0} and \ref{thm: beta.tiny} relies on the gradient flow interpretation of the dynamics. We give a general outline thereof. Since all the intervening functions and metrics are real-analytic, the celebrated \emph{\L{}ojasiewicz theorem}~\cite{lojasiewicz1963propriete} implies that for any initial condition, the gradient flow converges to some critical point of $\mathsf{E}_\beta$ as $t\to+\infty$. The genericity over the initial configurations seen in the statements comes from the \emph{center-stable manifold theorem} (Lemma \ref{l:nosaddleconv}). The former ensures that for almost every initial configuration, the gradient flow does not converge to a strict saddle point of $\mathsf{E}_\beta$ (namely critical points where the Hessian has at least one positive eigenvalue). Whence, generically over the initial configurations, the gradient flow converges to a local maximum. One is then left analyzing the landscape of the underlying energy, with the goal of ensuring that all local maxima are necessarily global.

\section{The high-dimensional case} \label{sec: high.d}

We now elucidate the role that the dimension $d$ plays in clustering results. 
To start, it turns out that the restrictions on $\beta$ provided by Theorems \ref{thm: beta.tiny} and \ref{thm: beta.interval} are specific to the two-dimensional case. The following result is shown in \cite{markdahl2017almost, criscitiello2024synchronization}.

\begin{theorem}[\cite{markdahl2017almost, criscitiello2024synchronization}] 
\label{thm: boumal}
    Fix $n\geq 2$, $d\geq 3$ and $\beta\geq 0$. Then the conclusion of Theorem \ref{thm: beta.tiny} holds for both \eqref{SA} and \eqref{USA}.
\end{theorem}

In Section \ref{sec: exp.rate} we show that the convergence of Theorem \ref{thm: boumal} is exponentially fast when $d\geq n$ (although, with a decay constant that is exponentially small in $\beta$) and in Section \ref{sec: phase} we describe the full dynamics when $n$ is fixed and $d\to+\infty$. We discuss some numerical experiments and posit questions on the intermediate behavior of the dynamics when $\beta\gg1$ in Section \ref{sec: metastability}.

We were made aware of Theorem \ref{thm: boumal} after the first version of this manuscript. The downside of our original proof, which caused us to miss the full range of \(\beta\), is that we focused on \(d=2\), where the natural perturbations  to a critical point involve selecting which particles are moving and which are standing still. However, in \(d>2\), one can move all particles in the same direction (as is done in the proof in \cite{criscitiello2024synchronization}). Using that there is a continuum of directions and only finitely many points, one can find some direction that brings all of them closer. 

\begin{remark}[Invariant measures]
In the theory of dynamical systems, often the first question of interest is to find smooth invariant measures. It is clear that whenever conclusions of Theorem \ref{thm: beta.tiny} hold (in particular, for all $\beta,n$ when $d\ge 3$), neither \eqref{SA} nor \eqref{USA} may possess a smooth invariant measure.
\end{remark}

\subsection{Clustering at an exponential rate when $d\geq n$} \label{sec: exp.rate}

One can ask whether for almost every initial configuration, the convergence provided by all of the results above holds with some rate.
The answer is affirmative---and the rate is in fact exponential---when the initial configuration lies in an open hemisphere.

\begin{theorem} \label{thm: d.infty}
Let $n\geq1$ and $\beta>0$. Suppose $d\geqslant n$. Consider the unique solution $(x_i(\cdot))_{i\in[n]}\in C^0(\mathbb{R}_{\geq0}; (\S)^n)$ to the Cauchy problem  
 for~\eqref{SA} or~\eqref{USA}, corresponding to an initial sequence of points $(x_i(0))_{i\in[n]}\in(\S)^n$ distributed uniformly at random. Then almost surely there exists $x^*\in\S$ and constants $C,\lambda>0$ such that 
    \begin{equation} \label{eq: expconvtocons}
    \|x_i(t)-x^*\|\leq Ce^{-\lambda t}
    \end{equation} 
    holds for all $i\in[n]$ and $t\geq 0$. 

    In fact, let $Q$ and $K$ be arbitrary $d\times d$ matrices. Then the same result also holds for the solution to the corresponding Cauchy problem for~\eqref{eq: transformerSd.QKV} with $V=I_d$ (or the natural analogue of~\eqref{USA} with these parameters).
\end{theorem}

When $d\geq n$ and the points $(x_i(0))_{i\in[n]}\in(\S)^n$ are distributed uniformly at random, with probability one there exists\footnote{This weak version of Wendel's theorem (Theorem~\ref{r:wendel}) is easy to see directly.} 
$w\in\mathbb{S}^{d-1}$ such that $\langle w,x_i(0)\rangle>0$ for any $i\in[n]$. In other words, all of the initial points lie in an open hemisphere almost surely.
The proof of \Cref{thm: d.infty} thus follows as a direct corollary of the following result, which holds for any $n\geq 1$ and $d\geq2$:

\begin{lemma}[Cone collapse] \label{lem: hemisphere.clustering}
Let $\beta>0$ and let $(x_i(0))_{i\in[n]}\in(\S)^n$ be such that there exists $w\in\mathbb{S}^{d-1}$ for which $\langle x_i(0),w\rangle >0$ for any $i\in[n]$. Consider the unique solution $(x_i(\cdot))_{i\in[n]}\in C^0(\mathbb{R}_{\geq0};(\S)^n)$ to the corresponding Cauchy problem for~\eqref{SA} or~\eqref{USA}. Then there exists $x^*\in\S$ and constants $C,\lambda>0$ such that
\begin{equation*}
\|x_i(t)-x^*\|\leq Ce^{-\lambda t}
\end{equation*}
holds for all $i\in[n]$ and $t\geq0$.

In fact, let $Q$ and $K$ be arbitrary $d\times d$ matrices. Then the same result also holds for the solution to the corresponding Cauchy problem for~\eqref{eq: transformerSd.QKV} with $V=I_d$ (or the natural analogue of~\eqref{USA} with these parameters).
\end{lemma}

\begin{remark}
Lemma~\ref{lem: hemisphere.clustering} implies that 
$\left\{(\bar{x}_i)_{i\in[n]}\in(\S)^n \colon \bar{x}_1=\ldots=\bar{x}_n\right\}$ is Lyapunov asymptotically stable as a set. In fact, it is exponentially stable.
\end{remark}

Lemma~\ref{lem: hemisphere.clustering} is reminiscent of results on interacting particle systems on the sphere (see~\cite[Theorem 3.7]{caponigro2015nonlinear} for instance), and the literature on synchronization for the Kuramoto model on the circle (\cite[Lemma 2.8]{abdalla2022expander}, \cite[Theorem 3.1]{ha2020asymptotic} and Section~\ref{s:kuramoto}).
We often make use of the following elementary lemma.

\begin{lemma} \label{lem: ez.lemma}
    Let $f:\R_{\geq0}\to\R$ be a differentiable function such that 
    \begin{equation*}
        \int_0^{+\infty} |f(t)|\diff t + \sup_{t\in\R_{\geq0}}\left|\dot{f}(t)\right|<+\infty.
    \end{equation*}
    Then $\lim_{t\to+\infty} f(t)=0$.
\end{lemma}

The proof of \Cref{lem: hemisphere.clustering} is an adaptation of \cite[Theorem 1]{caponigro2015nonlinear}. We present it here for completeness.

\begin{proof}[Proof of \Cref{lem: hemisphere.clustering}]
We focus on the case~\eqref{USA}, and set 
$$a_{ij}(t):=n^{-1}e^{\beta\langle x_i(t), x_j(t)\rangle}>0.$$ 
The proof for~\eqref{SA} is identical, and one only needs to change the coefficients $a_{ij}(t)$ by $Z_{\beta,i}(t)^{-1}e^{\beta\langle x_i(t),x_j(t)\rangle}$ throughout. Also note that since we only make use of the positivity of the coefficients $a_{i,j}(t)$ throughout the proof, all arguments are readily generalizable to the case of arbitrary $d\times d$ matrices $Q$ and $K$ appearing in the inner products.

\subsubsection*{Step 1. Clustering}

For $t\geq 0$, consider
\begin{equation*}
i(t)\in \argmin_{i\in[n]} \langle x_i(t), w\rangle.
\end{equation*}
Fix $t_0\geq 0$. We have
\begin{align*}
&\left(\frac{\diff}{\diff t}\langle x_{i(t_0)}(\cdot),w\rangle\right)\Bigm|_{t=t_0}\\
&=\sum_{j=1}^n a_{i(t_0)j}(t_0)\Big(\langle x_j(t_0),w\rangle-\langle x_{i(t_0)}(t_0),x_j(t_0)\rangle\langle x_{i(t_0)}(t_0),w\rangle\Big)\geq 0.
\end{align*}
This implies that all points remain within the same open hemisphere at all times and the map
\begin{equation*} 
t\mapsto r(t):=\min_{i\in[n]} \langle x_i(t),w\rangle 
\end{equation*}
is non-decreasing on $\mathbb{R}_{\geq0}$. It is also bounded from above by $1$. 
We may thus define $r_\infty:=\lim_{t\rightarrow +\infty} r(t)$. Note that $r_\infty\geq r(0)>0$ by assumption. By compactness, there exist a sequence of times $\{t_k\}_{k=1}^{+\infty}$ with $t_k\rightarrow +\infty$, and some $(\overline{x}_i)_{i\in[n]}\in(\S)^n$ such that $\lim_{k\rightarrow+\infty} x_i(t_k)=\overline{x}_i$ for all $i\in[n]$. Using the definition of $r(t)$, we also find that 
\begin{equation*}
\langle \overline{x}_j,w\rangle\geq r_\infty
\end{equation*}
for all $j\in[n]$, and by continuity, there exists $i\in[n]$ such that $\langle \overline{x}_i,w\rangle = r_\infty$. Then
\begin{align} \label{eq: therighthandside}
\lim_{k\rightarrow +\infty}\langle \dot{x}_i(t_k),w\rangle=\sum_{\substack{j=1}}^n \overline{a}_{ij}(\langle w,\overline{x}_j\rangle-\langle \overline{x}_i,\overline{x}_j\rangle\langle \overline{x}_i,w\rangle)\geq r_\infty \sum_{\substack{j=1}}^n\overline{a}_{ij}(1-\langle \overline{x}_i,\overline{x}_j\rangle),
\end{align}
where we set $\overline{a}_{ij}:=e^{\beta\langle \overline{x}_i, \overline{x}_j\rangle}>0$. Notice that 
\begin{equation*}
    \lim_{k\to+\infty}\int_{t_k}^{+\infty} \langle \dot{x}_i(s),w\rangle \diff s=r_\infty-\lim_{k\to+\infty}\langle x_i(t_k),w \rangle= 0,
\end{equation*}
and by using the equation~\eqref{USA} we also find that $|\langle\ddot{x}_i(t), w\rangle|=O(e^{2\beta})$ for any $t\geq0$. Therefore by Lemma~\ref{lem: ez.lemma}, the left-hand side of~\eqref{eq: therighthandside} is equal to $0$, and consequently the right-hand side term as well.
This implies that $\overline{x}_1=\ldots=\overline{x}_n:=x^*$. Repeating the argument by replacing $w$ with $x^*$, we see that the extraction of a sequence $\{t_k\}_{k=1}^{+\infty}$ as above is not necessary, and therefore 
\begin{equation} \label{eq: qual.conv}
    \lim_{t\to+\infty} x_i(t)=x^*
\end{equation}
for all $i\in[n]$.

\subsubsection*{Step 2. Exponential rate.}
We now improve upon \eqref{eq: qual.conv}. Set
$$\alpha(t):=\min_{i\in[n]} \langle x_i(t),x^*\rangle.$$
From \eqref{eq: qual.conv} we gather that there exists some $t_0>0$ such that $\alpha(t)\geq\frac12$ for all $t\geq t_0$.
Also, in view of what precedes we know that $x^*$ lies in the convex cone generated by the points $x_1(t),\ldots,x_n(t)$ for any $t>0$. Thus, there exists some $\eta\in (0,1]$ such that $\eta x^*$ is a convex combination of the points $x_1(t),\ldots,x_n(t)$, which implies that
\begin{equation} \label{e:decompox*.step2}
x^*=\sum_{k=1}^n \theta_k(t)x_k(t), \hspace{0.5cm} \text{ for some } \hspace{0.5cm} \sum_{k=1}^n \theta_k(t)\geq 1, \quad \theta_k(t)\geq 0 \quad \forall k\in[n]. 
\end{equation}
For any $t$, we denote by $i(t)$ an element of $\argmin(\langle x_i(t),x^*\rangle)$ for which $\langle \dot{x}_i(t),x^*\rangle$ is smallest. It follows from a Taylor expansion of $\langle x_i(t+h),x^*\rangle$ for $h>0$ and $i\in[n]$ that
$$
\dot{\alpha}(t)=\langle \dot{x}_{i(t)}(t),x^*\rangle.
$$
Therefore
\begin{equation}\label{e:dotalpha.step2}
\dot{\alpha}(t)=\langle \dot{x}_{i(t)}(t),x^*\rangle\geq \sum_{j=1}^n a_{i(t)j}(t)(1-\langle x_{i(t)}(t),x_j(t)\rangle) \alpha(t)
\end{equation}
On another hand,
\begin{equation}\label{e:mineqalpha.step2}
\min_{j\in[n]} \langle x_{i(t)}(t),x_j(t)\rangle\leq \sum_{k=1}^n \theta_k(t)\langle x_{i(t)}(t),x_k(t)\rangle= \langle x_{i(t)}(t),x^*\rangle=\alpha(t).
\end{equation}
Plugging~\eqref{e:mineqalpha.step2} into~\eqref{e:dotalpha.step2} and using $a_{ij}(t)\geq n^{-1}e^{-2\beta}$ we get
\begin{equation}\label{e:diffineqalpha.step2}
\dot{\alpha}(t)\geq \frac{1}{2ne^{2\beta}}(1-\alpha(t))
\end{equation}
for $t\geq t_0$. Applying the Gr\"onwall inequality we get 
\begin{equation}
    1-\alpha(t)\leq \frac12 e^{-\frac{1}{2ne^{\beta}}(t-t_0)}
\end{equation}
for all $t\geq t_0$. The conclusion follows.
\end{proof}

In the case $d<n$, we can still apply Wendel's theorem (recalled below) together with Lemma~\ref{lem: hemisphere.clustering} to obtain clustering to a single point with probability at least $p_{n,d}$ for some explicit $p_{n,d}\in(0,1)$.

\begin{theorem}[Wendel, \cite{wendel1962problem}] \label{r:wendel}
Let $d,n\geq 1$ be such that $d\le n$. Let $x_1,\ldots, x_n$ be $n$ i.i.d. uniformly distributed points on $\S$. The probability that these points all lie in the same hemisphere is:
\begin{equation*}
\mathbb{P}\Big(\exists w\in\S \colon \langle x_i, w\rangle >0 \quad \mathrm{ for } \,\mathrm{ all } \quad i\in[n]\Big) = 2^{-(n-1)}\sum_{k=0}^{d-1} \binom{n-1}{k}.
\end{equation*}
\end{theorem}

\subsection{More precise quantitative convergence}
\label{sec: phase}

When $n$ is fixed and $d\rightarrow +\infty$, in addition to showing the formation of a cluster as in Theorem~\ref{thm: d.infty}, it is possible to quantitatively describe the entire evolution of the particles with high probability. To motivate this, on the one hand we note that since the dynamics evolve on $\S$, inner products are representative of the distance between points, and clustering occurs if $\langle x_i(t), x_j(t)\rangle \rightarrow 1$ for any $(i,j)\in[n]^2$ as $t\rightarrow +\infty$.
On the other hand, if $d\gg n$, $n$ points in a generic initial sequence are \emph{almost orthogonal} by concentration of measure~\cite[Chapter~3]{Ver18}, and we are thus able to compare their evolution with that of an initial sequence of truly \emph{orthogonal} ones. 

We begin by describing the case of exactly orthogonal initial particles, which is particularly simple as the dynamics are described by a single parameter.

\begin{theorem} \label{thm: orthogonal}
Let $\beta\geq0$, $d,n\geq2$ be arbitrary. 
Consider an initial sequence $(x_i(0))_{i\in[n]}\in(\S)^n$ of $n$ pairwise orthogonal points: 
$\langle x_i(0), x_j(0)\rangle=0$ for $i\neq j$, and let $(x_i(\cdot))_{i\in[n]}\in C^0(\mathbb{R}_{\geqslant0};(\S)^n)$ denote the unique solution to the corresponding Cauchy problem for~\eqref{SA} (resp. for~\eqref{USA}).
Then the angle $\angle (x_i(t), x_j(t))$ is the same for all distinct $i,j\in[n]$:
\begin{equation*}
\angle (x_i(t), x_j(t)) = \theta_\beta(t) 
\end{equation*}
for $t\geqslant0$ and some $\theta_\beta\in C^0(\mathbb{R}_{\geq0};\mathbb{T})$. 
Furthermore, for~\eqref{SA}, $\gamma_\beta(t):=\cos(\theta_\beta(t))$ satisfies 
\begin{align} \label{eq: ybeta}
\begin{dcases}
    \dot{\gamma}_\beta(t)=\frac{2e^{\beta \gamma_\beta(t)}(1-\gamma_\beta(t))((n-1)\gamma_\beta(t)+1)}{e^\beta+(n-1)e^{\beta \gamma_\beta(t)}} &\text{ for } t\ge 0 \\
    \gamma_\beta(0)=0\,,
\end{dcases}
\end{align}
and for~\eqref{USA}, we have
\begin{align} \label{eq: ybetaUSA}
\begin{dcases}
    \dot{\gamma}_\beta(t)=\frac{2}{n}e^{\beta \gamma_\beta(t)}(1-\gamma_\beta(t))((n-1)\gamma_\beta(t)+1) &\text{ for } t\ge 0\\
    \gamma_\beta(0)=0\,.
\end{dcases}
\end{align}
\end{theorem} 

Here and henceforth, $\mathbb{T}=\mathbb{R}/2\pi\mathbb{Z}$ denotes the one-dimensional torus. We provide a brief proof of \Cref{thm: orthogonal} just below.
The following result then shows that when $d\gg n$, $t\mapsto \gamma_\beta(t)$ is a valid approximation for $t\mapsto \langle x_i(t),x_j(t)\rangle$ for any distinct $i,j\in [n]$.

\begin{theorem} \label{thm: phase.transition.curve}
Fix $\beta\geqslant0$ and $n\geq2$. Then there exists some $d^*(n,\beta)\geq n$ such that for all $d\geq d^*(n,\beta)$, the following holds.
Consider a sequence $(x_i(0))_{i\in[n]}$ of $n$ i.i.d. uniformly distributed points on $\S$, and let $(x_i(\cdot))_{i\in[n]}\in C^0(\mathbb{R}_{\geq0}; (\S)^n)$ denote the unique solution to the corresponding Cauchy problem for~\eqref{SA}. 
Then there exist 
$C=C(n,\beta)>0$ and $\lambda=\lambda(n,\beta)>0$, such that with probability at least $1-2n^2d^{-1/64}$, 
\begin{equation} \label{eq: upto-t}
\Big|\langle x_i(t), x_j(t)\rangle-\gamma_\beta(t)\Big|\leqslant \min\left\{2\cdot c(\beta)^{nt}\sqrt{\frac{\log d}{d}}, C e^{-\lambda t}\right\}
\end{equation}
holds for any $i\neq j$ and $t\geq0$, where $c(\beta)=e^{10\max\{1,\beta\}}$, and $\gamma_\beta$ is the unique solution to~\eqref{eq: ybeta}.
\end{theorem}

Since the proof is rather lengthy, we defer it to Appendix~\ref{app: proof.1}. It  relies on combining the stability of the flow with respect to the initial data (entailed by the Lipschitz nature of the vector field) with concentration of measure. 
An analogous statement also holds for~\eqref{USA}, and more details can be found in Remark~\ref{rem: usa.d}, whereas 
the explicit values of $C$ and $\lambda$ can be found in~\eqref{e:ineqsecondpart}.
The upper bound in~\eqref{eq: upto-t} is of interest in regimes where $d$ and/or $t$ are sufficiently large as the error in~\eqref{eq: upto-t} is trivially bounded by $2$.

\begin{proof}[Proof of \Cref{thm: orthogonal}] We split the proof in two parts. We focus on proving the result for the dynamics~\eqref{SA}, since the same arguments readily apply to the dynamics~\eqref{USA}.

\subsubsection*{Part 1. The angle $\theta_\beta(t)$}
We first show there exists $\theta\in C^0(\R_{\geq0};\mathbb{T})$ such that $\theta(t)=\angle (x_i(t), x_j(t))$ for any distinct $(i, j)\in[n]^2$ and $t\geq0$. 
Since the initial tokens are orthogonal (and thus $d\geq n$), we may consider an orthonormal  basis $(e_1,\ldots,e_d)$ of $\mathbb{R}^d$ such that $x_i(0)=e_i$ for $i\in[n]$.  Let $\pi:[d]\rightarrow [d]$ be a permutation. By decomposing any $x\in\S$ in this basis, we define $P_\pi:\S\rightarrow\S$ as $$P_\pi\left(\sum_{i=1}^n a_ie_i\right)=\sum_{i=1}^n a_ie_{\pi(i)}.$$
Setting $y_i(t)=P_\pi(x_i(t))$ for $i\in[n]$, we see that $y_i(t)$ solves~\eqref{SA} with initial condition $y_i(0)=P_\pi(x_i(0))$. But  $(x_{\pi(1)}(t),\ldots, x_{\pi(n)}(t))$ is a solution of~\eqref{SA} by permutation equivariance, and it has the same initial condition since $P_\pi(x_i(0))=x_{\pi(i)}(0)$. Consequently, we deduce that $P_\pi(x_i(t))=x_{\pi(i)}(t)$ for any $t\geq 0$ and any $i\in[d]$. Hence $$\langle x_i(t),x_j(t)\rangle =\langle P_\pi(x_i(t)),P_\pi(x_j(t))\rangle=\langle x_{\pi(i)}(t), x_{\pi(j)}(t)\rangle$$ 
which concludes the proof.

\subsubsection*{Part 2. The curve $\gamma_\beta(t)$}
By virtue of the orthogonality assumption we have $\gamma_\beta(0)=\cos(\theta_\beta(0))=0$.
To prove that $\gamma_\beta(t)$ satisfies~\eqref{eq: ybeta} for the case of~\eqref{SA}, recall that
\begin{align*}
\proj_{x_i(t)}(x_j(t))=x_j(t)-\langle x_i(t), x_j(t)\rangle\, x_i(t).
\end{align*}
Then for $k\neq i$,
\begin{align*}
\dot{\gamma}_\beta(t) &= 2\langle \dot{x}_i(t), x_k(t)\rangle \\
&= 2\sum_{\substack{j=1}}^n \left(\frac{e^{\beta\langle x_i(t), x_j(t)\rangle}}{\sum_{\ell=1}^n e^{\beta\langle x_i(t), x_\ell(t)\rangle}}\right) \left(\langle x_j(t), x_k(t)\rangle-\langle x_i(t), x_j(t)\rangle\langle x_i(t), x_k(t)\rangle\right).
\end{align*}
Since the denominator in the above expression is equal to $(n-1)e^{\beta \gamma_\beta(t)}+e^\beta$, we end up with
\begin{align*}
\dot{\gamma}_\beta(t) &= \frac{2e^{\beta \gamma_\beta(t)}}{(n-1)e^{\beta \gamma_\beta(t)}+e^\beta}\sum_{\substack{j=1}}^n \Big(\langle x_j(t), x_k(t)\rangle-\langle x_i(t), x_j(t)\rangle\langle x_i(t), x_k(t)\rangle\Big)\\
&= \frac{2e^{\beta \gamma_\beta(t)}}{(n-1)e^{\beta \gamma_\beta(t)}+e^\beta}(1-\gamma_\beta(t)^2+(n-2)(\gamma_\beta(t)-\gamma_\beta(t)^2)),
\end{align*}
as desired. 
\end{proof}

\subsection{Metastability and a phase transition} \label{sec: metastability}

An interesting byproduct of \Cref{thm: orthogonal} and \Cref{thm: phase.transition.curve} is the fact that they provide an accurate approximation of the exact \emph{phase transition curve} delimiting the clustering and non-clustering regimes, in terms of $t$ and $\beta$. 
To be more precise, given an initial sequence $(x_i(0))_{i\in[n]}\in(\S)^n$ of random points distributed independently according to the uniform distribution on $\S$, and for any fixed $0<\delta\ll1$, we define the phase transition curve as the boundary 
\begin{equation*}
    \Gamma_{d,\delta}=\partial\left\{t,\beta \ge 0\colon t=\arginf_{s\ge 0}\Big(\mathbb{P}(\langle x_1(s),x_2(s)\rangle\geq 1-\delta)=1-2n^2d^{-\frac{1}{64}}\Big) \right\}
\end{equation*}
where $(x_i(\cdot))_{i\in[n]}$ denotes the solution to the corresponding Cauchy problem for~\eqref{SA}. (Here the choice of the first two particles instead of a random distinct pair is justified due to permutation equivariance.) 
\Cref{thm: phase.transition.curve} then gives the intuition that  over compact subsets of $(\R_{\geq0})^2$, $\Gamma_{d,\delta}$ should be well-approximated by 
\begin{equation} \label{eq: gamma.infty}
 \Gamma_{\infty,\delta}=\Big\{t, \beta\geq0\colon \gamma_\beta(t)= 1-\delta\Big\}.   
\end{equation}

\begin{figure}[h!]
\centering
\begin{subfigure}{.33\textwidth}
  \centering
  \includegraphics[scale=0.3085]{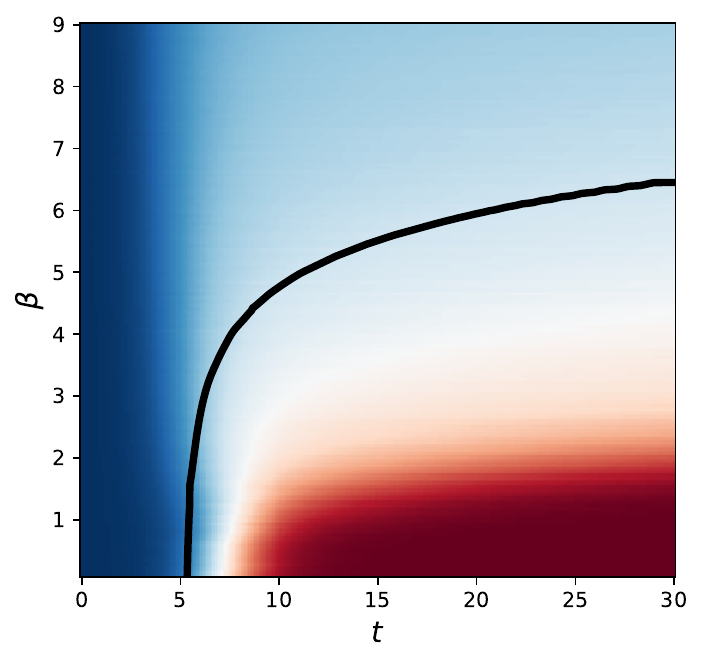}
  \caption{$d=2$}
\end{subfigure}%
\begin{subfigure}{.33\textwidth}
  \centering
  \includegraphics[scale=0.3085]{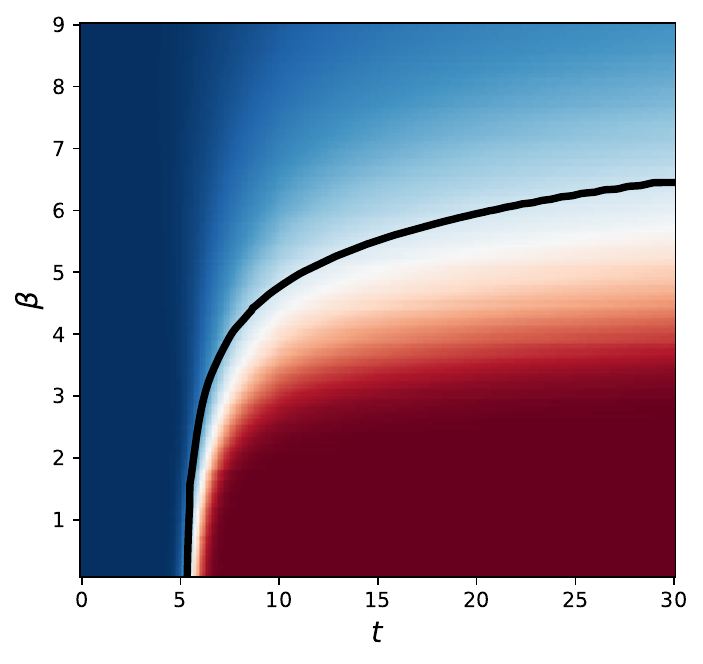}
  \caption{$d=8$}
\end{subfigure}%
\begin{subfigure}{.33\textwidth}
  \centering
  \includegraphics[scale=0.3095]{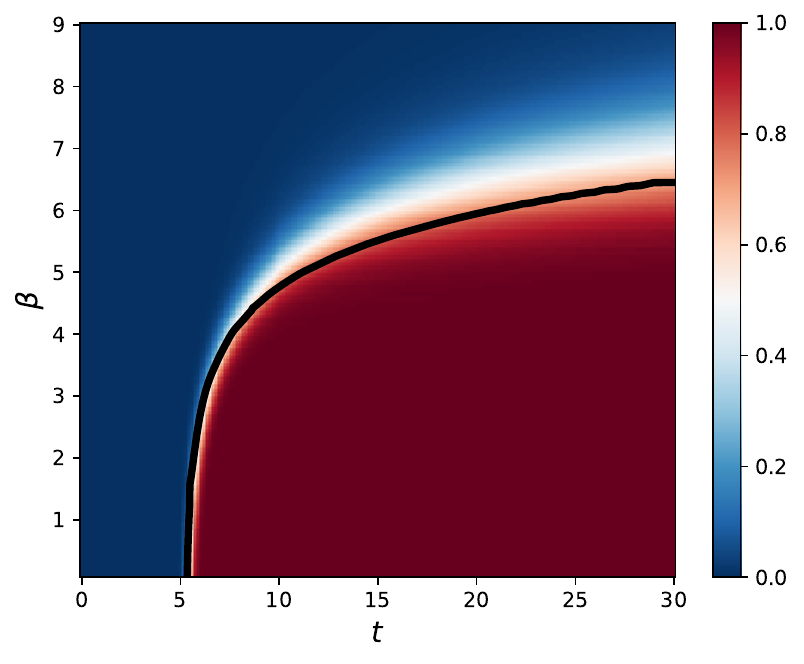}
  \caption{$d=32$}
\end{subfigure}
\begin{subfigure}{.33\textwidth}
  \centering
  \includegraphics[scale=0.3085]{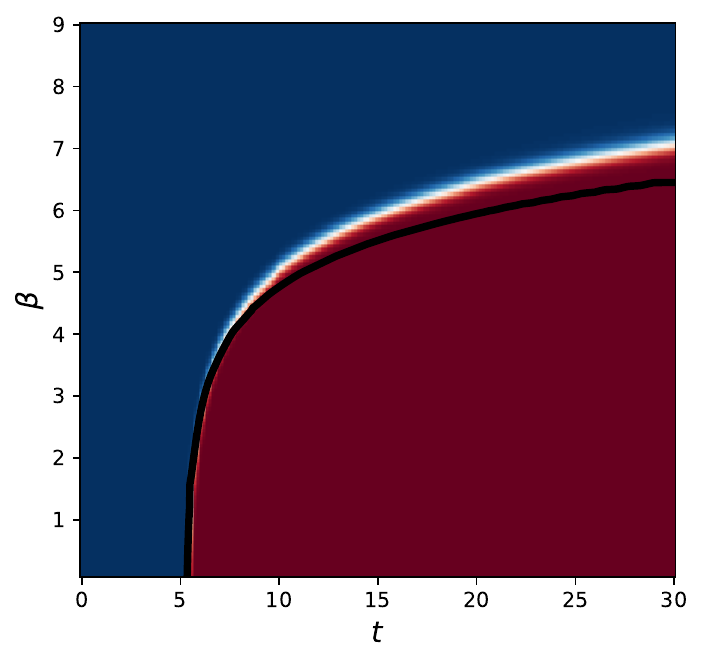}
  \caption{$d=128$}
\end{subfigure}%
\begin{subfigure}{.33\textwidth}
  \centering
  \includegraphics[scale=0.3085]{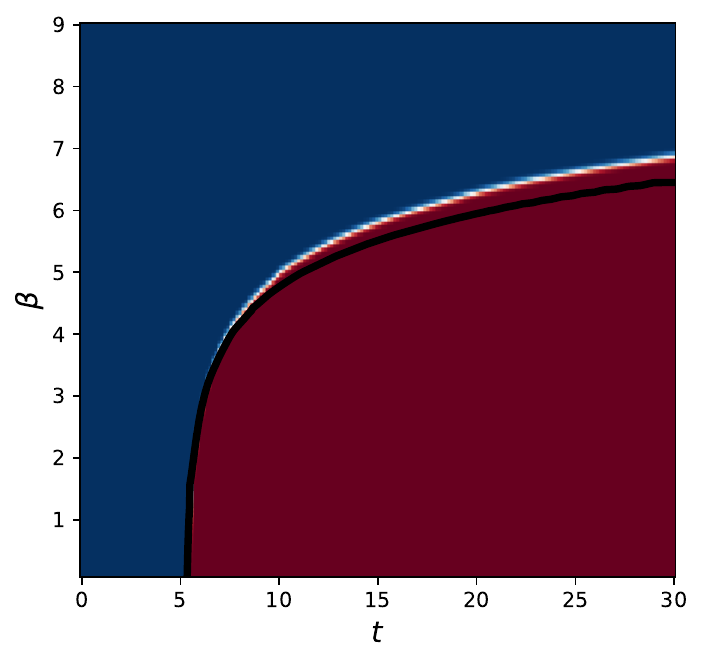}
  \caption{$d=512$}
\end{subfigure}%
\begin{subfigure}{.33\textwidth}
  \centering
  \includegraphics[scale=0.3095]{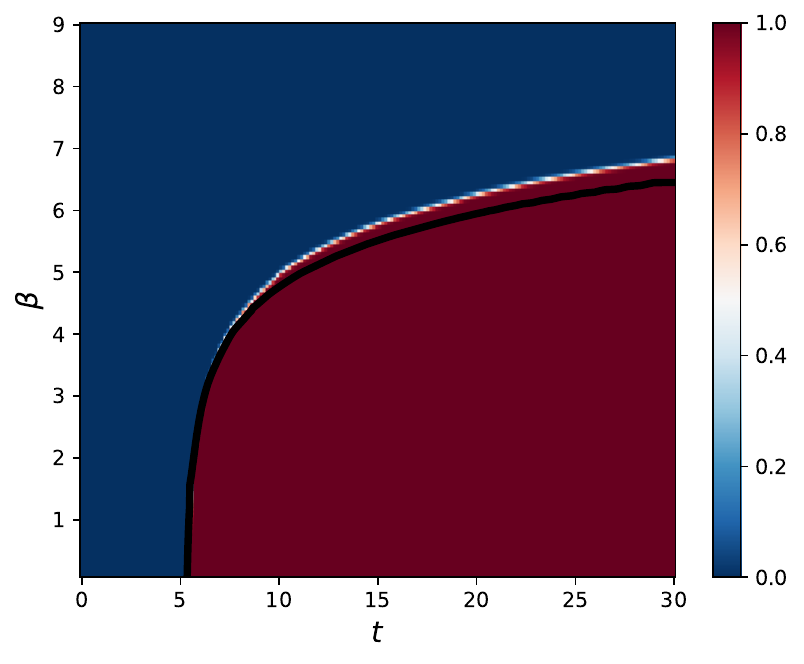}
  \caption{$d=1024$}
\end{subfigure}
\caption{
Plots of the probability that randomly initialized particles following~\eqref{SA} cluster to a single point as a function of $t$ and $\beta$: we graph the function $(t,\beta)\mapsto \mathbb{P}_{(x_1(0),\ldots,x_n(0))\sim\sigma_d}\left(\{\langle x_1(t),x_2(t)\rangle\geq1-\delta\}\right)$,
which is equal to $(t,\beta)\mapsto \mathbb{P}_{(x_1(0),\ldots,x_n(0))\sim\sigma_d, i\neq j \text{ fixed}}\left(\{\langle x_1(t),x_2(t)\rangle\geq1-\delta\}\right)$
by permutation equivariance. We compute this function by generating the average of the histogram of $\{\langle x_i(t),x_j(t)\rangle\geq1-\delta\colon(i,j)\in[n]^2, i\neq j\}$ over $2^{10}$ different realizations of initial sequences. 
Here, $\delta=10^{-3}$, $n=32$, while $d$ varies.
We see that the curve $\Gamma_{\infty,\delta}$ defined in~\eqref{eq: gamma.infty} approximates the actual phase transition with increasing accuracy as $d$ grows, as implied by \Cref{thm: phase.transition.curve}.
}
    \label{fig: phase.diag.Id}
\end{figure}

This is clearly seen in Figure~\ref{fig: phase.diag.Id}, along with the fact that the resolution of this approximation increases with $d\to+\infty$.

Figure~\ref{fig: phase.diag.Id} appears to contain more information than what we may gather from Theorem~\ref{thm: d.infty}, Theorem~\ref{thm: orthogonal} and Theorem~\ref{thm: phase.transition.curve}. 
In particular, for small $d$, we see the appearance of a zone (white/light blue in Figure~\ref{fig: phase.diag.Id}) of parameters $(t,\beta)$ for which the probability of particles being clustered is positive, but not close to one. A careful inspection of this region reveals that points are grouped in a finite number of clusters; see Figure~\ref{fig: metastability}. The presence of such a zone indicates the emergence of a long-time \emph{metastable} state where points are clustered into several groups but eventually relax to a single cluster in long-time. This two-time-scale phenomenon is illustrated in Figure~\ref{fig: metastability} and prompts us to formulate the following question.   

\begin{open} \label{prob: metastability}
    Do the dynamics enter a transient metastable state, in the sense that for $\beta\gg1$, all particles stay in the vicinity of $m<n$ clusters for long periods of time, before they all collapse to the final cluster $\{x^*\}$?
\end{open}

There have been important steps towards a systematic theory of metastability for gradient flows, with applications to nonlinear parabolic equations---typically reaction-diffusion equations such as the Allen-Cahn or Cahn-Hilliard equations \cite{otto2007slow, kohn2002upper}---. 
While these tools to not readily apply to the current setup, they form an important starting point to answer this question.

\begin{figure}[h!]
    \centering
    \includegraphics[width=\textwidth]{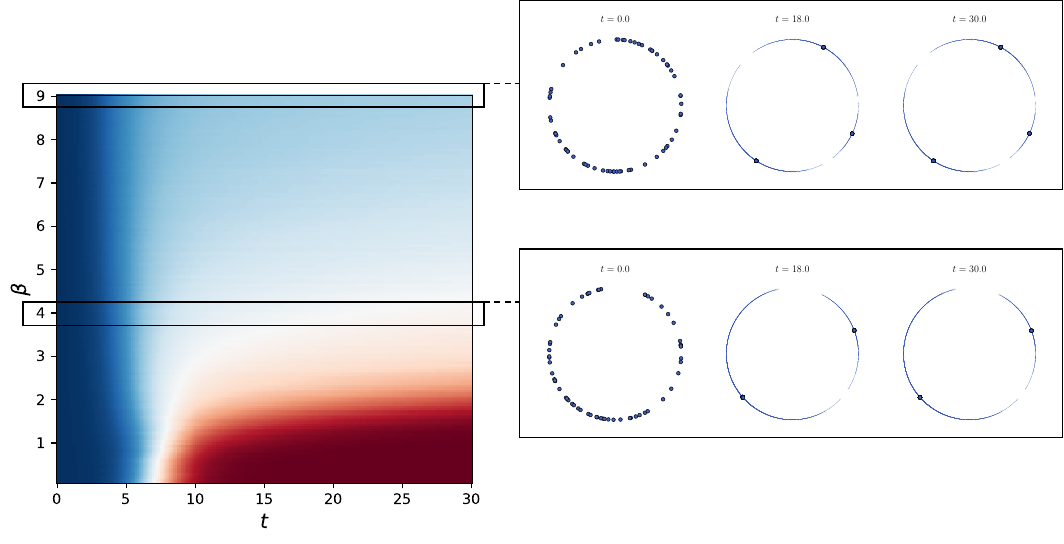}
    \caption{We zoom in on the phase diagram (Figure \ref{fig: phase.diag.Id}) for the dynamics on the circle: $d=2$. For $\beta=4, 9$, we also display a trajectory of~\eqref{SA} for a randomly drawn initial condition at times $t=2.5, 18, 30$. We see that the particles settle at $2$ clusters when $\beta=4$ (bottom right) and $3$ clusters when $\beta=9$ (top right), for a duration of time. This reflects our metastability claim for large $\beta$. 
    }
    \label{fig: metastability}
\end{figure}

Finally, one may naturally ask whether the clustering and phase diagram conclusions persist when the parameter matrices $(Q, K, V)$ are significantly more general: some illustrations\footnote{See \href{https://github.com/borjanG/2023-transformers-rotf}{\tt github.com/borjanG/2023-transformers-rotf} for additional figures which indicate that this phenomenon appears to hold in even more generality.} are given in Figure~\ref{fig: random.QKV}. 

\begin{open}
    Can the conclusions of \Cref{thm: orthogonal}--\Cref{thm: phase.transition.curve} be generalized to the case of random matrices $(Q, K, V)$?
\end{open}

    

\begin{figure}[h!]
    \centering
    \centering
\begin{subfigure}{.49\textwidth}
  \centering
  \includegraphics[scale=0.45]{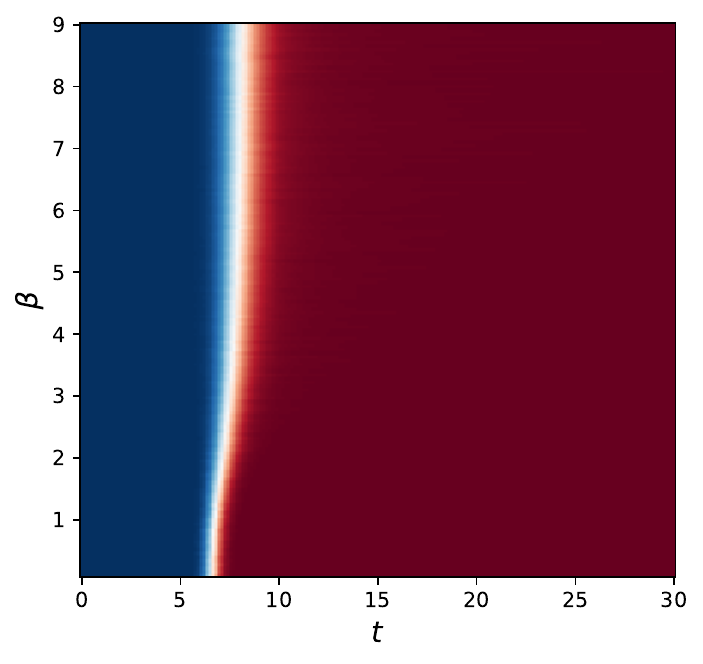}
  \caption{$Q, K, V$ in real Ginibre ensemble}
\end{subfigure}%
\begin{subfigure}{.49\textwidth}
  \centering
  \includegraphics[scale=0.45]{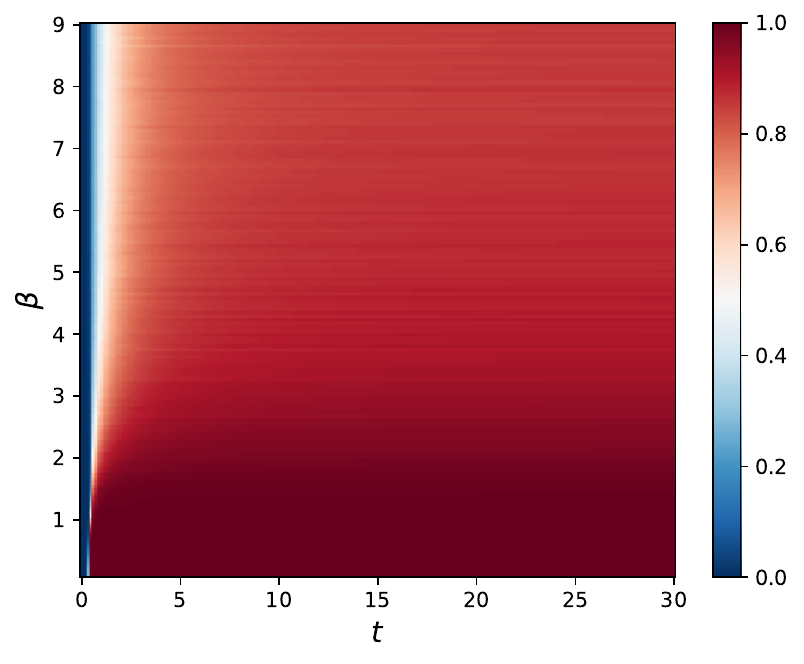}
  \caption{$Q^\top K$ Wigner, $V\succeq0$ GOE}
\end{subfigure}
\begin{subfigure}{.49\textwidth}
  \centering
  \includegraphics[scale=0.45]{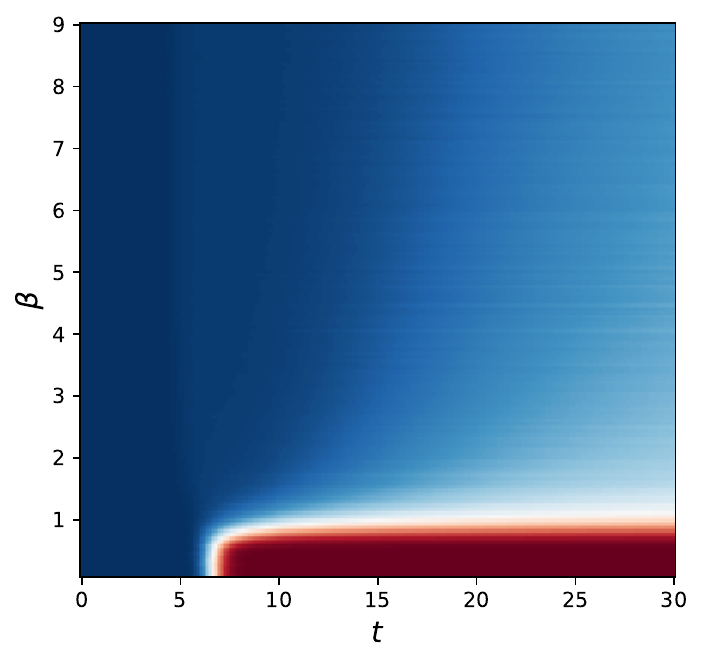}
  \caption{$Q, K$ in real Ginibre ensemble, $V=Q^\top K$}
\end{subfigure}%
\begin{subfigure}{.49\textwidth}
  \centering
  \includegraphics[scale=0.45]{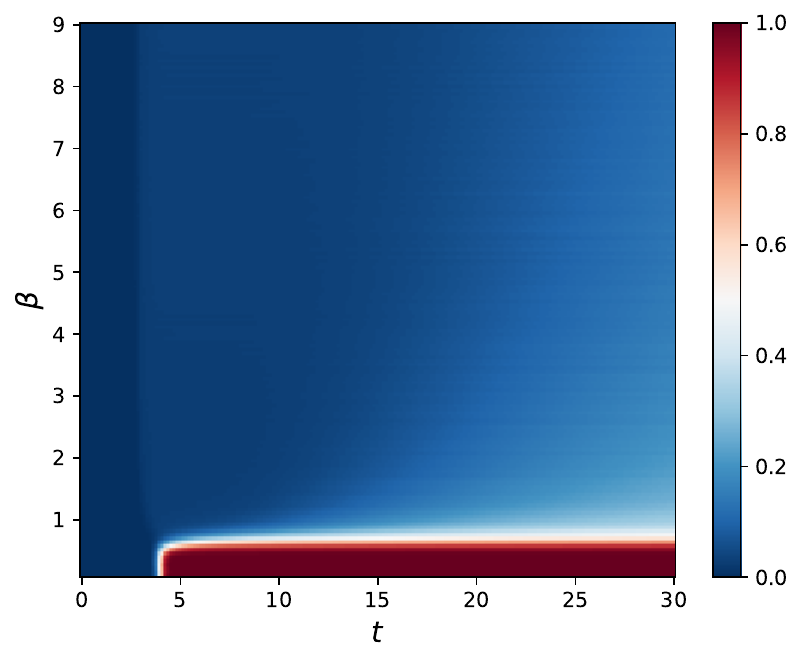}
  \caption{$Q^\top K$ Wigner, $V=Q^\top K$}
\end{subfigure}
\caption{Phase diagrams (see Figure \ref{fig: phase.diag.Id} for explanations) for some choices of random matrices $(Q, K, V)$; here $d=128$, $n=32$. Sharp phase transitions as well as metastable regions appear in all cases. 
}
\label{fig: random.QKV}
\end{figure}

\part{Further questions} \label{sec: beyond}

We conclude this manuscript by discussing several avenues of research that can lead to a finer understanding of the clustering phenomenon and generalizations of our results, and which, we believe, are of independent mathematical interest.
Specifically, 
\begin{itemize}
    \item In Section \ref{sec: circle}, we zero in on the special case $d=2$, where we make a link with the celebrated \emph{Kuramoto model} when $\beta=0$;
    \item In Section \ref{sec: bbgky}, we discuss an alternative approach for analyzing clustering, based on the so-called \emph{BBGKY hierarchy} from statistical mechanics;
    \item In Section \ref{sec: general.matrices}, we foray beyond the case $Q^\top K=V=I_d$ in a few directions. We start with Section \ref{sec:repulsive} by relating the case $V=-I_d$ to \emph{optimal configurations} on the sphere. We then discuss existing results in the absence of layer normalization in Section \ref{sec: neurips}, which motivate the study of a related singular equation (Section \ref{sec: filippov}) and a diffusive regularization (Section \ref{sec: diffusion});
    \item Finally, in Section \ref{sec: approximation}, we briefly discuss various ways focusing on the tuning of the parameter matrices themselves.
\end{itemize}

\section{Dynamics on the circle} \label{sec: circle}

We study the dynamics~\eqref{SA} and~\eqref{USA} in the special case $d=2$, namely on the unit circle $\mathbb{S}^1\subset\R^2$. This model, parametrized by angles and related to the celebrated Kuramoto model, is of independent interest and deserves a complete mathematical analysis.

\subsection{Angular equations} 
On the circle $\mathbb{S}^1$, all particles $x_i(t)\in\mathbb{S}^1$ are of course completely characterized by the angle $\theta_i(t)\in\mathbb{T}$:  $x_i(t)=\cos(\theta_i(t))e_1+\sin(\theta_i(t))e_2$ where $e_1=(1,0)$ and $ e_2=(0,1)\in \R^2$. 
We focus on the dynamics~\eqref{USA} for simplicity.
For any $i\in[n]$ and $t\geq0$, we may derive the equation  satisfied by $\theta_i(t)$ from $\cos(\theta_i(t))=\langle x_i(t), e_1\rangle$: differentiating in $t$ and plugging into~\eqref{USA} we obtain
\begin{align*}
    \dot{\theta}_i(t) = -\frac{n^{-1}}{\sin(\theta_i(t))}\left(\sum_{j=1}^n e^{\beta\langle x_i(t), x_j(t)\rangle}\Big[\langle x_j(t), e_1\rangle-\langle x_i(t), x_j(t)\rangle \langle x_i(t), e_1\rangle\Big]\right)
\end{align*}
where we used the definition of the projection (if $\theta_i(t)=0$ for some $t$, we differentiate the equality $\sin(\theta_i(t))=\langle x_i(t),e_2\rangle$ instead, which also leads to~\eqref{eq:onangles} in the end). 
Observing that 
\begin{equation*}
 \langle x_i(t), x_j(t)\rangle = \cos(\theta_i(t)-\theta_j(t)),   
\end{equation*}
we find 
\begin{align*}
    \dot{\theta}_i(t) &=- \frac{n^{-1}}{\sin(\theta_i(t))}\left(\sum_{j=1}^n e^{\beta\cos(\theta_i(t)-\theta_j(t))}\Big[\cos(\theta_j(t)) - \cos(\theta_i(t)-\theta_j(t))\cos(\theta_i(t))\Big]\right).
\end{align*}
Using elementary trigonometry, we conclude that
\begin{equation}\label{eq:onangles}
    \dot{\theta}_i(t) =-\frac1n\sum_{j=1}^n e^{\beta\cos(\theta_i(t)-\theta_j(t))}
    \sin(\theta_i(t)-\theta_j(t)).
\end{equation}
The case $\beta=0$ is exactly the Kuramoto model recalled in Section~\ref{s:kuramoto}. Suppose for the time being that $\beta>0$.
Defining the function $h_\beta:\mathbb{T}\to\R_{\geq0}$ as 
\begin{equation*}
    h_\beta(\theta)=e^{\beta\cos(\theta)},
\end{equation*}
we have effectively deduced that the empirical measure of the angles, $\nu(t)=\frac1n \sum_{j=1}^n \delta_{\theta_j(t)}$,
which is a measure on the torus $\mathbb{T}$, is a solution to the continuity equation
\begin{equation*}
\partial_t \nu(t) + \partial_\theta(\mathcal{X}[\nu(t)]\nu(t))=0, \hspace{1cm} \text{ on } \R_{\geq0}\times\mathbb{T},
\end{equation*}
where 
\begin{equation*}
    \mathcal{X}[\nu](\theta)=\frac{1}{\beta}\Big( h_\beta'\ast\nu\Big)(\theta).
\end{equation*}
When the particles $x_i(t)$ follow~\eqref{SA}, one readily checks that the same continuity equation is satisfied but rather with the field 
\begin{equation*}
 \mathcal{X}[\nu](\theta)=\frac{1}{\beta}\left(\frac{h_\beta'\ast\nu}{h_\beta\ast \nu}\right)(\theta).   
\end{equation*}


\subsection{The Kuramoto model} \label{s:kuramoto}
As mentioned above, when $\beta=0$, \eqref{eq:onangles} is a particular case of the Kuramoto model \cite{kuramoto1975self}:
\begin{equation}\label{e:kuramoto}
\dot{\theta}_i(t)=\omega_i+\frac{K}{n}\sum_{j=1}^n \sin(\theta_j(t)-\theta_i(t)),
\end{equation}
where $K>0$ is a prescribed coupling constant, and $\omega_i\in\mathbb{T}$ are the intrinsic natural frequencies of the oscillators $\theta_i(t)$.
It is known that for sufficiently small coupling strength $K$, the oscillators $\theta_i(t)$ in the Kuramoto model~\eqref{e:kuramoto} do not synchronize in long-time.
It is also known that when $K$ exceeds some critical threshold value, a phase transition occurs, leading to the synchronization of a  fraction of the oscillators. If $K$ is chosen very large, there is total synchronization of the oscillators in long-time. For more on the mathematical aspects of the Kuramoto model, we refer the reader to the review papers \cite{strogatz2000kuramoto, acebron2005kuramoto, ha2016collective} (see also \cite{carrillo2014contractivity, chiba2015proof, fernandez2016landau, dietert2018landau, ha2020asymptotic, townsend2020dense, abdalla2022expander, morales2022trend} for a non-exhaustive list of other recent mathematical results on the subject).

When all the frequencies $\omega_i$ are equal to some given frequency, $\omega\in\R$ say, after a change of variable of the form $\theta_i(t)\leftarrow \theta_i(t)-\omega t$, the dynamics in~\eqref{e:kuramoto} become the gradient flow
\begin{equation*}
    \dot{\theta}(t) = n\nabla\mathsf{F}(\theta)
\end{equation*}
where the energy $\mathsf{F}:\mathbb{T}^n\to\R_{\geq0}$ reads
\begin{equation}\label{e:energykuramoto}
\mathsf{F}(\theta)=\frac{K}{2n^2} \sum_{i=1}^n\sum_{j=1}^n\cos(\theta_i-\theta_j).
\end{equation}
The oscillators can be viewed as attempting to maximize this energy. The energy $\mathsf{F}$ is maximized when all the oscillators
are synchronized, that is, $\theta_i=\theta^*$ for some $\theta^*\in\mathbb{T}$ and for all $i\in[n]$. 
As the dynamics follow a gradient system, the equilibrium states are the critical points of the energy, namely those satisfying $\nabla \mathsf{F}(\theta)=0$. The local maxima of $\mathsf{F}$ correspond to equilibrium states $\theta$ that are physically achievable, since small perturbations thereof return the system back to $\theta$. 

Some authors consider a variant of the Kuramoto model where the oscillators are interacting according to the edges of a graph. In other words, the coefficients $A_{ij}$ of the graph's adjacency matrix are inserted in the sum in~\eqref{e:energykuramoto} as weights, and the dynamics are then the corresponding gradient flow. A recent line of work culminating with~\cite{abdalla2022expander} has established that synchronization occurs with high probability for Erd\H{o}s–Rényi graphs with parameter $p$, for every $p$ right above the connectivity threshold. 

Coming back to our dynamics~\eqref{eq:onangles}, we notice that it can also be written as a gradient flow on $\mathbb{T}^n$:
\begin{equation*}
 \dot{\theta}(t)=n\nabla \mathsf{E}_\beta(\theta(t)),  
\end{equation*}
for the interaction energy 
$\mathsf{E}_\beta:\mathbb{T}^n\to\mathbb{R}_{\geq0}$ defined as 
\begin{equation}\label{eq:energyF}
\mathsf{E}_\beta(\theta)=\frac{1}{2\beta n^2}\sum_{i=1}^n\sum_{j=1}^n e^{\beta\cos(\theta_i-\theta_j)},
\end{equation}
which is maximized when $\theta_i=\theta^*$ for some $\theta^*\in\mathbb{T}$ and for all $i\in[n]$.
In the spirit of \cite{ling2019landscape}, we suggest the following open problem---we recall that a critical point is called a \emph{strict saddle point} of $\mathsf{E}_\beta$ if the Hessian of $\mathsf{E}_\beta$ at these points has at least one positive eigenvalue.

\begin{open} \label{o:strictsaddle}
With the exception of the global maxima, are all critical points of $\mathsf{E}_\beta$ strict saddle points? 
\end{open}

The proofs of Theorems \ref{thm: beta.tiny} and \ref{thm: beta.interval} already yield a positive answer to Problem 5 in the regimes $\beta\leq1$ or $\beta\geq \frac{n^2}{\pi^2}$. The complementary regime remains open.
By classical arguments, recalled in Appendix \ref{app: proof.2}, a positive answer to Problem \ref{o:strictsaddle} would imply that for all initial conditions except a set of measure zero, all $\theta_i(t)$ converge under the dynamics \eqref{eq:onangles} to a common limit as $t\rightarrow+\infty$.

Extensions of the Kuramoto model of the form
\begin{equation}\label{e:variantkuramoto}
\dot{\theta}_i(t)=\omega_i+\frac{K}{n}\sum_{j=1}^n h(\theta_j(t)-\theta_i(t)),
\end{equation}
for a general non-linearity $h:\mathbb{T}\to\mathbb{R}$, which contains both~\eqref{e:kuramoto} and our model~\eqref{eq:onangles} as particular cases, have already been studied in the physics literature.
For instance, we refer the reader to \cite{daido1992order} (see also \cite[page 158]{acebron2005kuramoto}), where many heuristics are proposed to address the behavior of solutions to these dynamics. We are not aware of mathematical results for~\eqref{eq:onangles} besides Theorem \ref{thm: beta.interval}. We nevertheless have some hope that handling the dynamics~\eqref{eq:onangles} is easier than dealing with~\eqref{e:variantkuramoto} for a general $h$; for instance, we have 
\begin{equation*}
    h_\beta(\theta) = e^{\beta\cos(\theta)}= \sum_{k\in\mathbb{Z}} I_k(\beta)e^{\mathrm{i}k\theta}
\end{equation*}
where $I_k(\beta)$ are the modified Bessel function of the first kind, whose properties have been extensively studied.



\section{BBGKY hierarchy} \label{sec: bbgky}
For the sake of simplicity, we again focus on the dynamics on the circle $\mathbb{S}^1$, where recall that all particles are parametrized by angles (which we also refer to as particles).  
To carve out an even more complete understanding of the clustering phenomenon, it is natural to consider initial particles sampled i.i.d. from the uniform distribution on $\mathbb{S}^1$ and to study the time-evolution of the $r$-particle distribution $\rho^{(r)}_n(t,\theta_1,\ldots,\theta_r)$, defined as the joint law of the particles $\theta_1(t),\ldots,\theta_r(t)$. Otherwise put, it is the $r$-point marginal of the joint distribution $\rho^{(n)}(t,\cdot)\in\mathcal{P}(\mathbb{T}^n)$ of all $n$ particles.
Note that because of rotational invariance, $\rho^{(1)}(t,\cdot)$ is just the uniform distribution equal to $\frac{1}{2\pi}$ for all $t\geq0$. For $r=2$, again by rotational invariance, there exists some $\psi(t,\cdot):\mathbb{T}\rightarrow\R_{\geq0}$ such that
\begin{equation*}
\rho^{(2)}(t,\theta_1,\theta_2) = \frac{1}{2\pi}\psi(t, \theta_2-\theta_1).    
\end{equation*}
Proving the clustering/synchronization of all $\theta_i(t)$ in long-time amounts to proving that $\psi(t,\cdot)$ converges to a Dirac mass centered at $0$ as $t\rightarrow+\infty$.
Using the fact that $\rho^{(n)}(t,\cdot)$ solves the Liouville equation, by following the method used to derive the BBGKY\footnote{Bogoliubov–Born–Green–Kirkwood–Yvon.} hierarchy \cite{gallagher2013newton, golse2016dynamics}, it is possible to show that $\psi(t,\cdot)$ satisfies
\begin{equation} \label{e:transport}
\begin{dcases}
    \partial_t\psi(t, x)+\partial_x(v(t, x)\psi(t, x)) = 0 &\text{ in }\R_{\geq0}\times\mathbb{T}\\
    \psi(0,x) = (2\pi)^{-1} &\text{ in } \mathbb{T},
\end{dcases}
\end{equation}
where
\begin{equation*}
    v(t, x) = \frac{2}{\beta n} h_\beta'(x) -
\frac{2(n-2)}{\beta n} g(t,x),
\end{equation*}
and
\begin{equation*}
g(t,x) = \mathbb{E}\Big[-h_\beta'(\theta_3(t))\,\Big|\,\theta_1(t)=0,\ \theta_2(t)=x\Big].    
\end{equation*}
Note that the equation~\eqref{e:transport} is not closed since $g(t,x)$ depends on the $3$-point correlation function. This is typical in the BBGKY hierarchy, whereupon physical theory and experimental evidence is typically used to devise an ansatz for closing the system.
For instance, the Boltzmann equation is derived from the BBGKY hierarchy by assuming the \emph{molecular chaos hypothesis (Stosszahlansatz)} at the level of $r=2$.  
We suggest to close~\eqref{e:transport} in a way that reflects the formation of clusters:

\begin{open}
    Devise a realistic ansatz for $g(t,x)$ which allows to close equation~\eqref{e:transport}, and allows to prove the convergence of $\psi(t,\cdot)$ to a Dirac mass centered at $0$ as $t\to+\infty$. 
\end{open}

The derivation of a BBGKY hierarchy when $d>2$, as well as for~\eqref{SA}, are also problems which we believe merit further investigation.

\section{General matrices} \label{sec: general.matrices}

Figure~\ref{fig: random.QKV} hints at the likelihood of the clustering phenomenon being significantly more general than just the case $Q=K=V=I_d$.
However, extending our proofs to more general parameter matrices does not appear to be straightforward and is an open problem. Here we discuss some particular cases (without excluding other approaches). 

\subsection{The repulsive case} \label{sec:repulsive}

As seen from Lemma~\ref{lem: dissipation}, in the repulsive case $V=-I_d$ the interaction energy $\mathsf{E}_\beta$ decreases along trajectories. 
Recall that the unique global minimum of $\mathsf{E}_\beta$ over $\mathcal{P}(\S)$ is the uniform distribution (Proposition~\ref{prop: existence.uniqueness.energy}). In contrast, we explain in this section that many different configurations of $n$ points may yield global minima for $\mathsf{E}_\beta$ when minimized over empirical measures with $n$ atoms. 

We thus focus on minimizing $\mathsf{E}_\beta$ over the set $\mathcal{P}_n(\S)$ of empirical measures, namely sums of $n$ Dirac masses. Rewriting $\mathsf{E}_\beta$ as 
\begin{equation*}
\mathsf{E}_\beta[\mu]=\frac{e^{\beta}}{2\beta}\iint e^{-\frac{\beta}{2}\|x-x'\|^2}\diff\mu(x)\diff\mu(x'),
\end{equation*}
it turns out that minimizing $\mathsf{E}_\beta$ over $\mathcal{P}_n(\S)$ is precisely the problem of finding \emph{optimal configurations of points} on $\S$, which has direct links to the sphere packing problem \cite{cohn2007universally, cohn2022universal} and coding theory \cite{delsarte1991spherical}. For $\mu\in\mathcal{P}_n(\S)$, we can equivalently rewrite $\mathsf{E}_\beta$ in terms of the set of support points $\mathscr{C}\subset\S$, $\#\mathscr{C}=n$:
\begin{equation*}
\mathsf{E}_\beta[\mu]=\mathsf{H}_\beta[\mathscr{C}]=\frac{e^{\beta}}{2n^2\beta}\sum_{\substack{x,x'\in\mathscr{C}}} e^{-\frac{\beta}{2}\|x-x'\|^2}.
\end{equation*} 
In \cite{cohn2007universally}, Cohn and Kumar characterize the global minima $\mathscr{C}$ of $\mathsf{H}_\beta$. 
To state their result, we need the following definition. 

\begin{definition} Let $n\geq2$. A set of points $\mathscr{C}=\{x_1,\ldots, x_n\}\subset\S$ is called a \emph{spherical $t$-design} if 
\begin{equation*}
\int p(x)\diff\sigma_d(x)=\frac1n\sum_{i=1}^n p(x_i)
\end{equation*}
for all polynomials $p$ of $d$ variables, of total degree at most $t$. The set of points $\mathscr{C}$ is called a \emph{sharp configuration} if there are $m$ distinct inner products between pairwise distinct points in $\mathscr{C}$, for some $m>1$, and if it is  a spherical $(2m-1)$-design.
\end{definition}

The following result is a special case of \cite[Theorem 1.2]{cohn2007universally}.

\begin{theorem}[\cite{cohn2007universally}] Let $n\geq2$. Any global minimum of $\mathsf{H}_\beta$ among $\mathscr{C}\subset\S$, $\#\mathscr{C}=n$ is either a sharp configuration, or the vertices of a $600$-cell\footnote{A $600$-cell is a particular $4$-dimensional convex polytope with $n=120$ vertices.}.
\end{theorem}

The set of sharp configurations is not known for all regimes of $n, d$ or $m$ (the largest $m$ such that the configuration is a spherical $m$-design). A list of known examples is provided in \cite[Table 1]{cohn2007universally}: it consists of vertices of full-dimensional polytopes (specifically, regular polytopes whose faces are simplices), or particular derivations of the $E_8$ root lattice in $\R^8$ and the Leech lattice in $\R^{24}$. We refer the reader to \cite{cohn2007universally} and the illustrative experimental paper \cite{ballinger2009experimental} for further detail. A complete picture of the long-time behavior of Transformers in the repulsive case remains open.

\subsection{Pure self-attention} \label{sec: neurips}
An alternative avenue for conducting such an analysis which has shown to be particularly fruitful consists in removing the projector $\proj_x$,  leading to
\begin{equation} \label{eq: transf-1}
    \dot{x}_i(t) = \frac{1}{Z_{\beta, i}(t)}\sum_{j=1}^n e^{\beta\langle Qx_i(t), Kx_j(t)\rangle}Vx_j(t)
\end{equation}
for all $i\in[n]$ and $t\in\mathbb{R}_{\geq0}$. 
In fact, in \cite{geshkovski2023emergence} we analyze precisely these dynamics, and show different clustering results depending on the spectral properties of the matrix $V$. We briefly summarize our findings in what follows.

\subsubsection{A review of \cite{geshkovski2023emergence}}
For most choices of value matrices $V$, without rescaling time, most particles diverge to $\pm\infty$ and no particular pattern emerges. 
To make a very rough analogy,~\eqref{eq: transf-1} "looks like" $\dot{x}_i(t)=Vx_i(t)$ (which amounts to having $P_{ij}(t)=\delta_{ij}$ instead of~\eqref{eq:P}), whose solutions are given by $x_i(t)=e^{tV}x_i(0)$. To discern the formation of clusters, we introduce the rescaling\footnote{The rescaling~\eqref{eq:zifromxi} should be seen as a surrogate for layer normalization.}
\begin{equation}\label{eq:zifromxi}
z_i(t)=e^{-tV}x_i(t),
\end{equation}
which are solutions to
\begin{equation}\label{e:Rres}
\dot{z}_i(t) = \frac{1}{Z_{\beta,i}(t)}\sum_{j=1}^n e^{\beta\left\langle Qe^{tV}z_i(t),Ke^{tV}z_j(t)\right\rangle} V(z_j(t)-z_i(t)) \hspace{1cm}
\end{equation}
for $i\in[n]$ and $t\geq0$, where 
$$Z_{\beta,i}(t)=\sum_{k=1}^n e^{\beta\langle Qe^{tV}z_i(t),Ke^{tV}z_k(t)\rangle},$$
whereas the initial condition remains the same, namely $x_i(0)=z_i(0)$. It is crucial to notice that the coefficients $A_{ij}(t)$ (see~\eqref{eq:P}) of the self-attention matrix for the rescaled particles $z_i(t)$ are the same as those for the original particles $x_i(t)$. The weight $A_{ij}(t)$ indicates the strength of the attraction of $z_i(t)$ by $z_j(t)$. 
In \cite{geshkovski2023emergence} we show that the rescaled particles $z_i(t)$ cluster toward well-characterized geometric objects as $t\rightarrow+\infty$ for various choices of matrices $(Q,K,V)$. Our results are summarized in Table~\ref{table:results} below, whose first two lines are discussed thereafter.

{\footnotesize
\begin{table}[h!]
    \centering
    \begin{tabular}{c|c|c|c}
        $V$ &  $K$ and $Q$ & Limit geometry & Result in \cite{geshkovski2023emergence} \\
        \midrule
        $V=I_d$ & $Q^\top K\succ0$ & vertices of convex polytope &   Theorem~3.1 \\
       $\lambda_1(V)>0$, simple & $\langle Q\varphi_1,K\varphi_1\rangle>0$ & union of $3$ parallel hyperplanes & Theorem~4.1 \\
      $V$ paranormal & $Q^\top K\succ0$ & polytope $\times$ subspaces  & Theorem~5.1\\
     \midrule
        $V=-I_d$ & $Q^\top K=I_d$ &single cluster at origin$^\ast$ &   Theorem~C.5
    \end{tabular}
    \vspace{0.5cm}
    \caption{Summary of the clustering results of \cite{geshkovski2023emergence}. $^\ast$All results except for the case $V=-I_d$ hold for the time-scaled dynamics~\eqref{e:Rres}.} \label{table:results}
\end{table}
}

When $V=I_d$, outside from exceptional situations, all particles cluster to vertices of some convex polytope.
Indeed, since the velocity $\dot{z}_i(t)$ is a convex combination of the attractions $z_j(t)-z_i(t)$, the convex hull $\mathcal{K}(t)$ of the $z_i(t)$ shrinks and thus converges to some convex polytope. 
The vertices of the latter attract all particles as $t\rightarrow +\infty$.
When the eigenvalue with largest real part of $V$, denoted by $\lambda_1(V)$, is simple and positive, the rescaled particles $z_i(t)$ cluster on hyperplanes which are parallel to the direct sum of the eigenspaces of the remaining eigenvalues. Roughly speaking, the coordinates of the points $z_i(t)$ along the eigenvector of $V$ corresponding to $\lambda_1(V)$ quickly dominate the matrix coefficients $P_{ij}(t)$ in~\eqref{e:Rres} due to the factors $e^{tV}z_j(t)$. For more results and insights regarding clustering on $\R^d$, we refer the reader to \cite{geshkovski2023emergence}. We nonetheless leave the reader with the following general question:

\begin{open} \label{prob: Rd}
Is it possible to extend the clustering results of Table~\ref{table:results} to other cases of $(Q,K,V)$? What are the resulting limit shapes?
\end{open}

\subsection{Singular dynamics} \label{sec: filippov}
We mention another intriguing question, whose answer would allow for a transparent geometric understanding of clustering for~\eqref{e:Rres}. Let $(Q,K,V)$ be given $d\times d$ matrices. 
For $\beta>0$, we consider the system of coupled ODEs
\begin{equation}\label{e:exp}
\dot{z}_i(t)=\frac{1}{Z_{\beta,i}(t)}\sum_{j=1}^ne^{\beta\langle Qz_i(t),Kz_j(t)\rangle}V(z_j(t)-z_i(t)),
\end{equation}
where once again 
$$Z_{\beta,i}(t) = \sum_{k=1}^n e^{\beta\langle Qz_i(t),Kz_k(t)\rangle}.$$
For any $T>0$, and any fixed initial condition $(z_i(0))_{i\in[n]}\in (\R^d)^n$, as $\beta\rightarrow +\infty$, we expect that the solution to~\eqref{e:exp} converges uniformly on $[0,T]$ to a solution of 
\begin{equation}\label{e:maineq}
\dot{z}_i(t)=\frac{1}{|C_i(t)|}\sum_{j\in C_i(t)}V(z_j(t)-z_i(t))
\end{equation}
where 
\begin{equation}\label{e:Ci(t)}
C_i(t)=\Big\{j\in[n]\colon\langle Qz_i(t),Kz_j(t)\rangle \geq \langle Qz_i(t),Kz_k(t)\rangle\quad \text{ for all } k\in[n]\Big\}.
\end{equation}
However, defining a notion of solution to~\eqref{e:maineq}--\eqref{e:Ci(t)} is not straightforward, as illustrated by the following example.

\begin{example}\label{e:selection}
Suppose $d=2$, $n=3$. Let $Q=K=V=I_d$ and $z_1(0)=(1,1)$, $z_2(0)=(-1,1)$, $z_3(0)=(0,0)$. Consider the evolution of these particles through~\eqref{e:maineq}--\eqref{e:Ci(t)}. The points $z_1(t)$ and $z_2(t)$ do not move, because it is easily seen that $C_i(t)=\{i\}$ for $i\in\{1,2\}$. On the other hand, the point $z_3(t)$ can be chosen to solve either of three equations: $\dot{z}_3(t)=z_1(t)-z_3(t)$, or $\dot{z}_3(t)=z_2(t)-z_3(t)$, or even $\dot{z}_3(t)=\frac12(z_1(t)+z_2(t))-z_3(t)$. In any of these cases, both~\eqref{e:maineq} and~\eqref{e:Ci(t)} remain satisfied for almost every $t\geq0$.
\end{example}

It is possible to prove the existence of solutions to~\eqref{e:maineq}--\eqref{e:Ci(t)} defined in the sense of Filippov\footnote{We thank Enrique Zuazua for this suggestion.}: for this, we can either use a time-discretization of~\eqref{e:maineq}--\eqref{e:Ci(t)}, or use a convergence argument for solutions to~\eqref{e:exp} as $\beta\rightarrow+\infty$. Uniqueness however does not hold, as illustrated by Example~\ref{e:selection}. This naturally leads us to the following question: 

\begin{open}
Is it possible to establish a selection principle (similar to viscosity or entropy solutions) for solutions to~\eqref{e:maineq}--\eqref{e:Ci(t)} which allows to restore uniqueness?
In the affirmative, is it possible to revisit the clustering results of \cite{geshkovski2023emergence} and Problem~\ref{prob: Rd} in the setting of~\eqref{e:maineq}--\eqref{e:Ci(t)}?
\end{open}

\subsection{Diffusive regularization} \label{sec: diffusion}
We believe that~\eqref{e:maineq}--\eqref{e:Ci(t)} is also an original model for collective behavior. 
There are some similarities in spirit with methods arising in \emph{consensus based optimization} (CBO for short), \cite{pinnau2017consensus, carrillo2021consensus}. 
With CBO methods, one wishes to minimize a smooth and bounded, but otherwise arbitrary function $f:\mathbb{R}^d\to\mathbb{R}$ by making use of the Laplace method
\begin{equation*}
    \lim_{\beta\to+\infty}\left(-\frac{1}{\beta}\log\int_{\R^d} e^{-\beta f(x)}\diff\rho(x)\right)=\inf_{x\in \text{supp}(\rho)}f(x),
\end{equation*}
which holds for any fixed $\rho\in\mathcal{P}_{\text{ac}}(\R^d)$. This is accomplished by considering a McKean-Vlasov particle system of the form
\begin{equation*}
    \diff x_i(t) = -\lambda (x_i(t)-v_f)H^\epsilon(f(x_i(t))-f(v[\mu_n(t)]))\diff t +\sqrt{2}\sigma|x_i(t)-v[\mu_n(t)]|\diff W_i(t)
\end{equation*}
for fixed $\beta>0$, with drift parameter $\lambda>0$ and noise parameter $\sigma\geq0$; $H^\epsilon$ is a particular smoothed Heaviside function, and $\mu_n(t)$ is the empirical measure of the particles. The point $v[\mu]\in\R^d$ is a weighted average of the particles:
\begin{equation*}
    v[\mu] = \frac{1}{Z_{\beta,\mu}}\int_{\R^d} e^{-\beta f(x)}x\diff\mu(x)
\end{equation*}
where $Z_{\beta,\mu}=\int_{\R^d} e^{-\beta f(x)}\diff\mu(x)$.
Morally speaking, particles which are near a minimum of $f$ have a larger weight. The drift term is a gradient relaxation (for a quadratic potential) towards the current weighted average position of the batch of particles. The diffusion term is an exploration term whose strength is proportional to the distance of the particle from the current weighted average. Results of convergence to a global minimizer do exist, under various smallness assumptions on the initial distribution of the particles, and assumptions on the relative size of the coefficients. They rely on the analysis of the associated Fokker-Planck equation, see \cite{carrillo2021consensus, chaintron2022propagation}, and  also \cite{fornasier2021consensus} for the analog on $\S$.
We point out that similarities are mainly in spirit---these results and analysis are inapplicable to our setting because there is no analog for $f(x)$---. Nonetheless, they do raise the following interesting question:

\begin{open}
    What can be said about the long-time limit of Transformers with a noise/diffusion term of strength $\sigma>0$? 
\end{open}

The question is of interest for any of the Transformers models presented in what precedes. 

\section{Approximation, control, training} \label{sec: approximation}

Understanding the \emph{expressivity}, namely the ability of a neural network to reproduce any map in a given class (by tuning its parameters), is essential. 
Two closely related notions reflect the expressivity of neural networks: \emph{interpolation}---the property of exactly matching arbitrarily many input and target samples---and \emph{(universal) approximation}---the property of approximating input-target functional relationships in an appropriate topology---. We refer the reader to \cite{cheng2023interpolation} for a primer on the relationship between these two notions in the contex of deep neural networks. 

For discrete-time Transformers, universal approximation has been shown to hold in \cite{yun2019transformers}, making use of a variant of the architecture with translation parameters and letting the number of layers go to infinity; see also \cite{alberti2023sumformer, jiang2023approximation} and the review \cite{jiang2023brief}.

In the context of flow maps (from $\R^d$ to $\R^d$), it is now well understood that interpolation and approximation reflect the \emph{controllability} properties of the system. The transfer of control theoretical techniques to the understanding of expressivity has borne fruit, both in terms of controllability results \cite{agrachev2020control, cuchiero2020deep, tabuada2020universal, li2022deep, ruiz2023neural, veeravalli2023nonlinear, cheng2023interpolation} and optimal control insights \cite{li2018maximum, esteve2020large, geshkovski2022turnpike, esteve2023sparsity}. 
We refer the reader to \cite{agrachev2024generic, adu2024approximate, furuya2024transformers} for the first universal approximation results for Transformers, viewed as measure-to-measure maps, using control theoretic tools.

Besides approximation, understanding the training dynamics of Transformers is another major challenge which we haven't covered herein. As it is impossible to reference all works on this burgeoning topic, we refer the interested reader to \cite{tarzanagh2023transformers, ahn2024transformers, deora2023optimization} and references therein.

\section*{Acknowledgments}

We thank Pierre Ablin, S\'ebastien Bubeck, Gabriel Peyré, Matthew Rosenzweig, Sylvia Serfaty, Kimi Sun, and Rui Sun for discussions. We thank Nicolas Boumal for referring us to \cite{markdahl2017almost, criscitiello2024synchronization} and for clarifying comments.

B.G. acknowledges financial support from the French government managed by the National Agency for Research under the France 2030 program, with the reference ”ANR-23-PEIA-0004”.
The work of Y.P. was supported in part by the MIT-IBM Watson AI Lab and by the National Science Foundation under Grant No CCF-2131115.
P.R. was supported by NSF grants DMS-2022448, CCF-2106377, and a gift from Apple.

\appendix
\part*{Appendix}

\section{Proof of Theorem~\ref{p:beta0}} \label{app: proof.2}
The proof of Theorem~\ref{p:beta0} relies on standard arguments from dynamical systems, upon noticing that the evolution~\eqref{e:Snonres0} is a (continuous-time) gradient ascent for the energy $\mathsf{E}_0:(\S)^n\rightarrow\R$ defined as
\begin{equation*}
    \mathsf{E}_0(x_1,\ldots,x_n)=\frac1n\sum_{i=1}^n\sum_{j=1}^n \langle x_i,x_j\rangle.
\end{equation*}
Since the dynamics are the gradient ascent of a real-analytic functional on the compact real-analytic manifold $(\S)^n$, the celebrated \L{}ojasiewicz theorem~\cite{lojasiewicz1963propriete}, in the form given by \cite[Corollary 5.1]{ha2018relaxation}---which is valid in the context of general compact Riemannian manifolds---, implies that for any initial condition $X\in (\S)^n$, the solution $\Phi^t(X)\in(\S)^n$ converges to some critical point $X^*\in(\S)^n$ of $\mathsf{E}_0$ as $t\to+\infty$.

We recall that a \emph{strict saddle point} of $\mathsf{E}_0$ is a critical point of $\mathsf{E}_0$ at which the Hessian of $\mathsf{E}_0$ has at least one strictly positive eigenvalue. 
Theorem \ref{p:beta0} then follows by combining the following couple of lemmas with the \L{}ojasiewicz theorem.

\begin{lemma} \label{l:nosaddleconv}
Let $\mathcal{M}$ be a compact Riemannian manifold and let $f:\mathcal{M}\rightarrow\R$ be a smooth function. The set of initial conditions $X_0\in \mathcal{M}$ for which the gradient ascent
\begin{equation}\label{e:gradfl}
\begin{cases}
    \dot{X}(t)=\nabla f(X(t)) &\\
    X(0)=X_0
\end{cases}
\end{equation}
converges to a strict saddle point of $f$ is of volume zero.
\end{lemma}

\begin{proof}[Proof of Lemma \ref{l:nosaddleconv}] 
Let us denote by $\Phi^t(X_0):=X(t)$, $t\geq 0$ the solution to \eqref{e:gradfl}. 
We denote by $\mathscr{S}\subset\mathcal{M}$ the set of strict saddle points of $f$, and by $\mathscr{A}\subset\mathcal{M}$ the set of initial conditions $X_0\in\mathcal{M}$ for which $\Phi^t(X_0)$ converges to a strict saddle point of $f$ as $t\rightarrow+\infty$.
For any $y\in \mathscr{S}$, we denote by $B_y$ a ball in which the local center-stable manifold $W^{\text{sc}}_{\text{loc}}(y)$ exists (see \cite{shub2013global}, Theorem III.7 and Exercise III.3 for the adaptation to flows). 
Using compactness, we may write the union of these balls as a countable union $\bigcup_{k\in I} B_{y_k}$ (where $I$ is countable and $y_k\in \mathcal{M}$ for $k\in I$). If $X_0\in\mathscr{A}$, there exists some $t_0\geq 0$ and $k\in I$ such that $\Phi^t(X_0)\in B_{y_k}$ for all $t\geq t_0$.
From the center-stable manifold theorem (\cite{shub2013global}, Theorem III.7 and Exercise III.3,
where we note that the Jacobian of a gradient vector field coincides, at a critical point, with the
Hessian of the corresponding function) we gather that $\Phi^t(X_0)\in W^{\text{sc}}_{\text{loc}}(y_k)$ for $t\geq t_0$, hence
$X_0\in \Phi^{-t}(W^{\text{sc}}_{\text{loc}}(y_k))$ for all $t\geq t_0$. The dimension of $W^{\text{sc}}_{\text{loc}}(y_k)$ is at most $\text{dim}(\mathcal{M})-1$, thus it has zero volume. Since $\Phi^t$ is a diffeomorphism on a compact manifold, $\Phi^{-t}$ preserves null-sets and hence $\Phi^{-t}(W^{\text{sc}}_{\text{loc}}(y_k))$ has zero volume for all $t\geq 0$. Therefore $\mathscr{A}$, which satisfies
$$
\mathscr{A}\subset \bigcup_{k\in I} \bigcup_{\ell\in\mathbb{N}} \Phi^{-\ell}(W^{\text{sc}}_{\text{loc}}(y_k))
$$
has volume zero. 
\end{proof}

\begin{lemma} \label{lem: yury.lemma}
Any critical point $(x_1,\ldots,x_n)\in(\S)^n$ of $\mathsf{E}_0$ which is not a global maximum, namely such that $x_1=\ldots=x_n$, is a strict saddle point. In particular, all local maxima are global.
\end{lemma}

\begin{proof}[Proof of Lemma \ref{lem: yury.lemma}]
We extend the proof idea of~\cite[Theorem 4.1]{taylor2012there} as follows. Let $(x_1,\ldots,x_n)\in(\S)^n$ be a critical point of $\mathsf{E}_0$, and assume that the points $x_i$ are not all equal to each other. 

\subsubsection*{Step 1} We first prove that there exists a set of indices $\mathscr{S} \subset [n]$ such that
\begin{equation}\label{eq:taylor}
   \sum_{i\in \mathscr{S}} \sum_{j\in \mathscr{S}^c} \langle x_i,x_j\rangle < 0.
\end{equation}
To this end, define
$$ m := \sum_{j=1}^n x_j,$$
and consider two cases. If $m\neq 0$, then we deduce from $\nabla \mathsf{E}_0(x_1,\ldots,x_n)=0$ that for any $j\in[n]$, $x_j$ is collinear with $m$. Thus $x_j = \pm x_1$ for any $j\in[n]$. 
Setting 
$$\mathscr{S}=\{j\in[n]\colon  x_j = +x_1\},$$ 
we can see that~\eqref{eq:taylor} holds, unless $\mathscr{S}=[n]$ which has been excluded. 
Now suppose that $m=0$. Then by expanding $\langle m, x_i\rangle = 0$, we find that for any $i\in[n]$
$$ -1 = \sum_{j=2}^n \langle x_j, x_1 \rangle\,, $$ 
holds, which again implies~\eqref{eq:taylor} with $\mathscr{S}=\{1\}$.

\subsubsection*{Step 2} In this second step we look to deduce from \eqref{eq:taylor} that $(x_1,\ldots,x_n)$ is a strict saddle point. Consider an arbitrary non-zero skew-symmetric matrix $B$ and define the perturbation
$$ x_i(t) = \begin{cases}
    x_i & i\not \in \mathscr{S}\\
    e^{tB}x_i & i \in \mathscr{S}.
\end{cases}$$
Set $\mathsf{E}_0(t) = \mathsf{E}_0(x_1(t),\ldots,x_n(t))$. Note that we have 
$$ \mathsf{E}_0(t) = \mathrm{const.} + \frac{2}{n} \sum_{i\in\mathscr{S}}\sum_{j\in \mathscr{S}^c} \langle x_i(t), x_j \rangle\,,$$
where we grouped time-independent terms into the constant (recall that $e^{tB}$ is an orthogonal matrix, since skew-symmetric matrices are the Lie algebra of $\text{SO}(d)$). 
Thus
\begin{align*}
    \mathsf{E}_0'(t) &= \frac{2}{n} \sum_{i\in\mathscr{S}} \sum_{j\in \mathscr{S}^c} \langle \dot{x_i}(t), {x_j}\rangle \\
    \mathsf{E}_0''(t) &= \frac{2}{n} \sum_{i\in\mathscr{S}} \sum_{j\in \mathscr{S}^c} \langle \ddot{x_i}(t), {x_j}\rangle\,.
\end{align*}
Since $(x_1,\ldots,x_n)$ is a critical point of $\mathsf{E}_0$, we have $\mathsf{E}_0'(0)=0$. 
On the other hand, since $\ddot x_i(0) = B^2 x_i$ we have 
\begin{equation}\label{e:helpcl} \mathsf{E}_0''(0) = \frac2n\sum_{i\in \mathscr{S}} \sum_{j\in \mathscr{S}^c} \langle B^2 x_i, x_j\rangle\,. 
\end{equation} 
We claim that given~\eqref{eq:taylor}, there must exist some skew-symmetric matrix $B$ such that $\mathsf{E}_0''(0)>0$. 
Indeed, if $d$ is even, then we just take $B$ as the block-diagonal matrix with repeated block 
$$\begin{bmatrix}
    0 & 1 \\ -1 & 0
\end{bmatrix},$$ 
so that $B^2 = -I_d$. If $d$ is odd, we can represent 
\begin{equation}\label{e:russiantrick}
-I_d = {1\over d-1}\sum_{j=1}^d B_j^2,
\end{equation}
where $B_j$ is the same block-diagonal matrix, with the exception that the $j$-th block is a $1\times1$ zero-matrix. 
If each $B_j$ were to yield $\mathsf{E}_0''(0)\le 0$, then it would violate~\eqref{eq:taylor}. Thus, $\mathsf{E}_0''(0)> 0$ for some well-chosen skew-symmetric $B$, which proves that $(x_1,\ldots,x_n)$ is a strict saddle point.
\end{proof}

\section{Proof of Theorem \ref{thm: beta.interval}} \label{apx:beta_high}

\begin{proof}[Proof of Theorem \ref{thm: beta.interval}] 

We leverage the gradient flow structure presented in Remark \ref{rem: particle.gf} and Section \ref{sec: new.gradient.flow} (the manifold is compact, and the metric and functional are analytic), and use Lemma \ref{l:nosaddleconv} as in the proof of Theorem \ref{p:beta0}. Consequently, it suffices to show that, in the stated regime of $\beta$, the critical points of $\mathsf{E}_\beta$ which are not global maxima are strict saddle points, namely, that all local maxima are global. For simplicity we write the argument for \eqref{USA} and explain the extension to the case of \eqref{SA} in Remark \ref{r:extensiontoSAcase}.

We begin by focusing on the case $d=2$, and provide a brief argument which shows that the case of arbitrary $d\geq2$ readily follows.

Let $(\theta_1,\ldots,\theta_n)\in\mathbb{T}^n$ be a critical point such that all eigenvalues of the Hessian of $\mathsf{E}_\beta$ are
non-positive. We intend to show that if $\beta$ is sufficiently large, then necessarily
$\theta_1=\cdots = \theta_n$. 
To that end, note that the non-positivity of the Hessian of $\mathsf{E}_\beta$ implies in particular that for any subset of indices $\mathscr{S} \subset [n]$, we must have
\begin{equation} \label{eq:taylor2}
	\sum_{i \in \mathscr{S}}\sum_{j\in \mathscr{S}} \partial_{\theta_i} \partial_{\theta_j} \mathsf{E}_\beta(\theta_1,\ldots,\theta_n) \le 0\,.
\end{equation}
Notice that for any $i, j\in[n]$,
$$ \partial_{\theta_i} \mathsf{E}_\beta(\theta_1,\ldots,\theta_n) = {1\over n^2} \sum_{m\in[n]\setminus\{i\}} -\sin(\theta_i-\theta_m)
e^{\beta \cos(\theta_i-\theta_m)} $$
and
\begin{equation*}
\partial_{\theta_i}\partial_{\theta_j} \mathsf{E}_\beta(\theta_1,\ldots,\theta_n) = {1\over n^2}\cdot
\begin{dcases} 
g(\theta_i - \theta_j), & i\neq j\\
-\sum_{m\in[n]\setminus\{i\}} g(\theta_i-\theta_m), & i = j\,,
\end{dcases}
\end{equation*} 
where we set $g(x) := (\cos(x) - \beta\sin^2(x))e^{\beta \cos(x)}$. 
Plugging this expression back into~\eqref{eq:taylor2} and simplifying, we obtain
\begin{equation}\label{eq:taylor3}
	\sum_{i\in \mathscr{S}} \sum_{j\in \mathscr{S}^c} g(\theta_i - \theta_j) \ge 0\,.
\end{equation}
Let us now define $\tau^*_\beta$ be the unique solution on $[0,\frac{\pi}{2})$ of the equation
$$ \beta \sin^2(\tau) = \cos(\tau)\,.$$
Note that $\tau^*_\beta$ is a monotonically decreasing function of $\beta$, and in fact 
$$\tau^*_\beta =
{1+o(1)\over\sqrt{\beta}}$$ 
as $\beta\to+\infty$. 
The importance of $\tau^*_\beta$ is in implying the following property of the function $g$: for any $\tau\not\in[-\tau^*_\beta,
\tau^*_\beta]$, we must have that $g(\tau)<0$ (see Figure \ref{fig: g(x)}).  
We arrive at the following conclusion: it must be that for any proper subset $\mathscr{S} \subset [n]$
there exists, by virtue of~\eqref{eq:taylor3}, some index $j\in \mathscr{S}^c$ such that 
\begin{equation*}
    \inf_{k\in\mathscr{S}}|\theta_j-\theta_k|<\tau_\beta^*.
\end{equation*} 
So now let us start with $\mathscr{S}=\{1\}$ and grow $\mathscr{S}$ inductively by adding those points $\theta_j$ at distance $<\tau^*_\beta$ from $\{\theta_k\colon k\in\mathscr{S}\}$ at each induction step.  
If
$\beta$ is large enough so that 
$$ (n-1) \tau^*_\beta < \frac{\pi}{2},$$
then in the process of adding points we have travelled a total arc-length $<\pi/2$ on each side of $x_1$. Thus it must be that
the collection of points $\theta_1,\ldots,\theta_n$ is strictly contained inside a half-circle of
angular width $<\pi$. By Lemma~\ref{lem: hemisphere.clustering}  we know
that there can be no critical points of $\mathsf{E}_\beta$ that are strictly inside some half-circle, unless that critical point is trivial: $\theta_1=\cdots=\theta_n$. This completes the proof when $d=2$.

\begin{figure}[h!]
    \centering
    \includegraphics[scale=1.05]{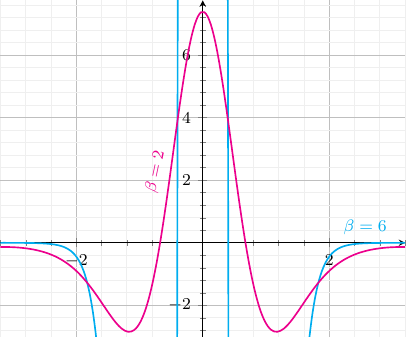}
    \caption{The function $\tau\mapsto g(\tau)$ for two values of $\beta$.} 
    \label{fig: g(x)}
\end{figure}

We can show that the same conclusion holds for any dimension $d\geq2$.
The proof follows by arguing just as above, making instead use of the following generalization of~\eqref{eq:taylor3}: given a collection $x_1,\ldots,x_n\in\mathbb{S}^{d-1}$ at which the Hessian of $\mathsf{E}_\beta$ is non-positive, we
must have for any subset $\mathscr{S}\subset[n]$ that 
\begin{equation} \label{eq: claim.yury}
\sum_{i\in \mathscr{S}} \sum_{j\in \mathscr{S}^c} g(\theta_{ij}) \ge 0,    
\end{equation} 	
where $g(\zeta)=e^{\beta \cos(\zeta)}((d-1)\cos(\zeta) - \beta \sin^2(\zeta))$ and $\theta_{ij}\in[0,\pi]$ is the geodesic distance between $x_i$ and $x_j$, namely $\cos(\theta_{ij}) =
\langle x_i, x_j \rangle$. 
We now show \eqref{eq: claim.yury}.
By repeating the argument in Step 2 of the proof of Lemma \ref{lem: yury.lemma}
, we see that for any skew-symmetric matrix $B$ we must have
\begin{equation}\label{eq:dr1}
\sum_{i\in \mathscr{S}} \sum_{j\in \mathscr{S}^c} e^{\beta \langle x_i,x_j \rangle} \Big( \beta
\langle B x_i, x_j\rangle^2 + \langle B^2 x_i, x_j \rangle\Big) \le 0.
\end{equation}
Now we take $B$ to be random by generating $B_{ij} \stackrel{\text{i.i.d.}}{\sim} P$, $i<j$ and $P$ being any zero-mean, unit-variance distribution. 
We set $B_{ji} = -B_{ij}$ and $B_{ii} = 0$. Then it is easy to
check that 
\begin{equation*}
 \mathbb{E}[B^2] = -(d-1)I_d    
\end{equation*}
and
\begin{equation*}
    \mathbb{E}[\langle B x_i, x_j\rangle^2] = 1-\langle x_i,x_j\rangle^2 = \sin^2(\theta_{ij}).
\end{equation*}
Thus, taking the expectation over all such $B$ in~\eqref{eq:dr1} yields \eqref{eq: claim.yury}. 
Mirroring the proof for $d=2$, we define $\tau_\beta^*$ to be the unique solution on $[0,\frac\pi2)$ of the equation $\beta\sin^2(\tau)=(d-1)\cos(\tau)$. We note that  
$$\tau^*_\beta = \sqrt{\dfrac{(d-1)+o(1)}{\beta}}$$ for $\beta\to +\infty$. Repeating verbatim the argument for the case $d=2$, we deduce the convergence to a single cluster whenever $\beta\gtrsim (d-1)n^2$.
\end{proof}

\begin{remark} \label{r:extensiontoSAcase}
We comment on the extension of the above proof to the dynamics \eqref{SA}. We recall that
\eqref{SA} is a gradient flow, but for a different metric---see Section~\ref{sec:
new.gradient.flow}---and we show that the saddle point property is preserved across metrics.  Our proof is an adaptation of a classical argument: the Hessian of a function at a critical point is a notion which does not depend on the choice of Riemannian metric.

Let $x=(x_1,\ldots,x_n)\in(\S)^n$ be a critical point of $\mathsf{E}_\beta$ (this does not depend on the metric) such that not all $x_i$ are equal to each other. 
Recall that for $f:(\S)^n\to\R$, for any metric on $(\S)^n$ (with associated Christoffel symbols $\Gamma_{ij}^k$) and any associated orthonormal basis $y_1,\ldots,y_{(d-1)n}$, the Hessian of $f$ reads
\begin{equation}\label{e:Hessianincoord}
    \mathrm{Hess}(f) = \left(\frac{\partial^2f}{\partial y_i\partial y_j}-\Gamma_{ij}^k\frac{\partial f}{\partial y_k}\right)dy_i\otimes dy_j.
\end{equation}
Since we are evaluating the Hessian at a critical point $x$ of $\mathsf{E}_\beta$, the term carrying the Christoffel symbols $\Gamma_{ij}^k$ vanishes. 
In the above argument, we saw that $\mathrm{Hess}\,\mathsf{E}_\beta$ evaluated at $x$, and written in an orthonormal basis for the canonical metric $g$ on $(\S)^n$, is not negative semi-definite. We denote this matrix by $M_g$; we know that there exists $v\in \Tan_x (\S)^n$ such that $g(v,v)=1$ and $v^\top M_gv>0$. 
Let $\tilde{g}$ be another metric on $(\S)^n$; we denote by $M_{\tilde{g}}$ the Hessian evaluated at $x$, and written in an orthonormal basis for $\widetilde{g}$. 
Let $c:\mathbb{R}_{\geq 0}\to (\S)^n$ be such that $c(0)=x$ and $\dot{c}(0) = v$. Since $x$ is a critical point (for both metrics), a Taylor expansion to second order in the two orthonormal bases yields
\begin{equation*}
    \mathsf{E}_\beta(c(t)) = \mathsf{E}_\beta(c(0)) + \frac12 t^2 v^\top M_g v + O(t^3)
\end{equation*}
as well as 
\begin{equation*}
\mathsf{E}_\beta(c(t))=\mathsf{E}_\beta(c(0))+\frac12 t^2 \|v\|_{\tilde{g}}^{-2}v^\top M_{\tilde{g}}v+O(t^3)
\end{equation*}
thanks to \eqref{e:Hessianincoord}.
Hence $v^\top M_{\tilde{g}}v>0$. Specializing to $\tilde{g}$ being the metric of Section \ref{sec: new.gradient.flow}, with respect to which \eqref{SA} is a gradient flow for $\mathsf{E}_\beta$, we conclude for \eqref{SA}.
\end{remark}

\section{Proof of Theorem \ref{thm: beta.tiny}} \label{app: beta.tiny}

\begin{proof}[Proof of Theorem \ref{thm: beta.tiny}] 
We leverage the gradient flow structure and follow the same strategy as in the proof of Theorem \ref{thm: beta.interval} presented above. For simplicity we write the argument for \eqref{USA} and explain the extension to the case of \eqref{SA} in Remark \ref{r:extensiontoSAcase.bis}.

    Consider
\begin{equation*}
    \mathsf{E}_\beta(x_1,\ldots,x_n) = \frac{1}{2\beta}\sum_{i=1}^n \sum_{j=1}^n \left(e^{\beta\langle x_i, x_j\rangle}-1\right).
\end{equation*}
Note that this is only a slight deviation from the energy studied in 
Section \ref{sec: new.gradient.flow}: we solely subtracted a constant. Consequently the dynamics \eqref{USA} are also a gradient flow for this energy. The main interest of considering this modified energy is the observation that
$$
\mathsf{E}_\beta(x_1,\ldots,x_n)=\mathsf{E}_0(x_1,\ldots,x_n)+\beta\, \mathsf{R}_\beta(x_1,\ldots,x_n),
$$
where $\mathsf{R}_\beta$ is smooth. Hence $\mathsf{R}_\beta$ has a bounded Hessian on $(\S)^n$ uniformly with respect to $\beta$, and 
\begin{equation}\label{e:contebeta}
\nabla \mathsf{E}_\beta=\nabla \mathsf{E}_0+O(\beta),\qquad \mathrm{Hess }\,\mathsf{E}_\beta=\mathrm{Hess }\,\mathsf{E}_0+O(\beta).
\end{equation}
Observe that in the proof of Theorem \ref{p:beta0}, we actually showed that there exists $c>0$ such that at any critical point $(x_1,\ldots,x_n)$ of $\mathsf{E}_0$ for which $x_i\neq x_j$ whenever $i\neq j$, at least one of the eigenvalues of the Hessian of $\mathsf{E}_0$, $\lambda$ say, satisfies $\lambda\geq c$. Indeed, in \eqref{eq:taylor} the proof actually shows the existence of some $\mathscr{S}\subset[n]$ such that
\begin{equation*}\label{eq:taylor22}
   \sum_{i\in \mathscr{S}} \sum_{j\in \mathscr{S}^c} \langle x_i,x_j\rangle \leq -1.
\end{equation*}
Then, \eqref{e:helpcl}, together with \eqref{e:russiantrick} for instance, yield 
\begin{equation} \label{eq: eigval.beta}
  \mathsf{E}_0''(0)\geq \frac{2(d-1)}{dn}=:c  
\end{equation}
for one of the $B_j$.

Now suppose that there exists a positive sequence $\beta_k\rightarrow 0$ as well as $X_k\in(\S)^n$ such that $X_k$ is a critical point of $\mathsf{E}_{\beta_k}$ and all of the eigenvalues of $\mathrm{Hess}\,\mathsf{E}_{\beta_k}(X_k)$ are non-positive. Then by virtue of the continuity properties of $\mathsf{E}_\beta$ with respect to $\beta$ in \eqref{e:contebeta}, we find that, up to extracting a subsequence, there is some limit point $\overline{X}=(x_1,\ldots,x_n)\in(\S)^n$ of $X_k$ which is a critical point of $\mathsf{E}_0$, and such that all of the eigenvalues of $\mathrm{Hess}\,\mathsf{E}_0(\overline{X})$ are non-positive. Per Theorem \ref{p:beta0}, this implies that $x_1=\ldots=x_n$. But then, for large enough $k$, $X_k$ is also constituted of points which are all nearly equal, whence in the same hemisphere, and the only such critical point of $\mathsf{E}_\beta$ is that in which all points are equal (synchronized).
This, combined with the continuity of the eigenvalues of $\mathrm{Hess}\,\mathsf{E}_\beta$ with respect to $\beta$ and \eqref{eq: eigval.beta}, proves that there exists some $c>0$ independent of $n$ such that whenever $\beta\leq c\,n^{-1}$, all critical points of $\mathsf{E}_\beta$ except synchronized ones are strict saddle points.
\end{proof}

\begin{remark} \label{r:extensiontoSAcase.bis}
We comment on the extension of the above proof to the dynamics \eqref{SA}. 
The point of contention is \eqref{e:contebeta}, since the metric with respect to which the gradient and Hessian of $\mathsf{E}_0$ are taken is not the same as that for $\mathsf{E}_\beta$. Denote the modified metric defined in Section \ref{sec: new.gradient.flow} by $g_\beta$, and the canonical metric by $g$. For any $x\in(\S)^n$ and $v\in\Tan_x(\S)^n$ we have
\begin{equation*}
    \mathrm{D}\mathsf{E}_\beta(x)[v] = g_\beta(\nabla_{g_\beta}\mathsf{E}_\beta(x), v),
\end{equation*}
but also $\mathrm{D}\mathsf{E}_\beta(x)[v]=g(\nabla_{g}\mathsf{E}_\beta(x),v)$. By virtue of the explicit form of $g_\beta$ and $\mathsf{E}_\beta$ as well as \eqref{e:contebeta}, we gather that 
\begin{equation} \label{eq: metric.grad}
    g_\beta(\nabla_{g_\beta} \mathsf{E}_\beta(x), v) = g(\nabla_g \mathsf{E}_0(x), v)+O(\beta)
\end{equation}
which implies that any sequence of critical points of $\mathsf{E}_\beta$ converges to a critical point for $\mathsf{E}_0$.
Similarly, since $\mathrm{Hess}_{g_\beta}\mathsf{E}_\beta(x)[v]=\mathrm{D}(\nabla_{g_\beta}\mathsf{E}_\beta(x))[v]$, we find
\begin{equation} \label{eq: metric.hess}
    \mathrm{Hess}_{g_\beta} \mathsf{E}_\beta(x)[v] = \mathrm{Hess}_g\mathsf{E}_0(x)[v]+O(\beta).
\end{equation}
We can then repeat the argument in the proof above by replacing \eqref{e:contebeta} by \eqref{eq: metric.grad} and \eqref{eq: metric.hess}.
\end{remark}

\section{Proof of Theorem~\ref{thm: phase.transition.curve}}
\label{app: proof.1}

\begin{proof} We focus on the dynamics~\eqref{SA}, since the proof for~\eqref{USA} follows from very similar computations. 

\subsubsection*{Step 1. The flow map is Lipschitz.} We begin by showing that the trajectories satisfy a Lipschitz property with respect to the initial data.
To this end, let $(x_i(\cdot))_{i\in[n]}\in C^0(\R_{\geq0};(\S)^n)$ and $(y_i(\cdot))_{i\in[n]}\in C^0(\R_{\geq0};(\S)^n)$ be two solutions to the Cauchy problem for~\eqref{SA} associated to data $(x_i(0))_{i\in[n]}$ and $(y_i(0))_{i\in[n]}$ respectively. For any $i\in[n]$ and $t\geq0$, we have
\begin{align} \label{eq: lip.1}
x_i(t) - y_i(t) &= x_i(0) - y_i(0)\nonumber\\
&\hspace{0.5cm}+ \int_0^t \sum_{j=1}^n\left(\frac{e^{\beta\langle x_i(s), x_j(s)\rangle}}{\sum_{k=1}^n e^{\beta\langle x_i(s), x_k(s)\rangle}}\right)\big(x_j(s)-\langle x_i(s), x_j(s)\rangle x_i(s)\big)\diff s\nonumber\\
&\hspace{0.5cm}-\int_0^t \sum_{j=1}^n\left(\frac{e^{\beta\langle y_i(s), y_j(s)\rangle}}{\sum_{k=1}^n e^{\beta\langle y_i(s), y_k(s)\rangle}}\right)\big(y_j(s)-\langle y_i(s), y_j(s)\rangle y_i(s)\big)\diff s.
\end{align}
We see that
\begin{align} \label{eq: lip.2}
\left\|\int_0^t \sum_{j=1}^n \left(\frac{e^{\beta\langle x_i(s), x_j(s)\rangle}}{\sum_{k=1}^n e^{\beta\langle x_i(s), x_k(s)\rangle}}\right)(x_j(s)-y_j(s))\diff s\right\|\leq \int_0^t\max_{j\in[n]}\|x_j(s)-y_j(s)\|\diff s.
\end{align}
On another hand, since the softmax function with a parameter $\beta$ is $\beta$--Lipschitz (with respect to the Euclidean norm), we also get 
\begin{align} \label{eq: lip.3} 
&\left\|\int_0^t\sum_{j=1}^n\left(\frac{e^{\beta\langle x_i(s), x_j(s)\rangle}}{\sum_{k=1}^n e^{\beta\langle x_i(s), x_k(s)\rangle}}-\frac{e^{\beta\langle y_i(s), y_j(s)\rangle}}{\sum_{k=1}^n e^{\beta\langle y_i(s), y_k(s)\rangle}}\right)y_j(s)\diff s \right\|\nonumber\\
&\qquad\leq\beta n^\frac12\int_0^t\left(\sum_{j=1}^n\Big[\langle x_i(s), x_j(s) \rangle-\langle y_i(s),y_j(s)\rangle\Big]^2\right)^{\frac12}\diff s\nonumber\\
&\qquad\leq2\beta n\int_0^t \max_{j\in[n]}\|x_j(s)-y_j(s)\|\diff s.\end{align}
Using~\eqref{eq: lip.2},~\eqref{eq: lip.3} and arguing similarly for the remaining terms in~\eqref{eq: lip.1}, we deduce that
\begin{equation*}
\|x_i(t)-y_i(t)\|\leq\|x_i(0)-y_i(0)\| + 10\max\{1,\beta\}n\int_0^t \max_{j\in[n]} \|x_j(s)-y_j(s)\|\diff s.
\end{equation*} 
Maximizing over $i\in[n]$ and applying the Gr\"onwall inequality yields
\begin{equation} \label{eq: stability.4ortho}
\max_{j\in[n]} \|x_j(t)-y_j(t)\|\leq c(\beta)^{nt}\max_{j\in[n]}\left\|x_j(0)-y_j(0)\right\|,
\end{equation}
for any $i\in[n]$ and $t\geq0$.

\subsubsection*{Step 2. Almost orthogonality}
Let $x_1(0), \ldots, x_n(0)\in\S$ be the random i.i.d. initial points. We prove that with high probability, there exist $n$ pairwise orthogonal points $y_1(0),\ldots,y_n(0)\in\S$, such that for any $i\in[n]$,
\begin{equation}\label{eq: almost.ortho.vec}
\|x_i(0)-y_i(0)\|\leq \sqrt{\frac{\log d}{d}}.
\end{equation}
To this end, we take $y_1(0)=x_1(0)$ and then construct the other points $y_i(0)$ by induction. Assume that $y_1(0),\ldots,y_i(0)$ are constructed for some $i\in[n]$, using only knowledge about the points $x_1(0),\ldots,x_i(0)$. Then by Lévy's concentration of measure, since $x_{i+1}(0)$ is independent from $x_1(0),\ldots,x_i(0)$ and uniformly distributed on $\S$, 
\begin{equation*} 
\mathbb{P}\left(\left\{\dist\left(x_{i+1}(0),\text{span}\{y_1(0),\ldots,y_i(0)\}^\perp\right)\leq\sqrt{\frac{\log d}{d}}\right\}\right)\geq 1-4id^{-1/64},
\end{equation*}
for some universal constants $c,C>0$.
Using the union bound, we gather that the event
\begin{equation*}
\mathscr{A}_0=\left\{\text{\eqref{eq: almost.ortho.vec} is satisfied for any $i\in[n]$}\right\}
\end{equation*}
has probability at least $p_0=1-2n^2d^{-1/64}$. We now consider the event
\begin{equation*}
    \mathscr{A}=\mathscr{A}_0\cap\{{\text{for some $C,\lambda>0$, \eqref{eq: expconvtocons} holds for any $i\in[n]$ and $t\geq 0$}}\}
\end{equation*}
which, since $d\geq n$ and thus the second event has probability $1$, 
also holds with probability at least $p_0=1-2n^2d^{-1/64}$. For the remainder of the proof, we assume that $\mathscr{A}$ is satisfied.

\subsubsection*{Step 3. Proof of~\eqref{eq: upto-t}}
Let $(y_i(\cdot))_{i\in[n]}\in C^0(\R_{\geq0};(\S)^n)$ denote the unique solution to the Cauchy problem for~\eqref{SA} corresponding to the initial datum $(y_i(0))_{i\in[n]}$. A combination of~\eqref{eq: stability.4ortho} and~\eqref{eq: almost.ortho.vec} yields
\begin{equation}\label{e:shortdist}
\|x_i(t)-y_i(t)\|\leq c(\beta)^{nt}\sqrt{\frac{\log d}{d}}
\end{equation}
for any $i\in[n]$ and $t\geq0$, under $\mathscr{A}$. Combining~\eqref{e:shortdist} with \Cref{thm: orthogonal} we obtain
\begin{equation}\label{e:ineqfirstpart}
\Big|\langle x_i(t), x_j(t)\rangle-\gamma_\beta(t)\Big|\leq 2c(\beta)^{nt}\sqrt{\frac{\log d}{d}}
\end{equation}
for any $i\neq j$ and $t\geq0$, under $\mathscr{A}$.

We turn to the proof of the second part of~\eqref{eq: upto-t}. For this, we prove that for large times $t$, both $\gamma_\beta(t)$ and $\langle x_i(t),x_j(t)\rangle$ are necessarily close to $1$. We first show that
\begin{equation}\label{e:ybetacloseto1}
1-\gamma_\beta(t)\leq \frac12 \exp\left(\frac{n^2e^\beta}{2\left(n+e^{\frac\beta2}\right)}-\frac{nt}{n+e^{\frac\beta2}} \right)
\end{equation}
for any $t\geq0$.
To this end, we notice that $t\mapsto\gamma_\beta(t)$ is increasing and thus $\gamma_\beta(t)\geq0$, as well as $\dot{\gamma}_\beta(t)\geq \frac{1}{ne^{\beta}}$ as long as $\gamma_\beta(t)\leq \frac12$. Therefore, 
\begin{equation*}
 \gamma_\beta\left(\frac{ne^{\beta}}{2}\right)\geq \frac12.   
\end{equation*}
We deduce that for $t\geq \frac{ne^{\beta}}{2}$,
\begin{equation*}
 \dot{\gamma}_\beta(t)\geq \frac{n(1-\gamma_\beta(t))}{n+e^{\frac\beta2}}.
\end{equation*}
Integrating this inequality from $\frac{ne^{\beta}}{2}$ to $t$, we obtain~\eqref{e:ybetacloseto1}. 
We now set $d^*(n,\beta)\geq n$ such that
\begin{equation} \label{eq: d.large} 
    \frac{d}{\log d}\geq \frac{16c(\beta)^2}{\gamma_\beta(\frac1n)^2}
\end{equation}
holds for any $d\geq d^*(n,\beta)$.
According to \Cref{lem: hemisphere.clustering}, since $\mathscr{A}$ is satisfied, there exists  $x^*\in\S$ such that $x_i(t)\rightarrow x^*$ for any $i\in[n]$ as $t\rightarrow +\infty$.
We set 
$$\alpha(t):=\min_{i\in[n]} \langle x_i(t),x^*\rangle,$$ 
and prove that
\begin{equation} \label{e:productcloseto1}
1-\alpha(t)\leq\exp\left(\frac{1-\gamma_\beta\left(\frac1n\right)t}{2ne^{2\beta}}\right).
\end{equation}
To this end, let us first prove that 
\begin{equation}\label{e:1/n}
    \alpha\left(\frac1n\right)\geq \frac12 \gamma_\beta\left(\frac1n\right).
\end{equation}
From Step 2 in the proof of Lemma~\ref{lem: hemisphere.clustering}, we gather that $x^*$ lies in the convex cone generated by the points $x_1(t),\ldots,x_n(t)$ for any $t>0$, and so the decomposition \eqref{e:decompox*.step2} holds.
Taking the inner product of $x_i(\frac1n)$ with the decomposition~\eqref{e:decompox*.step2} at time $t=\frac1n$, we get
\begin{align*}
    \alpha\left(\frac1n\right)\geq \min_{(i,j)\in[n]^2}\;\left\langle x_i\left(\frac1n\right),x_j\left(\frac1n\right)\right\rangle&\geq \gamma_\beta\left(\frac1n\right)-2c(\beta)\sqrt{\frac{\log(d)}{d}}\\
    &\geq\frac12\gamma_\beta\left(\frac1n\right),
\end{align*}
where the second inequality comes from~\eqref{e:shortdist} evaluated at time $t=\frac1n$, and the last inequality comes from~\eqref{eq: d.large}. This is precisely~\eqref{e:1/n}. 
Using the notation $a_{ij}(t)=Z_{\beta,i}(t)^{-1}e^{\beta\langle x_i(t), x_j(t)\rangle}$ as in the proof of \Cref{lem: hemisphere.clustering}, we now find
\begin{equation}\label{e:dotalpha}
\dot{\alpha}(t)=\langle \dot{x}_{i(t)}(t),x^*\rangle\geq \sum_{j=1}^n a_{i(t)j}(t)(1-\langle x_{i(t)}(t),x_j(t)\rangle) \alpha(t)
\end{equation}
for one of the indices $i(t)\in[n]$ achieving the minimum in the definition of $\alpha(t)$. Combining this with~\eqref{e:1/n}, we gather that $\alpha(t)\geq \alpha(\frac1n)$ for $t\geq\frac1n$. But
\begin{equation}\label{e:mineqalpha}
\min_{j\in[n]} \langle x_{i(t)}(t),x_j(t)\rangle\leq \sum_{k=1}^n \theta_k(t)\langle x_{i(t)}(t),x_k(t)\rangle= \langle x_{i(t)}(t),x^*\rangle=\alpha(t).
\end{equation}
Plugging~\eqref{e:mineqalpha} into~\eqref{e:dotalpha} and using $a_{ij}(t)\geq n^{-1}e^{-2\beta}$ we get
\begin{equation}\label{e:diffineqalpha}
\dot{\alpha}(t)\geq \frac{1}{ne^{2\beta}}\alpha\left(\frac1n\right)(1-\alpha(t))
\end{equation}
for $t\geq\frac1n$. Integrating~\eqref{e:diffineqalpha} from $\frac1n$ to $t$, we get~\eqref{e:productcloseto1}.
We therefore deduce from~\eqref{e:productcloseto1} that
$$\langle x_i(t), x_j(t)\rangle\geq 1-\exp\left(\frac{1-\gamma_\beta(\frac1n)t}{2ne^{2\beta}}\right)$$ 
holds for any distinct $i,j\in[n]$. 
Together with~\eqref{e:ybetacloseto1}, we then get
\begin{equation} \label{e:ineqsecondpart}
\Big|\langle x_i(t), x_j(t)\rangle-\gamma_\beta(t)\Big|\leq \exp\left(\frac{1-\gamma_\beta(\frac1n)t}{2ne^{2\beta}}\right)+\frac12 \exp\left(\frac{n^2e^\beta}{2(n+e^{\frac\beta2})}-\frac{nt}{n+e^{\frac\beta2}} \right).
\end{equation}
Finally, combining~\eqref{e:ineqfirstpart} and~\eqref{e:ineqsecondpart} we obtain~\eqref{eq: upto-t}.
\end{proof}

\begin{remark} \label{rem: usa.d}
An analogous statement to Theorem~\ref{thm: phase.transition.curve} holds for~\eqref{USA}, where $\gamma_\beta$ would rather be the unique solution to~\eqref{eq: ybetaUSA}. More concretely, Step 1 in the proof is only slightly changed---the constant one obtains in the analogue of~\eqref{eq: upto-t} is rather $c(\beta)^{nt}$ with $c(\beta)=e^{10\beta e^{2\beta}}$---. Step 2 remains unchanged. In Step 3,~\eqref{e:ybetacloseto1} is replaced by $\gamma_\beta(\frac{n}{2})\geq\frac12$ and 
$$1-\gamma_\beta(t)\leq\frac12\exp\left(-e^{\frac{\beta}{2}}\left(t-\frac{n}{2}\right)\right).
$$ 
The rest of the proof then remains essentially unchanged.
\end{remark}

\bibliographystyle{alpha}
\bibliography{refs}

\end{document}